\pdfoutput=1
\documentclass{article}

\usepackage[utf8]{inputenc} 
\usepackage[T1]{fontenc}    
\usepackage{hyperref}       
\usepackage{url}            
\usepackage{booktabs}       
\usepackage{amsfonts}       
\usepackage{nicefrac}       
\usepackage{microtype}      
\usepackage{graphicx}
\usepackage{subfigure}
\usepackage{bbm}
\usepackage{algorithm}
\usepackage{algorithmicx,wrapfig}
\usepackage{algpseudocode}
\usepackage{multirow}
\usepackage{amssymb}
\usepackage{float}
\usepackage{xcolor}
\floatname{algorithm}{Algorithm}

\newcommand*{\fullref}[1]{\hyperref[{#1}]{\autoref*{#1} \nameref*{#1}}}
\usepackage{verbatim}
\usepackage{amsmath}
\usepackage{enumitem}
\usepackage{authblk}
\usepackage[numbers,sort]{natbib}
\usepackage[margin=1in]{geometry}

\usepackage{ntheorem}
\allowdisplaybreaks[4]
\newtheorem{theorem}{Theorem}

\newtheorem{lemma}{Lemma}
\newtheorem{corollary}{Corollary}
\newtheorem*{proof}{Proof}
\newcommand{\tmix}{t_{\operatorname{mix}}}
\def\*#1{\boldsymbol{#1}}

\newcommand{\nospaceparagraph}[1]{\textbf{#1}\hspace{0.5em}}

\title{Moniqua: Modulo Quantized Communication\\ in Decentralized SGD}

\author{Yucheng Lu\thanks{Corresponds to: yl2967@cornell.edu} }
\author{Christopher De Sa\thanks{Corresponds to: cdesa@cs.cornell.edu} }
\affil{Department of Computer Science, Cornell\ University}
\date{}

\begin{document}

\maketitle

\begin{abstract}
Running Stochastic Gradient Descent (SGD) in a decentralized fashion has shown promising results.
In this paper we propose Moniqua, a technique that allows decentralized SGD to use quantized communication.
We prove in theory that Moniqua communicates a provably bounded number of bits per iteration, while converging at the same asymptotic rate as the original algorithm does with full-precision communication.
Moniqua improves upon prior works in that it (1) requires zero additional memory, (2) works with 1-bit quantization, and (3) is applicable to a variety of decentralized algorithms.
We demonstrate empirically that Moniqua converges faster with respect to wall clock time than other quantized decentralized algorithms. 
We also show that Moniqua is robust to very low bit-budgets, allowing  $1$-bit-per-parameter communication without compromising validation accuracy when training ResNet20 and ResNet110 on CIFAR10.
\end{abstract}

\section{Introduction}\label{introduction}
Stochastic gradient descent (SGD), as a widely adopted optimization algorithm for machine learning, has shown promising performance when running in parallel \cite{zhang2004solving,bottou2010large,dean2012large,goyal2017accurate}.
However, the communication bottleneck among workers\footnote{A worker could refer to any computing unit that is capable of computing, communicating and has local memory such as CPU, GPU, or even a single thread, etc.} can substantially slow down the training \cite{alistarh2018brief}. 
State-of-the-art frameworks such as TensorFlow \cite{abadi2016tensorflow}, CNTK \cite{seide2016cntk} and MXNet \cite{chen2015mxnet} are built in a centralized fashion, where workers exchange gradients either via a centralized parameter server \cite{li2014scaling, li2014communication} or the MPI AllReduce operation \cite{gropp1999using}. Such a design, however, puts heavy pressure on the central server and strict requirements on the underlying network. In other words, when the underlying network is poorly constructed, i.e. high latency or low bandwidth, it can easily cause degradation of training performance due to communication congestion in the central server or stragglers (slow workers) in the system.

There are two general approaches to deal with these problems: (1) decentralized training \cite{lian2017can,lian2017asynchronous,tang2018d,hendrikx2018accelerated} and (2) quantized communication\footnote{For brevity, in this paper we generally refer to lossy compression methods including quantization, sparsification, etc, as ``quantization.''} \cite{zhang2017zipml,alistarh2017qsgd,wen2017terngrad}.
In decentralized training, all the workers are connected to form a graph and each worker communicates only with neighbors by averaging model parameters between two adjacent optimization steps. This balances load and is robust to scenarios where workers can only be partially connected or the communication latency is high.
On the other hand, 
quantized communication reduces the amount of data exchanged among workers, leading to faster convergence with respect to wall clock time \cite{alistarh2017qsgd,seide20141,doan2018convergence,zhang2017zipml,wang2018atomo}. This is especially useful when the communication bandwidth is restricted.

At this point, a natural question is: \textit{Can we apply quantized communication to decentralized training, and thus benefit from both of them?} 
Unfortunately, directly combining them together negatively affects the convergence rate \citep{tang2018communication}. This happens because existing quantization techniques are mostly designed for centralized SGD, where workers communicate via exchanging gradients \citep{alistarh2017qsgd,seide20141,wangni2018gradient}. Gradients are robust to quantization since they get smaller in magnitude near local optima and in some sense carry less information, causing quantization error to approach zero \citep{de2018high}.
In contrast, decentralized workers are communicating the model parameters, which do not necessarily get smaller around local optima and thus the quantization error does not approach zero without explicitly increasing precision \cite{tang2018distributed}.
Previous work solved this problem by adding an error tracker to compensate for quantization errors \cite{tang2019texttt} 
or 
adding replicas of neighboring models and focusing on quantizing model-difference which does approach zero \cite{koloskova2019decentralized,tang2018communication}.
However, these methods have limitations in that: 
(1) the extra replicas or error tracking incurs substantial memory overhead that is proportional to size of models and the graph (more details in Section~\ref{related work});
and (2) these methods are either limited to constant step size or biased quantizers~\cite{koloskova2019decentralized,tang2018communication,tang2019texttt}.

To address these problems, in this paper we propose Moniqua, an additional-memory-free method for decentralized training to use quantized communication. Moniqua supports non-constant step size and biased quantizers.
Our contribution can be summarized as follows:
\begin{itemize}[nosep]
    \item We show by example that naively quantizing communication in decentralized training can fail to converge asymptotically. (Section~\ref{preliminary})
    \item We propose \textbf{Moniqua}, a general algorithm that uses \textbf{mo}dular arithmetic for commu\textbf{ni}cation \textbf{qua}ntization in decentralized training. We prove applying Moniqua achieves the same asymptotic convergence rate as the baseline full-precision algorithm (D-PSGD) while supporting extreme low bit-budgets. (Section~\ref{Moniqua intro})
    \item We apply Moniqua to decentralized algorithms with variance reduction and asynchronous communication ($D^2$ and AD-PSGD) and prove Moniqua enjoys the same asymptotic rate as with full-precision communication when applied to these cases. (Section~\ref{Scalability intro})
    \item We empirically evaluate Moniqua and show it outperforms all the related algorithms given an identical quantizer. We also show Moniqua is scalable and works with 1-bit quantization. (Section~\ref{Experiments})
\end{itemize}

\paragraph{Intuition behind Moniqua.}
In decentralized training, workers communicate to average their model parameters \cite{lian2017can}. As the algorithm converges, all the workers will approach the same stationary point as they reach consensus \cite{tang2018communication}. 
As a result, the difference in the same coordinate of models on two workers is becoming small.
Suppose $x$ and $y$ are the $i$th coordinates of models on workers $w_x$ and $w_y$, respectively. If we somehow know in advance that $|x-y|<\theta$, then if $w_y$ needs to obtain $x$, it suffices to fetch $x\bmod 2\theta$ rather than $x$ from $w_x$.
Note that $x\bmod 2\theta$ is generally a smaller number than $x$, which means to obtain the same absolute error, fewer bits are needed compared to fetching $x$ directly.
Formally, this intuition is captured in the following lemma.
\begin{lemma}\label{modulo_lemma}
Define the modulo operation $\bmod$ as the follows. For any $z\in\mathbb{R}$ and $a\in\mathbb{R}^{+}$,
\begin{equation}\label{modulo_definition}
    \{ z\bmod a \} = \{z+na | n\in\mathbb{N}\}\cap[-a/2,a/2)
\end{equation}
then for any $x,y\in\mathbb{R}$, if $|x-y|<\theta$, then
\begin{align*}
    x = (x\bmod 2\theta - y\bmod 2\theta)\bmod 2\theta + y.
\end{align*}
\end{lemma}

\section{Related Work}\label{related work}

\begin{table}[t]
\caption{Comparison among Moniqua and baseline algorithms, where workers form a graph with $n$ vertices and $m$ edges. $d$ refers to the model dimension. Detailed discussion can be found in Section~\ref{related work}. The additional memory refers to the space complexity required additional to the baseline full-precision communication decentralized training algorithm (D-PSGD).}
\label{algorithm comparison table}
\begin{center}
\begin{tabular}{ccccccc}
\toprule
& DCD-PSGD & ECD-PSGD & ChocoSGD & DeepSqueeze & \textbf{Moniqua}\\
\midrule
Supports biased quantizers & No  & No & Yes & Yes & \textbf{Yes} \\
\midrule
Supports 1-bit quantization & No  & No & Yes & No & \textbf{Yes} \\
\midrule
Works beyond D-PSGD & No  & No & No & No & \textbf{Yes} \\
\midrule
Non-constant Step Size & No  & No & No & No & \textbf{Yes} \\
\midrule
Additional Memory & $\Theta(md)$  & $\Theta(md)$ & $\Theta(md)$ & $\Theta(nd)$ & \textbf{0} \\
\bottomrule
\end{tabular}
\end{center}
\end{table}

\nospaceparagraph{Decentralized Stochastic Gradient Descent (SGD).}
Decentralized algorithms \cite{mokhtari2015decentralized,sirb2016consensus,lan2017communication,wu2018decentralized} have been widely studied with consideration of communication efficiency, privacy and scalability. In the domain of large-scale machine learning, D-PSGD was the first Decentralized SGD algorithm that was proven to enjoy the same asymptotic convergence rate $O(1/\sqrt{Kn})$ (where $K$ is the number of total iterations and $n$ is the number of workers) as centralized algorithms~\cite{lian2017can}. 
After D-PSGD came $D^2$, which improves D-PSGD and is applicable to the case where workers are not sampling from identical data sources~\cite{tang2018d}.
Another extension was AD-PSGD, which lets workers communicate \emph{asynchronously} and has a convergence rate of $O(1/\sqrt{K})$~\cite{lian2017asynchronous}. 
Other relevant work includes: \citet{he2018cola}, which investigates decentralized learning on linear models; \citet{nazari2019dadam}, which introduces decentralized algorithms with online learning; \citet{zhang2019asynchronous}, which
analyzes the case when workers cannot mutually communicate; and \citet{assran2018stochastic}, which investigates Decentralized SGD specifically for deep learning.

\paragraph{Quantized Communication in Centralized SGD.}
Prior research on quantized communication is often focused on centralized algorithms, such as randomized quantization~\cite{doan2018convergence,suresh2017distributed,zhang2017zipml} and randomized sparsification~\cite{wangni2018gradient,stich2018sparsified,wang2018atomo,alistarh2018convergence}. Many examples of prior work focus on studying quantization in the communication of deep learning tasks specifically~\cite{han2015deep,wen2017terngrad,grubic2018synchronous}. 
\citet{alistarh2017qsgd} proposes QSGD, which uses an encoding-efficient scheme, and discusses its communication complexity. Another method, 1bitSGD, quantizes exchanged gradients with one bit per parameter and shows great empirical success on speech recognition~\cite{seide20141}. Other work discusses the convergence rate under sparsified or quantized communication~\cite{jiang2018linear,stich2018sparsified}. \citet{acharya2019distributed} theoretically analyzes sublinear communication for distributed training.

\paragraph{Quantized Communication in Decentralized SGD.}
Quantized communication for decentralized algorithms is a rising topic in the optimization community.
Previous work has proposed decentralized algorithms with quantized communication for strongly convex objectives~\cite{DBLP:journals/corr/abs-1806-11536}. 
Following that, 
\citet{tang2018communication} proposes DCD/ECD-PSGD, which quantizes communication via estimating model difference.
Furthermore,
\citet{tang2019texttt} proposes DeepSqueeze, which applies an error-compensation method \cite{wu2018error} to decentralized setting.
\citet{koloskova2019decentralized} proposed ChocoSGD, a method that lets workers estimate remote models with a local estimator, which supports arbitrary quantization by tuning the communication matrix.
\paragraph{How Moniqua improves on prior works.}
We summarize the comparison among Moniqua and other baseline algorithms in Table~\ref{algorithm comparison table}.
Specifically, Moniqua works with a wider range of quantizers (those with biased estimation or extremely restricted precision, e.g. 1bit per parameter) with theoretical guarantees.
It enjoys several statistical benefits such as supporting non-constant step sizes and can be extended to different scenarios that are beyond synchronous setting (D-PSGD).
Most importantly, it prevents the algorithms from trading memory with bandwidth, requiring zero additional memory in the implementation.

\section{Setting and Notation}\label{preliminary}
In this section, we introduce our notation and the general assumptions we will make about the quantizers for our results to hold.
Then we describe D-PSGD~\citep{lian2017can}, the basic algorithm for Decentralized SGD, and we show how naive quantization can fail in decentralized training.

\paragraph{Quantizers.}\label{quantizer_intro}
Throughout this paper, we assume that we use a quantizer $\mathcal{Q}_\delta$ that has 
bounded error
\begin{equation}\label{property quantizer}
\textstyle
\left\|\mathcal{Q}_\delta(\*x)-\*x\right\|_\infty \leq \delta\hspace{1em} \text{when}\hspace{1em}  \*x\in\left[-\frac{1}{2},\frac{1}{2}\right)^d
\end{equation}
where $\delta$ is some constant.
Note that in this assumption, we do not assume any bound for $\*x$ outside $\left[-\frac{1}{2},\frac{1}{2}\right)^d$: as will be shown later, a bound in this region is sufficient for our theory.
This assumption holds for both linear \citep{gupta2015deep,de2017understanding} and non-linear \citep{stich2018local,alistarh2017qsgd} quantizers.
In general, a smaller $\delta$ denotes more fine-grained quantization requiring more bits.
For example, a biased linear quantizer can achieve (\ref{property quantizer}) by rounding a scalar $x$ to the nearest number in the set $\{2 \delta n \mid n \in \mathbb{Z} \}$; this will require about $\delta^{-1}$ quantization points to cover the interval $[-1/2,1/2)$, so such a linear quantizer can satisfy (\ref{property quantizer}) using only $\left\lceil\log_2\left(\frac{1}{2\delta}+1\right)\right\rceil$ bits \citep{li2017training,gupta2015deep}.

\paragraph{Decentralized parallel stochastic gradient descent (D-PSGD).}
D-PSGD \citep{lian2017can} is the first and most basic Decentralized SGD algorithm.
In D-PSGD, $n$ workers are connected to form a graph. Each worker $i$ stores a copy of model $\*x\in\mathbb{R}^d$ and a local dataset $\mathcal{D}_i$ and collaborates to optimize
\begin{equation}\label{Sync_objective}
\textstyle
\min_{\*x\in\mathbb{R}^d} f(\*x) = \frac{1}{n}\sum_{i=1}^{n}\underbrace{\mathbb{E}_{\xi\sim\mathcal{D}_i}f_i(\*x;\xi)}_{f_i(\*x)}.
\end{equation}
where $\xi$ is a data sample from $\mathcal{D}_i$. In each iteration of D-PSGD, worker $i$ computes a local gradient sample using $\mathcal{D}_i$.
Then it \emph{averages} its model parameters with its neighbors according to a symmetric and doubly stochastic matrix $\*W$, where $\*W_{ij}$ denotes the ratio worker $j$ averages from worker $i$. 
Formally: Let $\*x_{k,i}$ and $\*{\tilde{g}}_{k,i}$ denote local model and sampled gradient on worker $i$ at $k$-th iteration, respectively. Let $\alpha_k$ denote the step size. The update rule of D-PSGD can be expressed as:
\begin{equation*}\label{DT_basic_update_rule}
\begin{aligned}
\*x_{k+1,i} = \sum\nolimits_{j=1}^{n}\*x_{k,j}\*W_{ji} - \alpha_k\*{\tilde{g}}_{k,i} = \*x_{k,i} \underbrace{-\sum\nolimits_{j=1}^{n}(\*x_{k,i}-\*x_{k,j})\*W_{ji}}_{\text{communicate to reduce difference}} \underbrace{- \alpha_k\*{\tilde{g}}_{k,i}}_{\text{gradient step}}
\end{aligned}
\end{equation*}
From (\ref{DT_basic_update_rule}) we can see the update of a single local model contains two parts: communication to reduce model difference and a gradient step.
\citet{lian2017can} shows that all local models in D-PSGD reach the same stationary point.

\paragraph{Failure with naive quantization.}
Here, we illustrate why naively quantizing communication in decentralized training \textemdash directly quantizing the exchanged data\textemdash can fail to converge asymptotically even on a simple problem. This naive approach with quantizer $\mathcal{Q}_\delta$ can be represented by
\begin{equation}\label{naivequantizationupdaterule}
\*x_{k+1,i} = \*x_{k,i}\*W_{ii} + \sum\nolimits_{j\neq i}\mathcal{Q}_\delta(\*x_{k,j})\*W_{ji} - \alpha_k\*{\tilde{g}}_{k,i}
\end{equation}
Based on Equation~\ref{naivequantizationupdaterule}, we obtain the following theorem.
\begin{theorem}\label{quadratic}
For some constant $\delta$, suppose that we use an unbiased linear quantizer $\mathcal{Q}$ with representable points $\{ \delta n \mid n \in \mathbb{Z} \}$ to learn on the quadratic objective function $f(\*x) = (\*x - \delta\*{1}/2)^\top (\*x - \delta\*{1}/2) / 2$ with the direct quantization approach (\ref{naivequantizationupdaterule}).
Let $\phi$ denote the smallest value of a non-zero entry in $W$.
Regardless of what step size we adopt, it will always hold for all iterations $k$ and local model indices $i$ that $\mathbb{E}\left\|\nabla f(\*x_{k,i})\right\|^2 \geq \frac{\phi^2\delta^2}{8(1+\phi^2)}$.
That is, the local iterates will fail to asymptotically converge to a region of small gradient magnitude in expectation.
\end{theorem}
Theorem~\ref{quadratic} shows that naively quantizing communication in decentralized SGD, even with an unbiased quantizer, any local model can fail to converge on a simple quadratic objective.
This is not satisfying, since, it implies we would need more advanced quantizers which are likely to require more system resources such as memory.
In the following section, we propose a technique, Moniqua, that solves this problem.

\section{Moniqua}\label{Moniqua intro}
In Section~\ref{introduction}, we described the basic idea behind Moniqua: to use modular arithmetic to decrease the magnitude of the numbers we are quantizing. We now describe how Moniqua implements this intuition with a given quantizer $\mathcal{Q}_\delta$.
Consider the two-scalar example from Section~\ref{introduction}. Suppose we know $y$ and $|x-y|<\theta$ and need to fetch $x$ from a remote host via a quantizer $\mathcal{Q}_\delta$ to recover $x$. We've shown in Section~\ref{preliminary} that fetching and using $\mathcal{Q}_\delta(x)$ leads to divergence. Instead, we define a parameter $B_\theta = (2 \theta)/(1-2\delta)$ and then use the modulo operation and fetch $\mathcal{Q}_\delta\left((x /  B_\theta) \bmod 1\right)$ from the remote host, from which we can approximately recover $x$ as
\begin{equation}
\hat{x}=\left( B_\theta\mathcal{Q}_\delta\left((x/ B_\theta)\bmod 1\right)-y\right)\bmod B_\theta+y.
\label{eqnMoniquaFetch}
\end{equation}
Note that inside the quantizer we rescale $x$ to $x/ B_\theta$, which is required for (\ref{property quantizer}) to apply.
This approach has quantization error bounded proportional to the original bound $\theta$, as shown in the following lemma.
\begin{lemma}\label{modulo_numerical_lemma}
    For any scalars $x,y\in\mathbb{R}$, if $|x-y|<\theta$ and if $\delta<\frac{1}{2}$, then if we set $B_\theta = (2 \theta)/(1-2\delta)$ and $\hat x$ as in (\ref{eqnMoniquaFetch}),
    \[
        \left| \hat x - x \right|
        \leq \delta B_\theta = \theta \cdot (2 \delta)/(1-2\delta).
    \]
\end{lemma}
Importantly, since the quantization error is decreasing with $\theta$, if we are able to prove a decentralized algorithm approaches consensus and use this proof to give a bound of the form $| x - y | < \theta$, this bound will give us a compression procedure (\ref{eqnMoniquaFetch}) with smaller error as our consensus bound improves.
We formalize this approach as Moniqua (Algorithm~\ref{Moniqua algo}). 
(Note that all the division and mod operations in Algorithm~\ref{Moniqua algo} act element-wise.)
\begin{algorithm}[t]
	\caption{Pseudo-code of Moniqua on worker $i$}\label{Moniqua algo}
	\begin{algorithmic}[1]
		\Require initial point $\*x_{0,i} = \*x_0$, step size $\{\alpha_k\}_{k\geq 0}$, the a priori bound $\{\theta_k\}_{k\geq 0}$, communication matrix $\*W$, number of iterations $K$, quantizer $\mathcal{Q}_\delta$, neighbor list $\mathcal{N}_i$
		
		\For{$k=0,1,2,\cdots,K-1$}
			\State Compute a local stochastic gradient $\*{\tilde{g}}_{k,i}$ with data sample $\xi_{k,i}$ and current weight $\*x_{k,i}$
			\State Send modulo-ed model to neighbors: 
			\[
			\*q_{k,i}= \mathcal{Q}_\delta\left( \left( \*x_{k,i} / B_{\theta_k} \right) \bmod 1\right)
			\]
			\State Compute local biased term $\*{\hat{x}}_{k,i}$ as:
			\[
			\*{\hat{x}}_{k,i}=\*q_{k,i}B_{\theta_k}- \*x_{k,i}\bmod { B_{\theta_k}}+\*x_{k,i}
			\]
			\State Recover model received from worker $j$ as:
			\[
			\*{\hat{x}}_{k,j}=\left(\*q_{k,j}B_{\theta_k}-\*x_{k,i}\right)\bmod B_{\theta_k}+\*x_{k,i}
			\]
			\State Average with neighboring workers: 
			\[
			\*x_{k+\frac{1}{2},i} \leftarrow \*x_{k,i} + \sum_{j\in\mathcal{N}_i}(\*{\hat{x}}_{k,j}-\*{\hat{x}}_{k,i})\*W_{ji}
			\]
			\State Update the local weight with local gradient: 
			\[
			\*x_{k+1,i} \leftarrow \*x_{k+\frac{1}{2},i} - \alpha_k\*{\tilde{g}}_{k,i}
			\]
		\EndFor
		\State \textbf{return} Averaged model $\*{\overline{X}}_K=\frac{1}{n}\sum\nolimits_{i=1}^{n}\*x_{K,i}$
	\end{algorithmic}
\end{algorithm}

Note that in line 4 and 6, we compute and cancel out a local biased term, this is to cancel out the extra noise which may be brought to the global model. As we will show in the supplementary material, cancelling out this local biased term reduces extra noise to the algorithm.
And in Algorithm~\ref{Moniqua algo}, we consider the general case where $\theta$ can be a iteration dependent bound. As will be shown later, a constant $\theta$ also guarantees convergence.

We now proceed to analyze the convergence rate of Algorithm~\ref{Moniqua algo}.
We use the following common assumptions for analyzing decentralized optimization algorithms \citep{lian2017can,tang2018communication,koloskova2019decentralized}.
{\begingroup
\setlength{\abovedisplayskip}{3pt}
\setlength{\belowdisplayskip}{3pt}
\begin{enumerate}[label=(A\arabic*),font=\bfseries,itemsep=2pt,topsep=0pt]
	\item \textbf{Lipschitzian gradient.} All the functions $f_i$ have $L$-Lipschitzian gradients.\label{Assumption1}
	\begin{align*}
	    \|\nabla f_i(\*x) - \nabla f_i(\*y)\| \leq L\|\*x-\*y\|, \forall \*x, \*y\in\mathbb{R}^d
	\end{align*}
	\item \textbf{Spectral gap.} The communication matrix $\*W$ is a symmetric doubly stochastic matrix and
	\[\max\{|\lambda_2(\*W)|, |\lambda_n(\*W)|\}=\rho<1,\] 
	where $\lambda_i(\*W)$ denotes the the $i$th largest eigenvalue of $\*W$.\label{Assumption2}
	\item \textbf{Bounded variance.} There exist non-negative constants $\sigma$ and $\varsigma \in \mathbb{R}$ such that
	\begin{align*}
	\mathbb{E}_{\xi_i\sim\mathcal{D}_i}\left\|\nabla\tilde{f}_i(\*x;\xi_i) - \nabla f_i(\*x)\right\|^2 \leq & \sigma^2\\\mathbb{E}_{i\sim\{1,\cdots,n\}}\left\|\nabla f_i(\*x) - \nabla f(\*x)\right\|^2 \leq &\varsigma^2
	\end{align*}
	where $\nabla\tilde{f}_i(\*x;\xi_i)$ denotes the gradient sample on worker $i$ computed via data sample $\xi_i$.\label{Assumption3}
	\item \textbf{Initialization.} All the local models are initialized with the same weight: $\*x_{0,i}=\*x_0$ for all $i$, and without loss of generality $\*x_0=\*0$. \label{Assumption4}
	\item \textbf{Bounded gradient magnitude.} For some constant $G_\infty$, the norm of a sampled gradient is bounded by $\left\|\*{\tilde{g}}_{k,i}\right\|_\infty \leq G_\infty$, for all $i$ and $k$. \label{Assumption5}
\end{enumerate}
\endgroup}

Lemma~\ref{modulo_numerical_lemma} states that the error bound from quantization is proportional to $\theta$. In other words, a tight estimation or choice on the $\theta$ will lead to smaller quantization error in the algorithm.
We present these parameter choices in Theorem~\ref{Moniqua convergence rate}, along with the resulting convergence rate for Moniqua.
\begin{theorem}\label{Moniqua convergence rate}
Consider adopting a non-increasing step size scheme $\{\alpha_t\}_{t\geq 0}$ such that there exists constant $C_\alpha>0$ and $\eta$ ($0<\eta\leq 1$) that for any $k,t\geq 0$, $\frac{\alpha_k}{\alpha_{k+t}}\leq C_\alpha\eta^t$, set $\theta_k=\frac{2\alpha_kG_\infty C_\alpha\log(16n)}{1-\eta\rho}$ and $\delta=\frac{1-\eta\rho}{8C_\alpha^2\eta\log(16n)+2(1-\eta\rho)}$, then Algorithm~\ref{Moniqua algo} converges at the following rate:
{
\begin{align*}
\sum_{k=0}^{K-1}\alpha_k \mathbb{E}\left\|\nabla f(\*{\overline{X}}_k)\right\|^2 
\leq 4(\mathbb{E}f(\*0) - \mathbb{E}f^*) + \frac{2\sigma^2L}{n}\sum_{k=0}^{K-1}\alpha_k^2 + \frac{8(\sigma^2+3\varsigma^2)L^2}{(1-\rho)^2}\sum_{k=0}^{K-1}\alpha_k^3 + \frac{8G_\infty^2dL^2}{(1-\rho)^2C_\alpha^2}\sum_{k=0}^{K-1}\alpha_k^3
\end{align*}}
where $f^*=\inf_{\*x}f(\*x)$.
\end{theorem}
Theorem~\ref{Moniqua convergence rate} shows that the priori bound $\theta_k$ is proportional to the step size and increases at the logarithmic speed when system size $n$ increases. The two-constant assumption on the step size prevents it from decreasing too fast. As a rapidly decreasing step size would prevent us from obtaining such a priori bound in theory.
This assumption generally holds for most of the step size schemes.
Just as baseline algorithms, by setting step size to a constant, we can obtain a concrete convergence bound as shown in the following corollary.
\begin{corollary}\label{Moniqua convergence rate corollary}
If we adopt a step size scheme where $\alpha_k=\frac{1}{\varsigma^{2/3}K^{1/3}+\sigma\sqrt{K / n}+2L}$ in Theorem~\ref{Moniqua convergence rate},
then the output of Algorithm~\ref{Moniqua algo} converges at the asymptotic rate
{
\begin{align*}
\frac{1}{K}\sum_{k=0}^{K-1}\mathbb{E}\left\|\nabla f(\*{\overline{X}}_k)\right\|^2 \lesssim & \frac{1}{K} + \frac{\sigma}{\sqrt{nK}} + \frac{\varsigma^{\frac{2}{3}}}{K^{\frac{2}{3}}} + \frac{(\sigma^2+G_\infty^2d)n}{\sigma^2K+n}.
\end{align*}}
\end{corollary}

\paragraph{Consistent with D-PSGD.} Note that D-PSGD converges at the asymptotic rate of $O(\sigma / \sqrt{nK}+\varsigma^{\frac{2}{3}}/K^{\frac{2}{3}}+n/K)$, and thus Moniqua has the same asymptotic rate as D-PSGD~\citep{lian2017can}.
That is, the asymptotic convergence rate is not negatively impacted by the quantization.

\paragraph{Robust to large $d$.} In Assumptions~\ref{Assumption3} and~\ref{Assumption5}, we use $l_2$-norm and $l_\infty$-norm to bound sample variance and gradient magnitude, respectively.
Note that, when $d$ gets larger, the variance $\sigma^2$ will also tend to grow proportionally.
So, the last term will tend to remain $n/K$ asymptotically with large $d$.

\paragraph{Bound on the Bits.}
The specific number of bits required by Moniqua depends on the underlying quantizer ($\mathcal{Q}_\delta$). 
If we use nearest neighbor rounding \citep{gupta2015deep} with a linear quantizer as $\mathcal{Q}_\delta$ in Theorem~\ref{Moniqua convergence rate}, it suffices to use at each step a number of bits $\mathcal{B}$ for each parameter sent, where
\begin{align*}
\textstyle
\mathcal{B} \leq \left\lceil\log_2\left(\frac{1}{2\delta}+1\right)\right\rceil = \left\lceil\log_2\left(\frac{4\log_2(16n)}{1-\rho}+3\right)\right\rceil
\end{align*}
Note that this bound is independent of model dimension $d$. When the system scales up, the number of required bits grows at a rate of $O\left(\log\log n\right)$. Note that, this is a general bound on the number of bits required by Moniqua using the same communication matrix $\*W$ as the baseline.
To enforce a even more restricted bit-budget (e.g. 1 bit), Moniqua can still converge at the same rate by adjusting the communication matrix.

\paragraph{1-bit Quantization.}
We can also add a consensus step \cite{tang2019texttt,koloskova2019decentralized} to allow Moniqua to use 1 bit per number. Specifically, we adopt a slack communication matrix $\*{\overline{W}}=\gamma \*W+(1-\gamma)\*I$ and tune $\gamma$ as a hyperparameter. We formalize this result in the following Theorem.
\begin{theorem}\label{arbitrary lemma}
Consider using a communication matrix in the form of $\*{\overline{W}} = \gamma \*W + (1-\gamma)\*I$. If we set $\theta=\frac{2\alpha G_\infty\log(16n)}{\gamma(1-\rho)}$, 
$\gamma=\frac{2}{1-\rho + \frac{16\delta^2}{(1-2\delta)^2}\cdot\frac{64\log(4n)\log(K)}{1-\rho}}$,
and $\alpha=\frac{1}{\varsigma^{\frac{2}{3}}K^{\frac{1}{3}}+\sigma\sqrt{\frac{K}{n}}+2L}$, then the output of Algorithm~\ref{Moniqua algo} converges at the asymptotic rate
{
\begin{align*}
\frac{1}{K}\sum_{k=0}^{K-1}\mathbb{E}\left\|\nabla f(\*{\overline{X}}_k)\right\|^2
\lesssim \frac{\sigma}{\sqrt{nK}} + \frac{1}{K} + \frac{\varsigma^{\frac{2}{3}}\delta^4\log^2(n)\log^2(K)}{K^{\frac{2}{3}}(1-2\delta)^4}+\frac{\sigma^2n\delta^4\log^2(n)\log^2(K)}{(\sigma^2K+n)(1-2\delta)^4}
+\frac{n\delta^6\log^4(n)\log^2(K)}{(\sigma^2K+n)(1-2\delta)^6}
\end{align*}}
\end{theorem}
Note that the dominant term in Theorem~\ref{arbitrary lemma} is still $O(\sigma/\sqrt{nK})$, which means Moniqua converges at the asymptotic rate the same as full precision D-PSGD \cite{lian2017can} even with more restricted bits-budget. Note that in Theorem~\ref{arbitrary lemma}, the only requirement on the quantizer is $\delta<\frac{1}{2}$. Considering the properties of our quantizer (\ref{property quantizer}), this version of Moniqua allowes us to use 1 bit in general per parameter.

\section{Scalable Moniqua}\label{Scalability intro}
So far, we have discussed how Moniqua, along with baseline algorithms, modifies D-PSGD to use communication quantization. 
Note that the basic idea of using modular arithmetic in quantized communication is invariant to the algorithm being used. In light of this, in this section we show Moniqua is general enough to be applied on other decentralized algorithms that are beyond D-PSGD.
Previous work has extended D-PSGD to $D^2$~\cite{tang2018d} (to make Decentralized SGD applicable to workers sampling from different data sources) and AD-PSGD \cite{lian2017asynchronous}~(an asynchronous version of D-PSGD). In this section, we  prove Moniqua is applicable to both of these algorithms.

\paragraph{Moniqua with Decentralized Data}
Decentralized data refers to the case where all the local datasets $\mathcal{D}_i$ are not identically distributed~\cite{tang2018d}. More explicitly, the outer variance $\mathbb{E}_{i\sim\{1,\cdots,n\}}\left\|\nabla f_i(\*x) - \nabla f(\*x)\right\|^2$ is no longer bounded by $\varsigma^2$ as assumed in D-PSGD (Assumption~\ref{Assumption3}). We apply Moniqua to $D^2$ \cite{tang2018d}, a decentralized algorithm designed to tackle this problem by reduing the variance over time.
Applying Moniqua on $D^2$ can be explicitly expressed\footnote{For brevity, the detailed pseudo code can be found in the supplemenraty material.} as:
{
\begin{align*}
\*X_{k+\frac{1}{2}} &= 2\*X_k - \*X_{k-1} - \alpha_k\*{\tilde{G}}_k + \alpha_{k-1}\*{\tilde{G}}_{k-1}\\
\*X_{k+1} &= \*X_{k+\frac{1}{2}}\*W + (\*{\hat{X}}_{k+\frac{1}{2}}-\*X_{k+\frac{1}{2}})(\*W-\*I)
\end{align*}}
where $\*X_k$, $\*{\tilde{G}}_k$ and $\*{\hat{X}}_{k+\frac{1}{2}}$ are matrix in the shape of $\mathbb{R}^{d\times n}$, where their $i$-th column are $\*x_{k,i}$, $\*{\tilde{g}}_{k,i}$ and $\*{\hat{x}}_{{k+\frac{1}{2}},i}$ respectively. And $\*X_{-1}$ and $\*{\tilde{G}}_{-1}$ are $\*0^{d\times n}$ by convention. Based on this, we obtain the following convergence theorem.
\begin{theorem}\label{thmMoniquaD2}
If we apply Moniqua on $D^2$ in a setting where $\theta=(6D_1n+8)\alpha G_\infty$, $\delta=\frac{1}{12nD_2+2}$ and $\alpha_k=\alpha=\frac{1}{\sigma\sqrt{K/n}+2L}$ where $D_1$ and $D_2$ are two constants\footnote{they only depend on the eigenvalues of $\*W$ (definition can be found in supplementary material)}, applying Moniqua on $D^2$ has the following asymptotic convergence rate:
{
\begin{align*}
\frac{1}{K}\sum_{k=0}^{K-1}\mathbb{E}\left\|\nabla f(\*{\overline{X}}_k)\right\|^2 \lesssim \frac{1}{K} + \frac{\sigma}{\sqrt{nK}} + \frac{(\sigma^2+G_\infty^2d)n}{\sigma^2K+n}.
\end{align*}}
\end{theorem}
Note that $D^2$ \cite{tang2018d} with full-precision communication has the asymptotic convergence rate of $O\left(\frac{1}{K} + \frac{\sigma}{\sqrt{nK}} + \frac{n}{K}\right)$, Moniqua on $D^2$ has the same asymptotic rate.
\paragraph{Moniqua with Asychronous Communication}
Both D-PSGD and $D^2$ are synchronous algorithms as they require global synchronization at the end of each iteration, which can become a bottleneck when such synchronization is not cheap. Another algorithm, AD-PSGD, avoids this overhead by letting workers communicate asynchronously~\cite{lian2017asynchronous}. In the analysis of AD-PSGD, an iteration represents a \emph{single} gradient update on \emph{one} randomly-chosen worker, rather than a synchronous bulk update of all the workers. This single-worker-update analysis models the asynchronous nature of the algorithm. Applying Moniqua on AD-PSGD can be explicitly expressed\footnote{For brevity, the detailed pseudo code can be found in the supplemenraty material.} as:
\begin{align*}
\*X_{k+1} = \*X_k\*W_k + (\*{\hat{X}}_k-\*X_k)(\*W_k-\*I) - \alpha_k\*{\tilde{G}}_{k-\tau_k}
\end{align*}
where $\*W_k$ describes the communication behaviour between the $k$th and $(k+1)$th gradient update, and $\tau_k$ denotes the delay (measured as a number of iterations) between when the gradient is computed and updated to the model.
Note that unlike D-PSGD, here $\*W_k$ can be different at each update step and usually each individually has $\rho = 1$, so we can't expect to get a bound in terms of a bound on the spectral gap, as we did in Theorems~\ref{Moniqua convergence rate} and~\ref{thmMoniquaD2}.
Instead, we require the following condition, which is inspired by the literature on Markov chain Monte Carlo methods: for some constant $\tmix{}$ and for any $k$,
$\textstyle
    \forall \*\mu \in \mathbb{R}^n, \text{ if } \*e_i^\top\*\mu \ge 0 \text{ and } \*{1}^\top \*\mu = 1, \text{ it must hold that }
    \left\| \left( \prod_{i=1}^{\tmix{}} \*W_{k+i} \right) \*\mu - \frac{\*{1}}{n} \right\|_1 \le \frac{1}{2}.$
We call this constant $\tmix{}$ because it is effectively the \emph{mixing time} of the time-inhomogeneous Markov chain with transition probability matrix $\*W_k$ at time $k$~\cite{levin2017markov}.
Note that this condition is more general than those used in previous work on AD-PSGD because it does not require that the $\*W_k$ are sampled independently or in an unbiased manner.
Using this, we obtain the following convergence theorem.
\begin{theorem}\label{thmMoniquaAD-PSGD}
If we apply Moniqua on AD-PSGD in a setting where
$\theta=16 \tmix{}\alpha G_\infty$, $\delta=\frac{1}{64 \tmix{}+2}$ and $\alpha_k=\alpha=\frac{n}{2L + \sqrt{K(\sigma^2+6\varsigma^2)}}$, 
applying Moniqua on AD-PSGD has the following asymptotic convergence rate:
{
\begin{align*}
\frac{1}{K}\sum_{k=0}^{K-1}\mathbb{E}\left\|\nabla f(\*{\overline{X}}_{k})\right\|^2 \lesssim \frac{1}{K} + \frac{\sqrt{\sigma^2+6\varsigma^2}}{\sqrt{K}} + \frac{(\sigma^2+6\varsigma^2)\tmix^2{}n^2}{(\sigma^2+6\varsigma^2)K+1} + \frac{n^2 \tmix^2{}G_\infty^2d}{(\sigma^2+6\varsigma^2)K+1}
\end{align*}}
\end{theorem}
Note that AD-PSGD \cite{lian2017asynchronous} with full-precision communication has the asymptotic convergence rate of 

$O\left(\frac{1}{K} + \frac{\sqrt{\sigma^2+6\varsigma^2}}{\sqrt{K}} + \frac{n^2}{K}\right)$, Moniqua obtains the same asymptotic rate.

Since adopting a slack matrix to enable 1-bit quantization in these two algorithms will be similar to the case in Theorem~\ref{arbitrary lemma}, we omit the discussion here for brevity.


\section{Experiments}\label{Experiments}
\begin{figure*}[t]
    \subfigure[Train Loss vs Time(s)  \newline Bandwidth=200Mbps,
    Latency=0.15ms]{
        \includegraphics[width=0.48\textwidth]{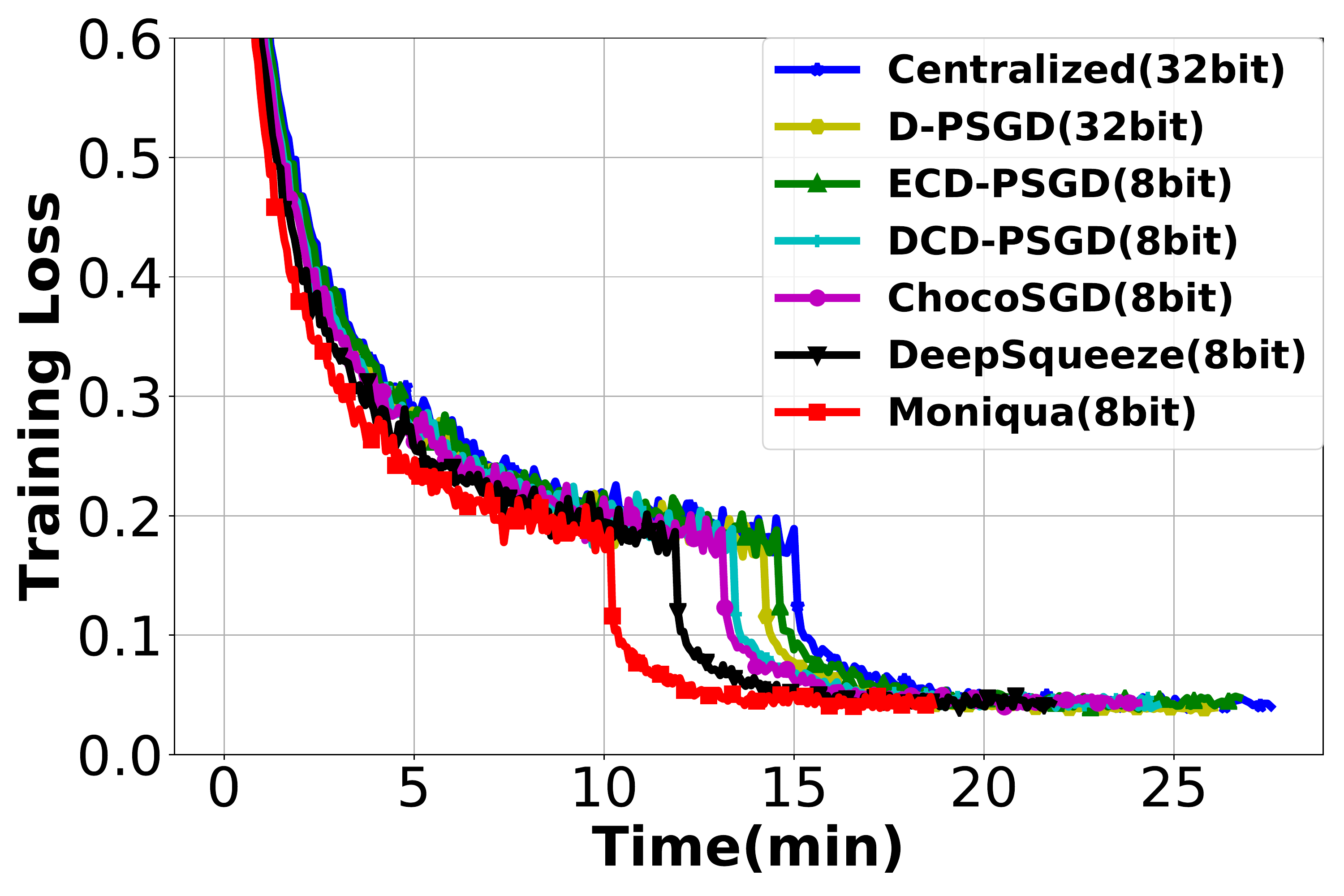}
        \label{200mbps}
    }
    \subfigure[
    Train Loss vs Time(s) \newline  
    Bandwidth=100Mbps,
    Latency=0.15ms
    ]{
        \includegraphics[width=0.48\textwidth]{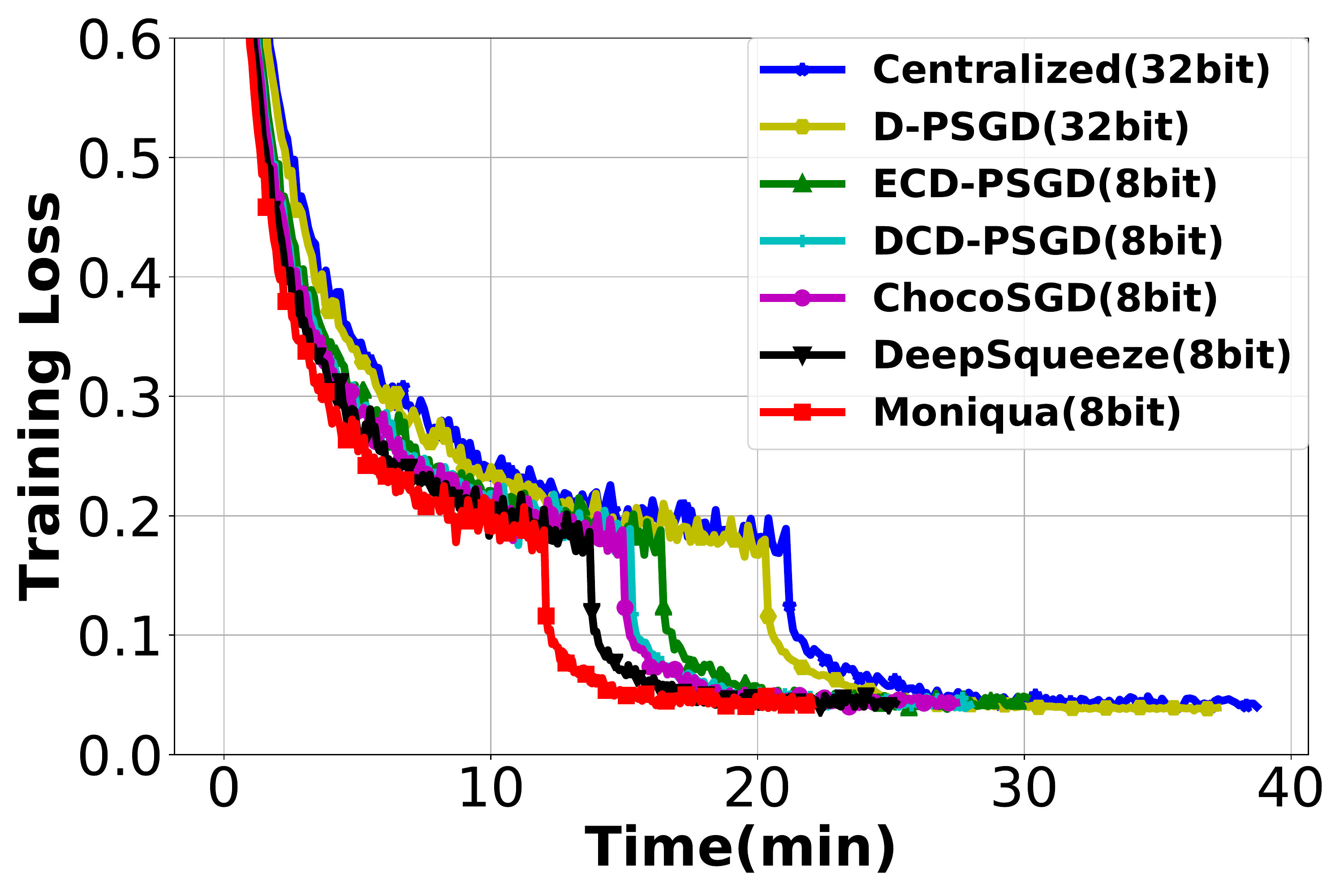}
        \label{100mbps}
    }
    \subfigure[Train Loss vs Time(s)  \newline Bandwidth=100Mbps,
    Latency=1ms]{
        \includegraphics[width=0.48\textwidth]{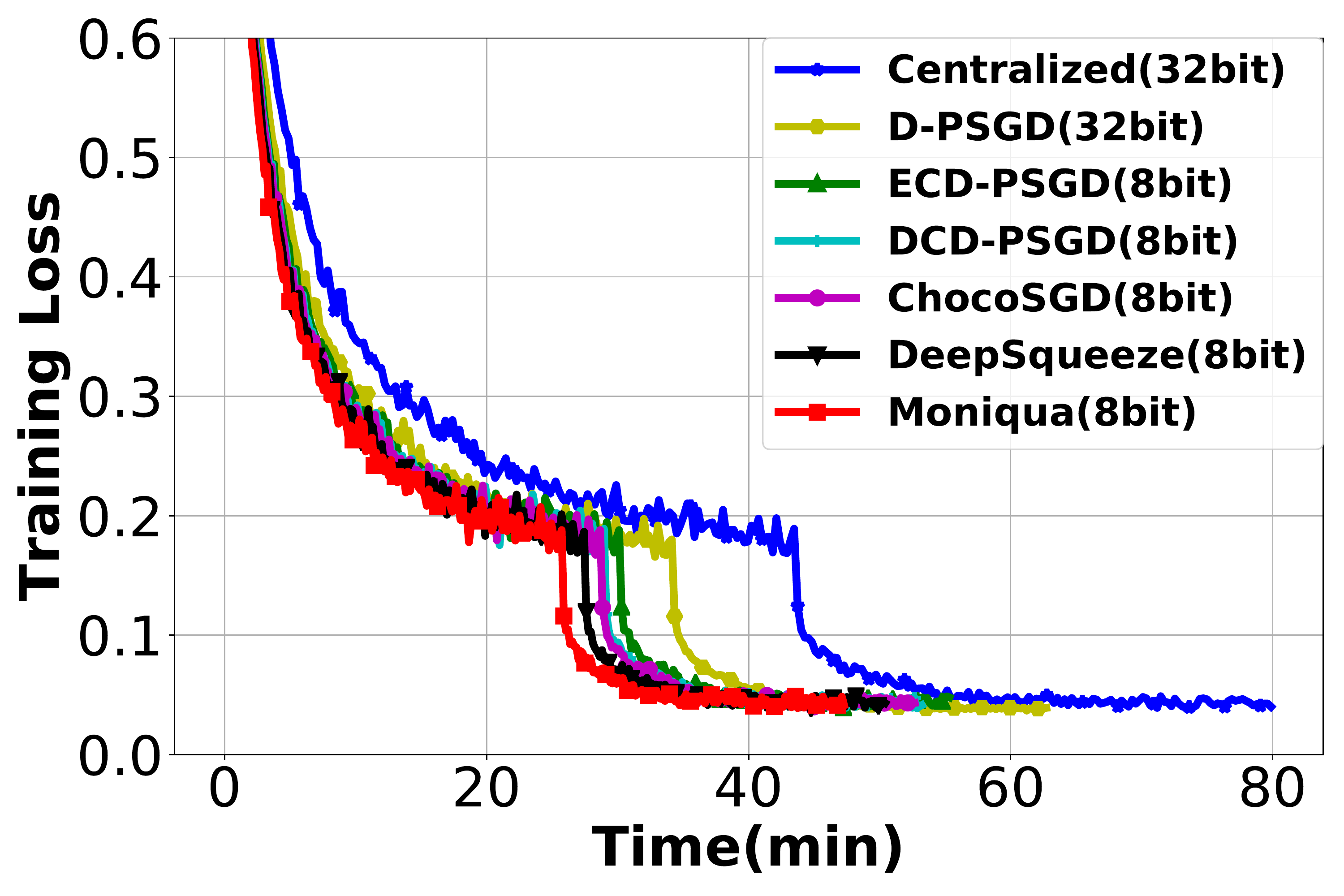}
        \label{1ms}
    }
    \subfigure[Train Loss vs Time(s)  \newline      Bandwidth=1.0Mbps,
    Latency=1.0ms]{
        \includegraphics[width=0.48\textwidth]{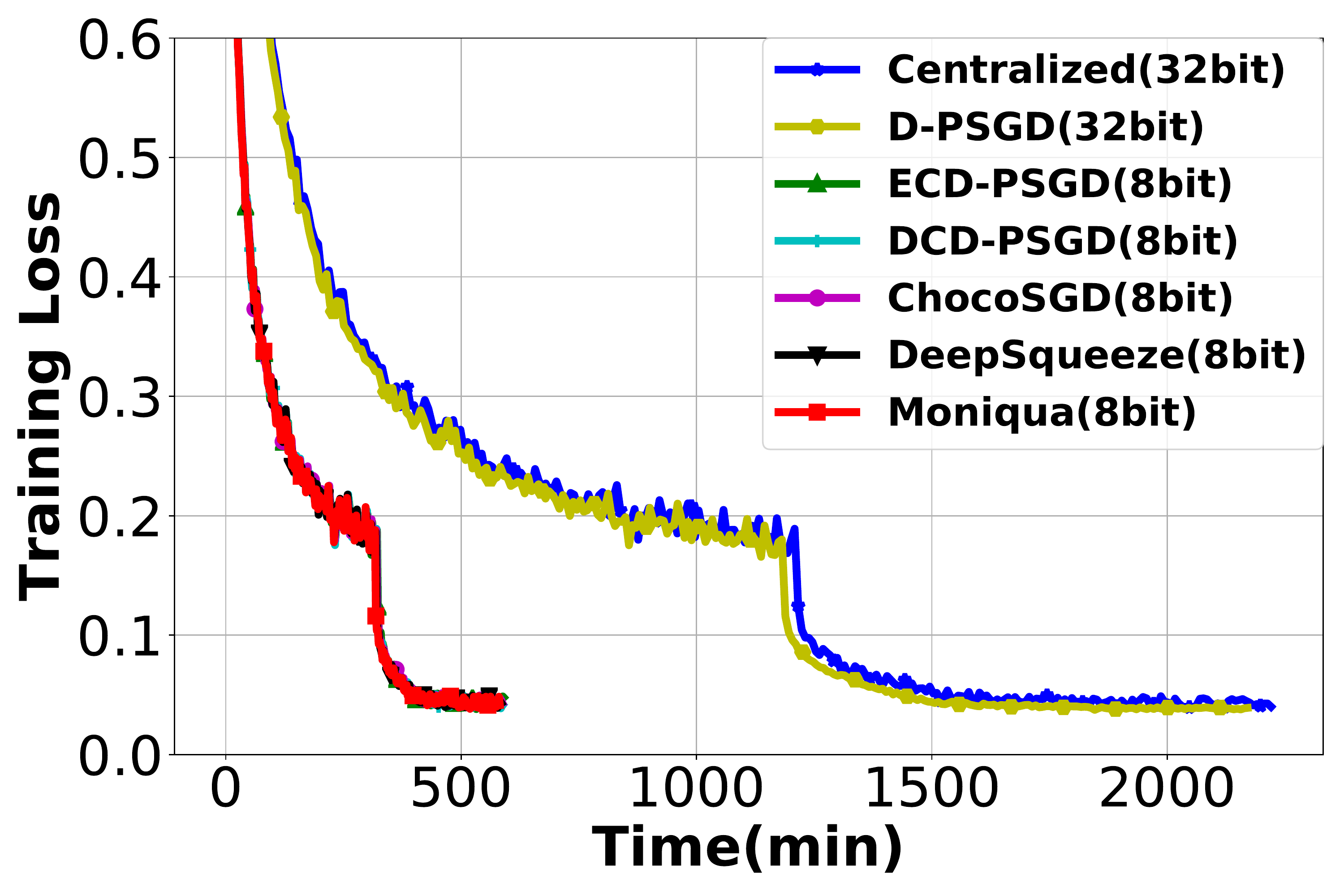}
        \label{1mbps}
    }
    \caption{Performance of different algorithms under different network configurations}
    \label{superiority}
\end{figure*}

In this section, we evaluate Moniqua empirically. First, we compare Moniqua and other quantized decentralized training algorithms' convergence under different network configurations. Second, we compare the validation performance of them under extreme bit-budget. Then we investigate Moniqua's scalability on $D^2$ and AD-PSGD.
Finally, we introduce several useful techniques for running Moniqua efficiently.

\paragraph{Setting and baselines.} All the models and training scripts in this section are implemented in PyTorch and run on Google Cloud Platform. We launch an instance as one worker, each configured with a 2-core CPU with 4 GB memory and an NVIDIA Tesla P100 GPU. We use MPICH as the communication backend. All the instances are running Ubuntu 16.04, and latency and bandwidth on the underlying network are configured using the \texttt{tc} command in Linux.
Throughout our experiments, we adopt the commonly used~\cite{gupta2015deep,li2017training} 
stochastic rounding\footnote{Since several baselines are not applicable to biased quantizer, for fair comparison we consistently use stochastic rounding (unbiased). }. 
We compare Moniqua with the following baselines:
Centralized (implemented as MPI AllReduce operation), D-PSGD \cite{lian2017can} with full-precision communication, DCD/ECD-PSGD \cite{tang2018communication}, ChocoSGD \cite{koloskova2019decentralized} and DeepSqueeze \cite{tang2019texttt}. In the experiment, we adopt the following hyperparameters for Moniqua: $\{\texttt{Momentum}=0.9,\texttt{Weight Decay}=5e-4,\texttt{Batch Size}=128,\texttt{Step Size}\footnote{Decay by a factor of 0.1 at epoch 250, 280.}=0.1,\theta_k=2.0\}$. In the extreme-bit-budget experiment, we further use adopt the average ratio $\{\gamma=5e-3\}$.

\paragraph{Wall-clock time evaluation.}\label{exp Superiority of Moniqua}
We start by evaluating the performance of Moniqua and other baseline algorithms under different network configurations. We launch 8 workers connected in a ring topology and train a ResNet20 \cite{he2016deep} model on CIFAR10 \cite{krizhevsky2014cifar}.
For all the algorithms, we quantize each parameter into 8-bit representation.

We plot our results in Figure~\ref{superiority}. 
We can see from Figures~\ref{200mbps} to \ref{100mbps} that when the network bandwidth decreases, the curves begin to separate. AllReduce and full-precision D-PSGD suffer the most, since they require a large volume of high-precision exchanged data. And from Figure~\ref{100mbps} to Figure~\ref{1ms}, when the network latency increases, AllReduce is severely delayed since it needs to transfer large volume of messages (such as handshakes between hosts to send data). On the other hand, from Figure~\ref{200mbps} to Figure~\ref{100mbps} and Figure~\ref{1ms}, curves of all the quantized baselines (DCD/ECD-PSGD, ChocoSGD and DeepSqueeze) are getting closer to Moniqua. This is because, as shown in Figure~\ref{200mbps}, the extra updating of the replicas in DCD/ECD-PSGD and ChocoSGD as well as the error tracking in DeepSqueeze counteract the benefits from accelerated communication. However, when network bandwidth decreases or latency increases, communication becomes the bottleneck and makes these algorithms diverge from centralized SGD and D-PSGD. Delay between Moniqua and quantized baselines does not vary with the network since that only depends on the their extra local computation (error tracking and replica update).
Figure~\ref{1mbps} shows an extremely poor network, and we can see that all the quantized baselines are having similar convergence speed since now network is a serious overhead.
\begin{figure}[h!]
    \subfigure[Training Loss vs Epoch (Decentralized Data)]{
    \includegraphics[width=0.48\textwidth]{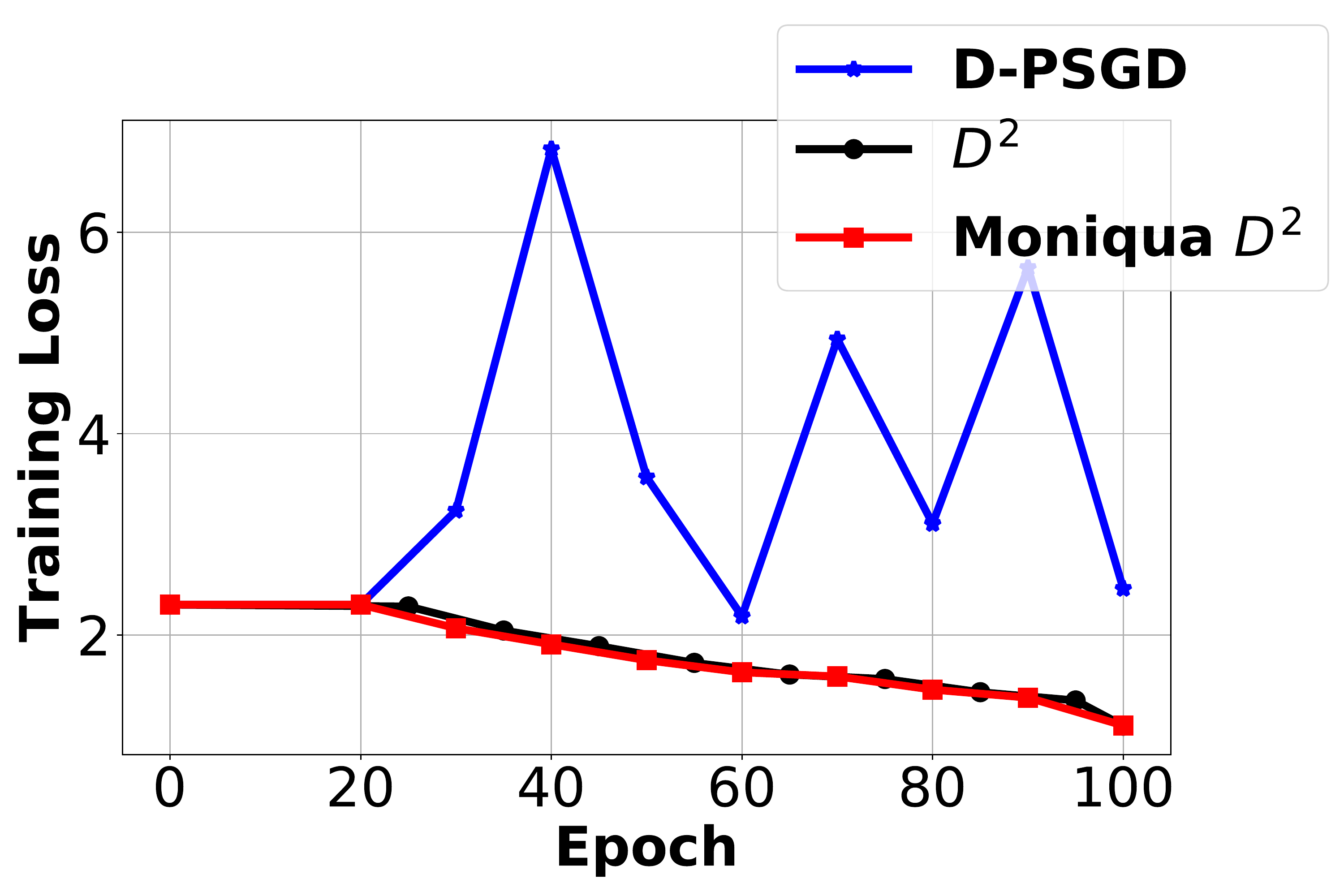}
    \label{D2compare}
    }
    \subfigure[Training Loss vs Time(s) (Asynchronous Communication)]{
    \includegraphics[width=0.48\textwidth]{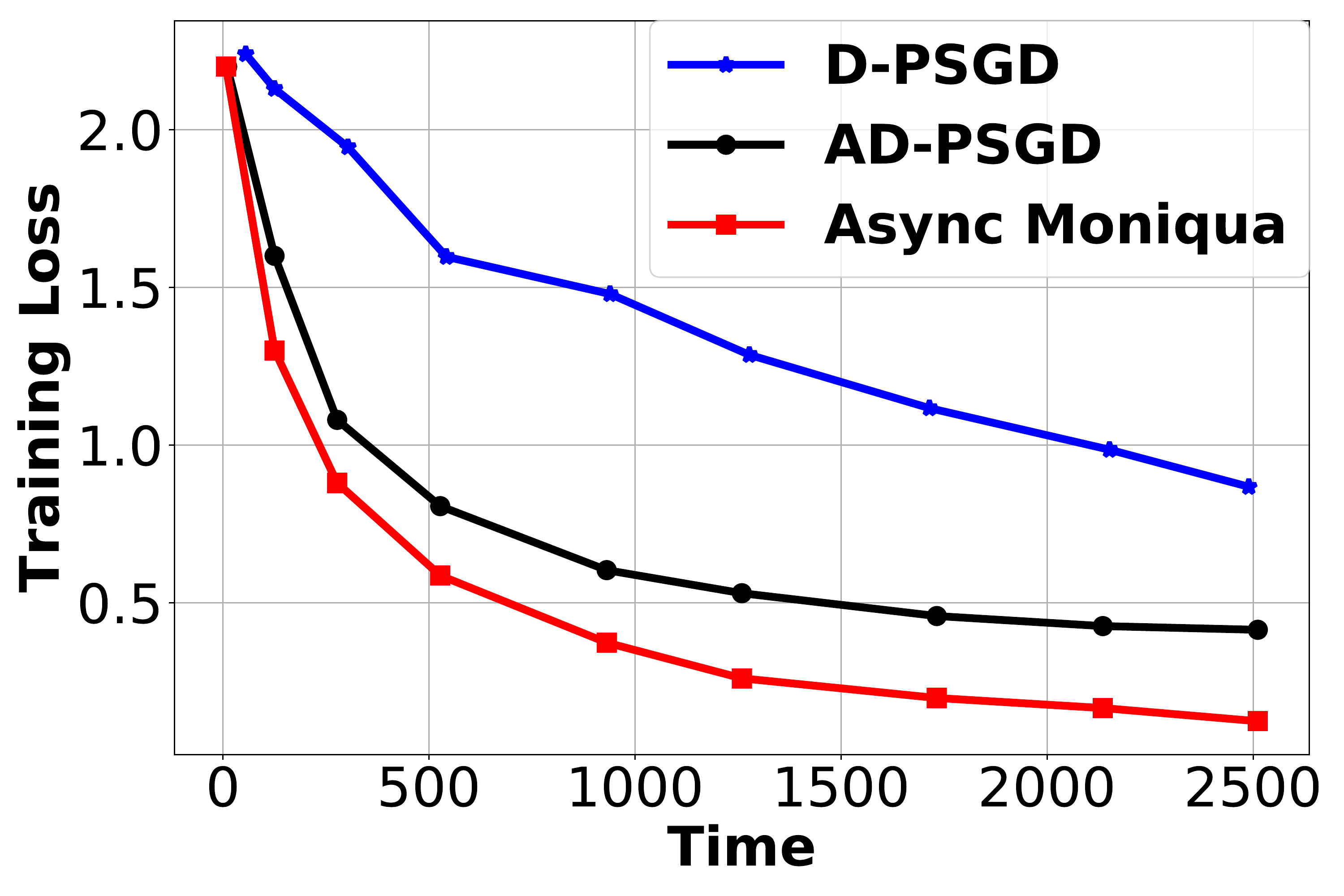}
    \label{Asynccompare}
    }
    \caption{Performance of applying Moniqua on $D^2$ and AD-PSGD}
    \label{exp_scalability}
\end{figure}
\begin{table*}[t]
\small
\label{exp aggressive table}
\caption{Final test accuracy of ResNet20 and ResNet110 on CIFAR10 trained by different algorithms. (``diverge'' means the algorithm cannot converge. ``extra memory'' means the extra memory required by different algorithms compared to full precision D-PSGD.)}
\begin{center}
\begin{tabular}{ccccccc}
\toprule
& & DCD-PSGD & ECD-PSGD & ChocoSGD & DeepSqueeze & Moniqua\\
\midrule
\multirow{3}{*}{ResNet20}  &  budget: 1bit & diverge  & diverge & $90.88\pm 0.13\%$ & $90.02\pm 0.22\%$ & $91.08\pm 0.19\%$ \\
&  budget: 2bit & diverge   & $36.32\pm 2.46\%$ & $91.09\pm 0.09\%$ & $91.12\pm 0.11\%$ & $91.13\pm 0.12\%$ \\
& extra memory (MB) & $16.48$  & $16.48$ & $16.48$ & $8.24$ & $0$ \\
\midrule
\multirow{3}{*}{ResNet110}  &  budget: 1bit & diverge   & diverge & $91.24\pm 0.21\%$ & $91.80\pm 0.27\%$ & $92.97\pm 0.23\%$ \\
&  budget: 2bit & diverge   & diverge & $93.43\pm 0.12\%$ & $92.96\pm 0.17\%$ & $93.47\pm 0.18\%$ \\
& extra memory (MB) & $103.68$  & $103.68$ & $103.68$ & $51.84$ & $0$ \\
\bottomrule
\end{tabular}
\end{center}
\end{table*}

\paragraph{Extremely low bit-budget.}
We proceed to evaluate whether Moniqua and other baselines are able to achieve state-of-the-art accuracy under extremely low bit budgets.
We train two different models: ResNet20 and ResNet110 on CIFAR10. State-of-the-art results \cite{he2016deep} show that ResNet20 can achieve test accuracy of $91.25\%$ while ResNet110 can achieve $93.57\%$.
We enforce two strict bit-budget: 1bit and 2bit (per parameter). 
We plot the final test accuracy under different algorithms in Table~\ref{exp aggressive table}.
We can see that DCD-PSGD and ECD-PSGD are generally not able to converge. Among all the other algorithms, Moniqua achieves slightly better test accuracy while requiring no additional memory. By comparison, ChocoSGD and DeepSqueeze are able to get close to state-of-the art accuracy, but at the cost of incurring substantial memory overhead.

\paragraph{Scalability.}\label{exp Scalability}
We evaluate the performance of Moniqua when applied to $D^2$ \cite{tang2018d} and AD-PSGD \cite{lian2017asynchronous}.
First, we demonstrate how applying Moniqua to $D^2$ can handle decentralized data.
We launch 10 workers, collaborating to train a VGG16 \cite{simonyan2014very} model on CIFAR10. Similar to the setting of $D^2$ \cite{tang2018d}, we let each worker have exclusive access to 1 label (of the 10 labels total in CIFAR10). In this way, the data variance among workers is maximized. 
We plot the results in Figure~\ref{D2compare}. We observe that applying Moniqua on $D^2$ does not affect the convergence rate while D-PSGD can no longer converge because of the outer variance. Here we omit the wall clock time comparison since the communication volume is the same in comparison of Moniqua and Centralized algorithm in Figure~\ref{superiority}. 

Next, we evaluate Moniqua on AD-PSGD. We launch 6 workers organized in a ring topology, collaborating to train a ResNet110 model on CIFAR10. We set the network bandwidth to be 20Mbps and latency to be 0.15ms. We plot the results in Figure~\ref{Asynccompare}. We can see that both AD-PSGD and asynchronous Moniqua outperform D-PSGD. Besides, Moniqua outperforms AD-PSGD in that communication is reduced, which is aligned with the intuition and theory.

\paragraph{Choosing $\theta$ empirically.}
We can see that the $\theta$ chosen will largely affect the running of Moniqua.
In practice, there are several methods to effectively tune $\theta$. The \textbf{first} is to directly compute $\theta$ via its expression. Specifically, we could first run a few epochs and keep track of the infinity norm of the gradient and then use expression in Theorem~\ref{Moniqua convergence rate} to obtain $\theta$. Note that gradient is usually decreasing in magnitude as algorithm proceeds. In general the computed $\theta$ can be used throughout the training. 
The \textbf{second} method is to treat $\theta$ as a hyperparameter and use standard methods such as random search or grid search \cite{bergstra2012random} to tune $\theta$ until we find the correct $\theta$.
The \textbf{third} method is to add verification. For instance, consider using stochastic rounding with quantization step being $\delta$. Suppose we have $x\in\mathbb{R}$ and need to send it to machine $M$ with $y$. If $|x-y|<\theta$, then if we send $\mathcal{Q}_\delta(x/\delta)\bmod \theta/\delta$ to $M$, it will recover $\mathcal{Q}_\delta(x/\delta)$ based on $y$. In addition, we can also send $H(\mathcal{Q}_\delta(x/\delta))$, where $H$ is a hash function that takes the un-modded vector. When $M$ recovers $\mathcal{Q}_\delta(x/\delta)$, it can detect whether the thing it recovered has the correct hash. If the $\theta$ is mistakenly chosen, $M$ will detect any errors with high probability \cite{al2003certificateless}.
Note that compared to the model parameters, the output of hash function will not cause any overhead in general.

In the experiments of previous subsections, we mainly use the first method, which is sufficient for a good $\theta$. The second method is a standard tuning protocol, but we do not usually use it in practice. The third method is optional to further guarantee the  correctness of $\theta$ with little cost. Besides, we found constant $\theta$(s) suffice to perform well in the experiments, and thus in practice we usually do not need to modify $\theta$ in each iteration.

\paragraph{More efficient Moniqua.}\label{exp quantization limit}
There are two techniques we have observed to improve the performance of Moniqua when using stochastic rounding: $Q_\delta(\*x)=\delta\lfloor \frac{\*x}{\delta} + \*u \rfloor$ (where $u$ is uniformly sampled from $[0,1]$), $\forall \*x\in\mathbb{R}^d$.
The \textbf{first} is to use \emph{shared randomness}, in which the same random seed is used for stochastic rounding on all the workers. That is, if two workers are exchanging tensors $\*x$ and $\*y$ respectively, then the floored tensors $\lfloor \frac{\*x}{\delta} + \*u \rfloor$ and $\lfloor \frac{\*y}{\delta} + \*u \rfloor$ they send use the \emph{same} randomly sampled value $u$. This provably reduces the error due to quantization (more details are in the supplementary material).
The \textbf{second} technique is to use a standard entropy compressor like \texttt{bzip}
to further compress the communicated tensors. This can help further reduce the number of bits because the modulo operation in Moniqua can introduce some redundancy in the higher-order bits, which a traditional compression algorithm can easily remove.

\section{Conclusions}
In this paper we propose Moniqua, a simple unified method of quantizing the communication in decentralized training algorithms.
Theoretically, Moniqua supports biased quantizer and non-convex problems, while enjoying the same asymptotic convergence rate as full-precision-communication algorithms without incurring storage or computation overhead. Empirically, we observe Moniqua converges faster than other related algorithms with respect to wall clock time.
Additionally, Moniqua is robust to very low bits-budget.

\bibliography{main}

\newpage
\begin{center}
{\huge\textbf{Supplementary Material}}
\end{center}

\appendix

\section{Overview}
This supplementary material contains proof to all the theoretical results. It is organized as follows: 
In Section~\ref{modulo section}, we analyze how to work with Modulo and quantization, as proofs to Lemma~\ref{modulo_lemma} and Lemma~\ref{modulo_numerical_lemma} in the paper.
In Section~\ref{randomness}, we provably explain why using shared randomness in communication with stochastic rounding can improve performance. 
In Section~\ref{lowerbound}, we illustrate why directly quantizing communication in D-PSGD fails to converge asymptotically, as a proof to Theorem~\ref{quadratic}. 
In Section~\ref{matrix}, we introduce some useful tools of modeling communication as a Markov Chain for the rest of the proof (part of the intuition is illustrated in the paper). We recommend to go through this before getting into Section~\ref{Moniqua} to \ref{Async Moniqua}. Finally we will provide proof to Theorem~\ref{Moniqua convergence rate} to \ref{thmMoniquaAD-PSGD} from Section~\ref{Moniqua} to \ref{Async Moniqua}.
\section{Modulo Operation with Quantization}\label{modulo section}
\paragraph{Proof to Lemma~\ref{modulo_lemma}.}
\begin{proof}
Rewrite $x$ and $y$ as
\begin{align*}
    x &= N_xa + r_x, -\frac{a}{2}\leq r_x < \frac{a}{2}\\
    y &= N_ya + r_y, -\frac{a}{2}\leq r_y < \frac{a}{2}
\end{align*}
where $N_x$, $N_y\in\mathbb{Z}$ then, 
\begin{align*}
    \text{LHS} &= (r_x-r_y)\bmod a\\
    \text{RHS} &= ((N_x-N_y)a + r_x-r_y)\bmod a= (r_x-r_y)\bmod a = \text{LHS}
\end{align*}
Thus we complete the proof.
\end{proof}
\paragraph{Proof to Lemma~\ref{modulo_numerical_lemma}.}
\begin{proof}
We start from
\begin{align*}
    &  B_\theta\mathcal{Q}_\delta\left(\frac{x}{ B_\theta}\bmod 1\right) -  B_\theta\left(\frac{x}{ B_\theta}\bmod 1\right) + x
    =  B_\theta\mathcal{Q}_\delta\left(\frac{x}{ B_\theta}\bmod 1\right) -  B_\theta\left(\frac{x}{ B_\theta}\bmod 1\right) + x -y + y
\end{align*}
If $B_\theta$ is sufficiently large such that $ B_\theta\geq 2\theta + 2\delta B_\theta > 2|x-y| + 2\delta B_\theta$, we could put a "$\bmod B_\theta$" to the first four terms as follows:
\begin{align*}
    & B_\theta\mathcal{Q}_\delta\left(\frac{x}{ B_\theta}\bmod 1\right) -  B_\theta\left(\frac{x}{ B_\theta}\bmod 1\right) + x -y + y\\
    = & \left( B_\theta\mathcal{Q}_\delta\left(\frac{x}{ B_\theta}\bmod 1\right) -  B_\theta\left(\frac{x}{ B_\theta}\bmod 1\right) + x -y\right)\bmod B_\theta + y\\
    \overset{\text{Lemma 1}}{=} & \left[\left( B_\theta\mathcal{Q}_\delta\left(\frac{x}{ B_\theta}\bmod 1\right) -  B_\theta\left(\frac{x}{ B_\theta}\bmod 1\right) + x\right)\bmod B_\theta -y\bmod B_\theta\right]\bmod B_\theta + y\\
    \overset{\text{Lemma 1}}{=} & \left\{\left[ B_\theta\mathcal{Q}_\delta\left(\frac{x}{ B_\theta}\bmod 1\right)\bmod B_\theta - \left( B_\theta\left(\frac{x}{ B_\theta}\bmod 1\right) - x\right)\bmod B_\theta\right]\bmod B_\theta -y\bmod B_\theta\right\}\bmod B_\theta + y
\end{align*}
Note that the term $\left( B_\theta\left(\frac{x}{ B_\theta}\bmod 1\right) - x\right)\bmod B_\theta=0$, then we can proceed as:
\begin{align*}
    & \left\{\left[ B_\theta\mathcal{Q}_\delta\left(\frac{x}{ B_\theta}\bmod 1\right)\bmod B_\theta - \left( B_\theta\left(\frac{x}{ B_\theta}\bmod 1\right) - x\right)\bmod B_\theta\right]\bmod B_\theta -y\bmod B_\theta\right\}\bmod B_\theta + y\\
    = & \left( B_\theta\mathcal{Q}_\delta\left(\frac{x}{ B_\theta}\bmod 1\right)\bmod B_\theta -y\bmod B_\theta\right)\bmod B_\theta + y\\
    = & \left( B_\theta\mathcal{Q}_\delta\left(\frac{x}{ B_\theta}\bmod 1\right) -y\right)\bmod B_\theta + y
\end{align*}
By moving $x$ to the right side we obtain
\begin{align*}
    \left|\left( B_\theta\mathcal{Q}_\delta\left(\frac{x}{ B_\theta}\bmod 1\right) -y\right)\bmod B_\theta + y - x\right| = \left| B_\theta\mathcal{Q}_\delta\left(\frac{x}{ B_\theta}\bmod 1\right) -  B_\theta\left(\frac{x}{ B_\theta}\bmod 1\right)\right| \leq \delta B_\theta
\end{align*}
That completes the proof.
\end{proof}
\section{Shared Randomness}\label{randomness}
In this section, we provide a theoretical explanation why using shared randomness in the stochastic rounding is able to improve the performance.
Without the loss of generality, in the following analysis, we let the quantization step associated with stochastic rounding quantizer $\mathcal{Q}_\delta$ be $\delta=1$. For any $z\in\mathbb{R}$ quantized using $\mathcal{Q}_\delta$, let $z_f = z - \lfloor z\rfloor$, the variance of quantization error can be expressed as
\begin{equation}\label{quantization variance}
\mathbb{E}|\mathcal{Q}_\delta(z) - z|^2 = (1 - z_f)(-z_f)^2 + z_f(1 - z_f)^2 = z_f(1 - z_f)
\end{equation}
Note that in Moniqua, the term asssociate with quantization error is
\begin{align*}
\mathbb{E}\left\|(\*q_{k,j} - \*x_{k,j}) - (\*q_{k,i} - \*x_{k,i})\right\|^2
\end{align*}
We now show for $\forall x, y\in\mathbb{R}$
\begin{displaymath}
\mathbb{E}\left|(\mathcal{Q}_\delta(x) - x) - (\mathcal{Q}_\delta(y) - y)\right|^2 = \mathbb{E}\left|\mathcal{Q}_\delta(y - x) - (y - x)\right|^2
\end{displaymath}
With out the loss of generality, let $x - \lfloor x\rfloor \leq y - \lfloor y\rfloor$. Let $x_f = x - \lfloor x\rfloor$ and $y_f = y - \lfloor y\rfloor$, then
\begin{align*}
    \lfloor x + u\rfloor = \lfloor x\rfloor & \hspace{1em}\text{and}\hspace{1em} \lfloor y + u\rfloor = \lfloor y\rfloor, \text{with probability}\hspace{1em} \lceil y\rceil - y\\
     \lfloor x + u\rfloor = \lceil x\rceil & \hspace{1em}\text{and}\hspace{1em} \lfloor y + u\rfloor = \lceil y\rceil, \text{with probability}\hspace{1em} x - \lfloor x\rfloor\\
      \lfloor x + u\rfloor = \lfloor x\rfloor & \hspace{1em}\text{and}\hspace{1em} \lfloor y + u\rfloor = \lceil y\rceil, \text{with probability}\hspace{1em} (\lceil x\rceil - x) - (\lceil y\rceil - y)
\end{align*}
Then we have
\begin{align*}
    & \mathbb{E}\left|(\mathcal{Q}_\delta(x) - x) - (\mathcal{Q}_\delta(y) - y)\right|^2\\
    = & \mathbb{E}\left|\left(\delta\left\lfloor\frac{x}{\delta} + u\right\rfloor - x\right) - \left(\delta\left\lfloor\frac{y}{\delta} + u\right\rfloor - y\right)\right|^2\\
    = & (\lceil y\rceil - y)((\lfloor x\rfloor - x) - (\lfloor y\rfloor - y))^2 + (x - \lfloor x\rfloor)((\lceil x\rceil - x) - (\lceil y\rceil - y))^2\\
    & + ((\lceil x\rceil - x) - (\lceil y\rceil - y))((\lfloor x\rfloor - x) - (\lceil y\rceil - y))^2\\
    = & (1 - y_f)(x_f - y_f)^2 + (x_f)(x_f - y_f) + (y_f - x_f)(y_f - x_f - 1)^2\\
    = & (1 - y_f + x_f)(y_f - x_f)^2 + (y_f - x_f)(y_f - x_f - 1)^2\\
    = & (1 - y_f + x_f)(y_f - x_f)\\
    = & \mathbb{E}\left|\mathcal{Q}_\delta(y - x) - (y - x)\right|^2
\end{align*}
The last equality holds due to equation~\ref{quantization variance}. Next, for $\forall \*x, \*y\in\mathbb{R}^d$ let
\begin{align*}
&\*\Delta = \*y - \*x\\
&\*r = \mathcal{Q}_\delta(\*\Delta) - \*\Delta
\end{align*}
And let $\*r_h$ denote $h$-th entry of $\*r$, let $\*\Delta_h$ denote $h$-th entry of $\*\Delta$. We obtain
\begin{align*}
\*r_h	= & \mathcal{Q}_\delta(\*\Delta_h) - \*\Delta_h \\
= & \delta\begin{cases} -\frac{\*\Delta_h}{\delta} + \left\lfloor\frac{\*\Delta_h}{\delta}\right\rfloor + 1, &p_t\leq \frac{\*\Delta_h}{\delta} - \left\lfloor\frac{\*\Delta_h}{\delta}\right\rfloor \cr  -\frac{\*\Delta_h}{\delta} + \left\lfloor\frac{\*\Delta_h}{\delta}\right\rfloor, &\text{otherwise}\end{cases}\\
= & \delta\begin{cases} -q + 1, &p_t\leq q \cr  -q, &\text{otherwise}\end{cases}
\end{align*}
where
\begin{displaymath}
q = \frac{\*\Delta_h}{\delta} - \left\lfloor\frac{\*\Delta_h}{\delta}\right\rfloor, q\in [0,1]
\end{displaymath}
Based on that, we have
\begin{align*}
\mathbb{E}\left[\*r_h^2\right] \leq & \delta^2((-q+1)^2q + (-q)^2(1-q))\\
= & \delta^2q(1-q)\\
\leq & \delta^2\min\{q, 1-q\}
\end{align*}
Since $\min\{q,1-q\} \leq \left|\frac{\*x_h}{\delta}\right|$, we have
\begin{displaymath}
\mathbb{E}\left[\*r_h^2\right] \leq \delta^2\left|\frac{\*\Delta_h}{\delta}\right| \leq \delta\left|\*\Delta_h\right|
\end{displaymath}
Summing over the index $h$ yields,
\begin{displaymath}
\mathbb{E}\left\|\*r\right\|_2^2 \leq \delta\mathbb{E}\left\|\*\Delta\right\|_1 \leq \sqrt{d}\delta\mathbb{E}\left\|\*\Delta\right\|_2
\end{displaymath}
Pushing back $\*x$ and $\*r$, we have
\begin{displaymath}
\mathbb{E}\left\|\mathcal{Q}_\delta(\*y - \*x) - (\*y - \*x)\right\|^2\leq \sqrt{d}\delta\mathbb{E}\left\|\*y - \*x\right\| = \sqrt{d}\delta\mathbb{E}\left\|\*x - \*y\right\|
\end{displaymath}
Putting it back we have
\begin{displaymath}
\mathbb{E}\left\|(\mathcal{Q}_\delta(\*x) - \*x) - (\mathcal{Q}_\delta(\*y) - \*y)\right\|^2 \leq \sqrt{d}\delta\mathbb{E}\left\|\*x - \*y\right\|
\end{displaymath}
Now we can see that the error term is bounded by the distance of two quantized tensor, which, in decentralized training, refers to the distance between two models on adjacent workers. In such a way, the error bound can be reduced since the workers are getting close to each other.
\section{Why Naive Quantization Fails in D-PSGD (Proof to Theorem~\ref{quadratic})}\label{lowerbound}
The update rule of naive quantization on D-PSGD is
\begin{align*}
\*x_{k+1,i} = \*x_{k,i}\*W_{ii} + \sum_{j=1,j\neq i}^{n}\mathcal{Q}_\delta(\*x_{k,j})\*W_{ji} - \alpha_k\*{\tilde{g}}_{k,i} =
\*x_{k,i} + \sum_{j=1,j\neq i}^{n}(\mathcal{Q}_\delta(\*x_{k,j})-\*x_{k,i})\*W_{ji} - \alpha_k\*{\tilde{g}}_{k,i}
\end{align*}
where $\alpha_k$ is allowed to vary with any policy. Let
\begin{align*}
\*X_k &= \left[\*x_{k,1}, \cdots, \*x_{k,n}\right]\in\mathbb{R}^{d\times n}\\
\*\Omega_k &= \left[\sum_{j\neq 1}\*W_{j1}\left(\mathcal{Q}_\delta(\*x_{k,j}) - \*x_{k,1}\right), \cdots, \sum_{j\neq n}\*W_{jn}\left(\mathcal{Q}_\delta(\*x_{k,j}) - \*x_{k,n}\right)\right]\in\mathbb{R}^{d\times n}\\
\*{\tilde{G}}_k &= \left[\*{\tilde{g}}_{k,1}, \cdots, \*{\tilde{g}}_{k,n}\right]\in\mathbb{R}^{d\times n}
\end{align*}
by rewritting the update rule, we obtain
\begin{align*}
\*X_{k+1} = \*X_k + \*\Omega_k - \alpha_k\*{\tilde{G}}_k
\end{align*}
Let $\*Y_k = \*X_k - \*x^*\*{1}^\top$, and considering the fact that $\nabla f(\*x)=\*x-\delta\*1/2=\*x-\*x^*$, we can rewrite the update rule as
\begin{align*}
\*Y_{k+1}\*e_i = \*Y_k\*e_i + \*\Omega_k\*e_i - \alpha_k\*Y_k\*e_i + \alpha_k\left(\*{\tilde{G}}_k-\*G_k\right)\*e_i
\end{align*}
where $\left(\*{\tilde{G}}_k-\*G_k\right)$ denotes variance in the gradient sampling.

Suppose that by using the update rule of naive quantization, worker $i$ converges to $\*x^*$. Then there must exist a $K$ such that $\forall k \geq K$, 
\begin{equation}\label{theorem1_assumption}
\mathbb{E}\left\|\*Y_{k+1}\*e_i\right\|^2 \leq \mathbb{E}\left\|\*Y_k\*e_i\right\|^2 < \frac{\phi^2\delta^2}{8(1+\phi^2)}\end{equation}
Next we show that this assumption lets us derive a contradiction. Firstly, considering the property of linear quantizer,
\begin{align*}
\frac{\delta^2}{4} \leq \mathbb{E}\left\|\mathcal{Q}_\delta(\*x_{k,i}) - \*x^*\right\|^2 \leq 2\mathbb{E}\left\|\mathcal{Q}_\delta(\*x_{k,i}) - \*x_{k,i}\right\|^2 + 2\mathbb{E}\left\|\*x_{k,i} - \*x^*\right\|^2
\end{align*}
As a result
\begin{align*}
\mathbb{E}\left\|\mathcal{Q}_\delta(\*x_{k,i}) - \*x_{k,i}\right\|^2 \geq \frac{\delta^2}{8} - \frac{\phi^2\delta^2}{8(1+\phi^2)} = \frac{\delta^2}{8(1+\phi^2)}
\end{align*}
Since $\mathcal{Q}_\delta$ is unbiased, that means $\mathbb{E}[\mathcal{Q}_\delta(\*x)-\*x]=0$, then we have
\begin{align*}
&\mathbb{E}\left\|\*\Omega_k\*e_i\right\|^2\\
= & \mathbb{E}\left\|\sum_{j\neq i}\*W_{ji}\left(\mathcal{Q}_\delta(\*x_{k,j}) - \*x_{k,i}\right)\right\|^2\\
= & \sum_{j\in\mathcal{N}_i}\*W_{ji}^2\mathbb{E}\left\|\left(\mathcal{Q}_\delta(\*x_{k,j}) - \*x_{k,i}\right)\right\|^2 + \sum_{m\neq n\neq i}\mathbb{E}\left\langle \left(\mathcal{Q}_\delta(\*x_{k,m}) - \*x_{k,i}\right)\*W_{mi}, \left(\mathcal{Q}_\delta(\*x_{k,n}) - \*x_{k,i}\right)\*W_{ni}\right\rangle\\
\geq & \phi^2\sum_{j\in\mathcal{N}_i}\mathbb{E}\left\|\left(\mathcal{Q}_\delta(\*x_{k,j}) - \*x_{k,i}\right)\right\|^2 + \sum_{m\neq n\neq i}\mathbb{E}\left\langle \left(\mathcal{Q}_\delta(\*x_{k,m}) - \*x_{k,i}\right)\*W_{mi}, \left(\mathcal{Q}_\delta(\*x_{k,n}) - \*x_{k,i}\right)\*W_{ni}\right\rangle\\
\overset{(*)}{=} & \phi^2\sum_{j\in\mathcal{N}_i}\mathbb{E}\left\|\mathcal{Q}_\delta(\*x_{k,j}) - \*x_{k,i}\right\|^2\\
\geq & \frac{\phi^2\delta^2}{8(1+\phi^2)}
\end{align*}
where step $(*)$ holds due to unbiased quantizer. Putting it back to the update rule, we obtain
\begin{align*}
&\mathbb{E}\left\|\*Y_{k+1}\*e_i\right\|^2\\
= & \mathbb{E}\left\|\left(\*Y_k + \*\Omega_k - \alpha_k\*Y_k + \alpha_k\left(\*{\tilde{G}}_k-\*G_k\right)\right)\*e_i\right\|^2\\
\overset{(*)}{=} & \mathbb{E}\left\|(1-\alpha_k)\*Y_k\*e_i\right\|^2 + \mathbb{E}\left\|\*\Omega_k\*e_i\right\|^2 + \mathbb{E}\left\|\alpha_k\left(\*{\tilde{G}}_k-\*G_k\right)e_i\right\|^2\\
\geq & \mathbb{E}\left\|\*\Omega_k\*e_i\right\|^2\\
\geq & \frac{\phi^2\delta^2}{8(1+\phi^2)}
\end{align*}
where cross terms in the $(*)$ step are all 0 due to the unbiased quantizer and unbiased sampling of the gradient. Her we obtain the contradictory that $\frac{\phi^2\delta^2}{8(1+\phi^2)}\leq\mathbb{E}\left\|\*x_{k+1}-\*x^*\right\|^2<\frac{\phi^2\delta^2}{8(1+\phi^2)}$. That being said, for $\forall k,i$
\begin{align*}
\mathbb{E}\left\|\*x_{k,i}-\*x^*\right\|^2 = \mathbb{E}\left\|\nabla f(\*x_{k,i})\right\|^2 \geq \frac{\phi^2\delta^2}{8(1+\phi^2)}
\end{align*}
Thus we complete the proof.
\section{A Markov Chain Analysis on the Communication}\label{matrix}
To better understand how the parallel workers reach consensus over a communication matrix, in this section we use theory from the analysis of Markov Chains to obtain some useful lemmas for proof of Moniqua on D-PSGD and AD-PSGD.

Since the communication matrix $\*W$ is doubly stochastic (each row and column sum to 1), it has the same structure as the transition matrix of a Markov Chain with $\frac{\*1}{n}$ as its the stationary
distribution $\left(\*W\frac{\*1}{n}=\frac{\*1}{n}\right)$. Now let $\tmix{}$ and $d(t)$ denote the mixing time and maximal distance between initial state and stationary distribution as defined in Markov Chain theory.\footnote{Here we are using notation from Chapter 4.5 of \textit{Markov Chains and Mixing Times} (Levin 2009), available at \url{https://pages.uoregon.edu/dlevin/MARKOV/markovmixing.pdf}}
\subsection{D-PSGD}
In D-PSGD, the communication matrix is fixed during the training. That makes it perfectly aligned with the structure of a Markov Chain. As a result, we obtain the following lemma:
\begin{lemma}\label{lemmafixedmc}
\begin{displaymath}
\left\|\*W^t\left(I - \frac{\*1\*1^\top}{n}\right)\right\|_1\leq 2\cdot 2^{-\left\lfloor\frac{t}{\tmix{}}\right\rfloor}
\end{displaymath}
\end{lemma}
\begin{proof}
For $\forall \*x\in\mathbb{R}^d$, let $\*u\in\mathbb{R}^d$ be such a vector that every entry of $\*u$ is the positive entry of $\*x$ and 0 otherwise. Let $\*v\in\mathbb{R}^d$ be such a vector that every entry of $\*v$ is the absolute value of negative entry of $\*x$ and 0 otherwise. The setting above means $\*x=\*u-\*v$. For example, 
\begin{align*}
\*x & = [2, -1]^\top\\
\*u & = [2, 0]^\top\\
\*v & = [0, 1]^\top
\end{align*}
And we have
\begin{align*}
    & \left\|\*W^t\left(\*I - \frac{\*1\*1^\top}{n}\right)\*x\right\|_1\\
    = & \left\|\*W^t\left(\*I - \frac{\*1\*1^\top}{n}\right)(\*u-\*v)\right\|_1\\
    \leq & \left\|\*W^t\left(\*I - \frac{\*1\*1^\top}{n}\right)\*u\right\|_1 + \left\|\*W^t\left(\*I - \frac{\*1\*1^\top}{n}\right)\*v\right\|_1\\
    = & \*1^\top \*u\left\|\*W^t\frac{\*u}{\*1^\top \*u} - \frac{\*1}{n}\right\|_1 + \*1^\top \*v\left\|\*W^t\frac{\*v}{\*1^\top \*v} - \frac{\*1}{n}\right\|_1\\
    \leq & 2(\*1^\top \*u + \*1^\top \*v)d(t)\\
    \leq & 2d(t)\left\|\*x\right\|_1\\
\end{align*}
Considering the definition of L1-norm, we have
\begin{displaymath}
    \left\|\*W^t\left(\*I - \frac{\*1\*1^\top}{n}\right)\right\|_1=\max\frac{\left\|\*W^t\left(\*I - \frac{\*1\*1^\top}{n}\right)\*x\right\|_1}{\left\|\*x\right\|_1}\leq 2d(t)
\end{displaymath}
According to a well-known results on the theory of Markov Chains,\footnote{Again, see \textit{Markov Chains and Mixing Times} for more details.} $d(l\tmix{})\leq 2^{-l}$ holds for any non-negative integer $l$, so we have
\begin{displaymath}
    \left\|\*W^t\left(\*I - \frac{\*1\*1^\top}{n}\right)\right\|_1\leq 2d(t) \leq 2d\left(\frac{t}{\tmix{}}\cdot \tmix{}\right) \leq 2d\left(\left\lfloor\frac{t}{\tmix{}}\right\rfloor \tmix{}\right) \leq 2\cdot 2^{-\left\lfloor\frac{t}{\tmix{}}\right\rfloor}
\end{displaymath}
That completes the proof.
\end{proof}
Additionally, based on standard results in the theory of reversible Markov Chains, we also have\footnote{Detailed analysis and proofs of this result can be found in chapter 12.2 of \textit{Markov Chains and Mixing Times}.}
\begin{align*}
\tmix{}\leq\log\left(\frac{1}{\frac{1}{4}\cdot\frac{1}{n}}\right)\frac{1}{1-\rho}\leq \frac{\log(4n)}{1-\rho}.
\end{align*}

\subsection{AD-PSGD}
Note that unlike D-PSGD, here $\*W_k$ can be different at each update step and usually each individually have spectral radius $\rho = 1$, so we can't expect to get a bound in terms of a bound on the spectral gap as we did in Theorems~\ref{Moniqua convergence rate} and~\ref{thmMoniquaD2}.
Instead, we require the following condition, which is inspired by the literature on Markov chain Monte Carlo methods: for some constant $\tmix{}$ (here $\tmix{}$ is the same as $\tmix{}$ in the paper) and for any $k$ and any non-negative vector $\*\mu \in \mathbb{R}^d$ such that $\*1^\top \*\mu = 1$, it must hold that
\[
    \left\| \left( \prod_{i=1}^{\tmix{}} \*W_{k+i} \right) \*\mu - \frac{\*1}{n} \right\|_1 \le \frac{1}{2}.
\]
We call this constant $\tmix{}$ because it is effectively the \emph{mixing time} of the time-inhomogeneous Markov chain with transition probability matrix $\*W_k$ at time $k$.
Note that this condition is more general than those used in previous work on AD-PSGD because it does not require that the $\*W_k$ are sampled independently or in an unbiased manner.
Based on the above analysis, we can prove the following lemma, which is analogous to the lemma used in the synchronous case.
\begin{lemma}
For any $k\geq 0$ and for any $b\geq a\geq 0$, there exists $\tmix{}$ such that
\begin{align*}
\left\|\prod_{q=a}^{b}\*W_q\left(\*I - \frac{\*1\*1^\top}{n}\right)\right\|_1 \leq 2\cdot 2^{-\left\lfloor\frac{b-a+1}{\tmix{}}\right\rfloor}
\end{align*}
\end{lemma}
\begin{proof}
Note that for any $\*x\in\mathbb{R}^d$, and let $\*u$ and $\*v$ be two vectors having same definition as in Lemma~\ref{lemmafixedmc} with respect to $\*x$, then we have for any $k$
\begin{align*}
& \left\|\prod_{q=1}^{\tmix{}}\*W_{q+k}\left(\*I - \frac{\*1\*1^\top}{n}\right)\*x\right\|_1\\
= & \left\|\prod_{q=1}^{\tmix{}}\*W_{q+k}\left(\*I - \frac{\*1\*1^\top}{n}\right)(\*u-\*v)\right\|_1\\
\leq & \left\|\prod_{q=1}^{\tmix{}}\*W_{q+k}\left(\*I - \frac{\*1\*1^\top}{n}\right)\*u\right\|_1 + \left\|\prod_{q=1}^{\tmix{}}\*W_{q+k}\left(\*I - \frac{\*1\*1^\top}{n}\right)\*v\right\|_1\\
= & \*1^\top \*u\left\|\prod_{q=1}^{\tmix{}}\*W_{q+k}\frac{\*u}{\*1^\top \*u} - \frac{\*1}{n}\right\|_1 + \*1^\top \*v\left\|\prod_{q=1}^{\tmix{}}\*W_{q+k}\frac{\*v}{\*1^\top \*v} - \frac{\*1}{n}\right\|_1\\
\leq & \frac{1}{2}(\*1^\top \*u + \*1^\top \*v)\\
\leq & \frac{1}{2}\left\|\*x\right\|_1
\end{align*}
Considering the definition of the induced $\ell_1$ operator norm, we have
\begin{displaymath}
\left\|\prod_{q=1}^{\tmix{}}\*W_{q+k}\left(\*I - \frac{\*1\*1^\top}{n}\right)\right\|_1=\max_{\*x} \frac{\left\|\prod_{q=1}^{\tmix{}}\*W_{q+k}\left(\*I - \frac{\*1\*1^\top}{n}\right)\*x\right\|_1}{\left\|\*x\right\|_1}\leq \frac{1}{2}
\end{displaymath}
As a result, from the submultiplicativity of the matrix induced norm, we obtain
\begin{align*}
& \left\|\prod_{q=a}^{b}\*W_q\left(\*I - \frac{\*1\*1^\top}{n}\right)\right\|_1\\
\leq & \left\|\prod_{q=1}^{\tmix{}}\*W_{a-1+q}\left(\*I - \frac{\*1\*1^\top}{n}\right)\right\|_1\cdots\left\|\prod_{q=1}^{\tmix{}}\*W_{\cdots+q}\left(\*I - \frac{\*1\*1^\top}{n}\right)\right\|_1\cdot\left\|\prod_{q=1}^{t_r}\*W_{\cdots + q}\left(\*I - \frac{\*1\*1^\top}{n}\right)\right\|_1\\
\leq & 2^{-\left\lfloor\frac{b-a+1}{\tmix{}}\right\rfloor}\left\|\prod_{q=1}^{t_r}\*W_{\cdots + q}\left(\*I - \frac{\*1\*1^\top}{n}\right)\right\|_1
\end{align*}
where $t_r=(b-a+1)\bmod \tmix{}$.
Note that
\begin{align*}
\left\|\prod_{q=1}^{t_r}\*W_q\left(\*I - \frac{\*1\*1^\top}{n}\right)\right\|_1\leq 1-\frac{1}{n} + (n-1)\frac{1}{n} = 2-\frac{2}{n} \leq 2
\end{align*}
Putting it back we obtain
\begin{align*}
\left\|\prod_{q=a}^{b}\*W_{\cdots+q}\left(\*I - \frac{\*1\*1^\top}{n}\right)\right\|_1 \leq 2\cdot 2^{-\left\lfloor\frac{b-a+1}{\tmix{}}\right\rfloor}
\end{align*}
That completes the proof.
\end{proof}
Note that in the analysis of Moniqua on AD-PSGD (Section~\ref{Async Moniqua}), we will use this lemma as an assumption.
\section{Moniqua on D-PSGD (Proof to Theorem~\ref{Moniqua convergence rate} and \ref{arbitrary lemma})}\label{Moniqua}

\subsection{Notations}
For convenience, we adopt the following notation
\begin{align*}
\*X_k & = \left[\*x_{k,1}, \cdots, \*x_{k,n}\right],\hspace{2em}
\*{\hat{X}}_k = \left[\*{\hat{x}}_{k,1}, \cdots, \*{\hat{x}}_{k,n}\right]\\
\*{\tilde{G}}_k & = \left[\*{\tilde{g}}_{k,1}, \cdots, \*{\tilde{g}}_{k,n}\right], \hspace{2em}
\*G_k = \left[\*g_{k,1}, \cdots, \*g_{k,n}\right]\\
\*{\overline{X}} &= \*X\frac{\*1}{n}, \forall \*X\in\mathbb{R}^{d\times n}, \hspace{2em}
\*\Omega_k = (\*{\hat{X}}_k-\*X_k)(\*W-\*I)
\end{align*}
where $\*g_{k,i}$ denotes gradient computed via the whole dataset $\mathcal{D}_i$ and $\*x_{k,i}$

From a local view, the update rule on worker $i$ at iteration $k$ can be written as
\begin{displaymath}
	\*x_{k+1,i} \leftarrow \*x_{k,i} + \sum\nolimits_{j\in\mathcal{N}_i}\left(\*{\hat{x}}_{k,j} - \*{\hat{x}}_{k,i}\right)\*W_{ji} - \alpha_k \*{\tilde{g}}_{k,i}
\end{displaymath}
which is equivalent to
\begin{equation}\label{update rule}
\*x_{k+1,i} = \sum_{j=1}^{n}\*x_{k,j}\*W_{ji} - \alpha_k \*{\tilde{g}}_{k,i} +  \sum_{j=1}^{n}\left((\*{\hat{x}}_{k,j} - \*x_{k,j}) - (\*{\hat{x}}_{k,i} - \*x_{k,i})\right)\*W_{ji}
\end{equation}
with a more compact notation, this can be expressed as:
\begin{equation}\label{global update}
\*{X}_{k+1} = \*X_k + \*{\hat{X}}_k(\*W-\*I) - \alpha_k \*{\tilde{G}}_k= \*X_k\*W - \alpha_k \*{\tilde{G}}_k + (\*{\hat{X}}_k-\*X_k)(\*W-\*I)
\end{equation}
\subsection{Proof to Theorem~\ref{Moniqua convergence rate}.}
\begin{proof}
From Lemma~\ref{dss ready lemma} we have
\begin{align*}
	\sum_{k=0}^{K-1}\alpha_k\mathbb{E}\left\|\nabla f(\*{\overline{X}}_k)\right\|^2\leq & 4(\mathbb{E}f(\*0) - \mathbb{E}f^*) + \frac{2\sigma^2L}{n}\sum_{k=0}^{K-1}\alpha_k^2 + \frac{8\sigma^2L^2}{(1-\rho)^2}\sum_{k=0}^{K-1}\alpha_k^3 + \frac{24\varsigma^2L^2}{(1-\rho)^2}\sum_{k=0}^{K-1}\alpha_k^3\\
	& + \frac{8L^2}{n(1-\rho )^2}\sum_{k=0}^{K-1}\alpha_k\mathbb{E}\left\|\*\Omega_k\right\|^2_F
\end{align*}
Note that
\begin{align*}
\sum_{k=0}^{K-1}\alpha_k\mathbb{E}\left\|\*\Omega_k\right\|^2_F
= \sum_{k=0}^{K-1}\alpha_k\sum_{i=1}^{n}\mathbb{E}\left\|\sum_{j=1}^{n}\left((\*{\hat{x}}_{k,j}-\*x_{k,j}) - (\*{\hat{x}}_{k,i}-\*x_{k,i})\right)\*W_{ji}\right\|^2
\overset{\text{Lemma}~\ref{modifynoise},\ref{dss_bound}}{\leq} 4\sum_{k=0}^{K-1}\alpha_k\delta^2 B_{\theta_k}^2nd
\end{align*}
By using Lemma~\ref{dss_bound} and by assigning $\delta=\frac{1-\eta\rho}{8C_\alpha^2\eta\log(16n)+2(1-\eta\rho)}$, we obtain
\begin{align*}
\sum_{k=0}^{K-1}\alpha_k\mathbb{E}\left\|\*\Omega_k\right\|^2_F
\leq \frac{G_\infty^2dn}{C_\alpha^2}\sum_{k=0}^{K-1}\alpha_k^3
\end{align*}
Pushing it back we obtain
\begin{align*}
	\sum_{k=0}^{K-1}\alpha_k\mathbb{E}\left\|\nabla f(\*{\overline{X}}_k)\right\|^2\leq & 4(\mathbb{E}f(\*0) - \mathbb{E}f^*) + \frac{2\sigma^2L}{n}\sum_{k=0}^{K-1}\alpha_k^2 + \frac{8\sigma^2L^2}{(1-\rho)^2}\sum_{k=0}^{K-1}\alpha_k^3 + \frac{24\varsigma^2L^2}{(1-\rho)^2}\sum_{k=0}^{K-1}\alpha_k^3\\
	& + \frac{8G_\infty^2dL^2}{(1-\rho)^2C_\alpha^2}\sum_{k=0}^{K-1}\alpha_k^3
\end{align*}
That completes the proof.
\end{proof}

\subsection{Proof to Corollary~\ref{Moniqua convergence rate corollary}.}
\begin{proof}
When $\alpha_k=\alpha$, $C_\alpha=\eta=1$, and we have:
\begin{displaymath}
	\frac{1}{K}\sum_{k=0}^{K-1}\mathbb{E}\left\|\nabla f(\*{\overline{X}}_k)\right\|^2\leq \frac{4(f(\*0) - f^*)}{\alpha K} + \frac{2\alpha L}{n}\sigma^2 +  \frac{8\alpha^2L^2\left(\sigma^2 + 3\varsigma^2\right)}{(1-\rho)^2} + \frac{8\alpha^2G_\infty^2d L^2}{(1-\rho)^2}
\end{displaymath}
By setting $\alpha=\frac{1}{\varsigma^{\frac{2}{3}}K^{\frac{1}{3}}+\sigma\sqrt{\frac{K}{n}}+2L}$, we have
\begin{align*}
\frac{1}{K}\sum_{k=0}^{K-1}\mathbb{E}\left\|\nabla f(\*{\overline{X}}_k)\right\|^2 \leq & \frac{8(f(\*0) - f^*)L}{K} + \frac{4\sigma(f(\*0) - f^* + L/2)}{\sqrt{nK}} +  \frac{4\varsigma^{\frac{2}{3}}(f(\*0)-f^*)}{K^{\frac{2}{3}}}\\
& + \frac{8L^2\sigma^2n}{(1-\rho)^2(\sigma^2K+4nL^2)} + \frac{24L^2\varsigma^{\frac{2}{3}}}{(1-\rho)^2K^{\frac{2}{3}}}+ \frac{8G_\infty^2dn L^2}{(1-\rho)^2(\sigma^2K+4nL^2)}\\
\lesssim & \frac{1}{K} + \frac{\sigma}{\sqrt{nK}} + \frac{\varsigma^{\frac{2}{3}}}{K^{\frac{2}{3}}} + \frac{\sigma^2n}{\sigma^2K+n}+\frac{G_\infty^2dn}{\sigma^2K+n}
\end{align*}
That completes the proof of Corollary 1.
\end{proof}

\subsection{Lemma for Moniqua on D-PSGD}

\begin{lemma}\label{modifynoise}
If $\|\*x_{t,i}-\*x_{t,j}\|_\infty<{\theta_t}$, $\forall i, j$ holds at iteration $t$, then
\begin{align*}
    \left\|\sum_{j=1}^{n}\left((\*{\hat{x}}_{t,j}-\*x_{t,j}) - (\*{\hat{x}}_{t,i}-\*x_{t,i})\right)\*W_{ji}\right\|_\infty \leq \frac{4\delta}{1-2\delta}{\theta_t}
\end{align*}
\end{lemma}
\begin{proof}
Let $ B_{\theta_t}=\frac{2}{1-2\delta}{\theta_t}$, based on the algorithm, we obtain
\begin{align*}
    \*{\hat{x}}_{t,j} &= \left( B_{\theta_t}\mathcal{Q}_\delta\left(\frac{\*x_{t,j}}{ B_{\theta_t}}\bmod 1\right)-\*x_{t,i}\right)\bmod  B_{\theta_t} + \*x_{t,i}\\
    \*{\hat{x}}_{t,i} & \overset{\text{Lemma 2}}{=}  B_{\theta_t}\mathcal{Q}_\delta\left(\frac{\*x_{t,i}}{ B_{\theta_t}}\bmod 1\right) -  B_{\theta_t}\left(\frac{\*x_{t,i}}{ B_{\theta_t}}\bmod 1\right) + \*x_{t,i}
\end{align*}
We start from
\begin{align*}
    \left\|\sum_{j=1}^{n}\left((\*{\hat{x}}_{t,j}-\*x_{t,j}) - (\*{\hat{x}}_{t,i}-\*x_{t,i})\right)\*W_{ji}\right\|_\infty
    \leq & \sum_{j=1}^{n}\*W_{ji}\left\|(\*{\hat{x}}_{t,j}-\*x_{t,j}) - (\*{\hat{x}}_{t,i}-\*x_{t,i})\right\|_\infty\\
    \leq & \sum_{j=1}^{n}\*W_{ji}\left\|\*{\hat{x}}_{t,j}-\*x_{t,j}\right\|_\infty + \sum_{j=1}^{n}\*W_{ji}\left\|\*{\hat{x}}_{t,i}-\*x_{t,i}\right\|_\infty
\end{align*}
On the first hand, due to Lemma 2 we obtain
\begin{align*}
    \left\|\*{\hat{x}}_{t,j}-\*x_{t,j}\right\|_\infty \leq \delta B_{\theta_t}
\end{align*}
on the other hand,
\begin{align*}
    \left\|\*{\hat{x}}_{t,i}-\*x_{t,i}\right\|_\infty = \left\| B_{\theta_t}\mathcal{Q}_\delta\left(\frac{\*x_{t,i}}{ B_{\theta_t}}\bmod 1\right) -  B_{\theta_t}\left(\frac{\*x_{t,i}}{ B_{\theta_t}}\bmod 1\right)\right\|_\infty\leq \delta B_{\theta_t}
\end{align*}
Putting it back, we obtain
\begin{align*}
    \left\|\sum_{j=1}^{n}\left((\*{\hat{x}}_{t,j}-\*x_{t,j}) - (\*{\hat{x}}_{t,i}-\*x_{t,i})\right)\*W_{ji}\right\|_\infty
    \leq 2\delta B_{\theta_t}=\frac{4\delta}{1-2\delta}{\theta_t}
\end{align*}
which completes the proof.
\end{proof}

\begin{lemma}\label{lemma5}
    For any $\*X_t\in\mathbb{R}^{d\times n}$, we have
    \begin{displaymath}
        \left\|\sum_{t=0}^{k-1}\*X_t\left(\frac{\*1\*1^\top}{n} - \*W^{k-t-1}\right)\right\|^2_F\leq\left(\sum_{t=0}^{k-1}\rho^{k-t-1}\left\|\*X_t\right\|_F\right)^2
    \end{displaymath}
\end{lemma}
\begin{proof}
\begin{align*}
    \left\|\sum_{t=0}^{k-1}\*X_t\left(\frac{\*1\*1^\top}{n} - \*W^{k-t-1}\right)\right\|^2_F
    = & \left(\left\|\sum_{t=0}^{k-1}\*X_t\left(\frac{\*1\*1^\top}{n} - \*W^{k-t-1}\right)\right\|_F\right)^2\\
    \leq & \left(\sum_{t=0}^{k-1}\left\|\*X_t\left(\frac{\*1\*1^\top}{n} - \*W^{k-t-1}\right)\right\|_F\right)^2\\
    \leq & \left(\sum_{t=0}^{k-1}\left\|\*X_t\right\|_F\left\|\frac{\*1\*1^\top}{n} - \*W^{k-t-1}\right\|\right)^2\\
    \leq & \left(\sum_{t=0}^{k-1}\rho^{k-t-1}\left\|\*X_t\right\|_F\right)^2
\end{align*}
That completes the proof.
\end{proof}

\begin{lemma}\label{dss_bound}
In any iteration $k\geq 0$, and for any two worker $i$ and $j$, 
when $\delta=\frac{1-\eta\rho}{8C_\alpha^2\eta\log(16n)+2(1-\eta\rho)}$
we have:
\begin{align*}
\left\|\*X_{k}(\*e_i - \*e_j)\right\|_\infty < \frac{2\alpha_kG_\infty C_\alpha\eta\log(16n)}{1-\eta\rho}=\theta_k
\end{align*}
\end{lemma}
\begin{proof}
We use mathematical induction to prove this:

I. When $k=0$, $\left\|\*X_0(\*e_i - \*e_j)\right\|_\infty=0<\theta_0, \forall i,j$

II. Suppose $\left\|\*X_{t}(\*e_i - \*e_j)\right\|_\infty<\theta_t, \forall t\leq k, \forall i,j$, we obtain
\begin{align*}
    \left\|\*X_{k+1}(\*e_i - \*e_j)\right\|_\infty
    = & \left\|\sum_{t=0}^{k}(-\alpha_{t}\*G_t+\*\Omega_t)\*W^{k-t}(\*e_i - \*e_j)\right\|_\infty\\
    \leq & \sum_{t=0}^{k}\left\|-\alpha_{t}\*G_t\right\|_{1,\infty}\left\|\*W^{k-t}(\*e_i - \*e_j)\right\|_1 + \sum_{t=0}^{k}\left\|\*\Omega_t\right\|_{1,\infty}\left\|\*W^{k-t}(\*e_i - \*e_j)\right\|_1\\
    \overset{\text{Lemma}~\ref{modifynoise}}{\leq} & \sum_{t=0}^{k}\alpha_tG_\infty\left\|\*W^{k-t}(\*e_i - \*e_j)\right\|_1 + \frac{4\delta}{1-2\delta}\sum_{t=0}^{k}\theta_t\left\|\*W^{k-t}(\*e_i - \*e_j)\right\|_1\\
    \leq & \alpha_{k+1}G_\infty\sum_{t=0}^{k}\frac{\alpha_{k-t}}{\alpha_{k+1}}\left\|\*W^t(\*e_i - \*e_j)\right\|_1 + \frac{4\delta\theta_{k}}{1-2\delta}\sum_{t=0}^{k}\frac{\theta_t}{\theta_{k}}\left\|\*W^{k-t}(\*e_i - \*e_j)\right\|_1\\
    < & \alpha_{k+1}G_\infty C_\alpha\eta\sum_{t=0}^{\infty}\eta^{t}\left\|\*W^t(\*e_i - \*e_j)\right\|_1 + \frac{4\delta C_\alpha\theta_{k}}{1-2\delta}\sum_{t=0}^{\infty}\eta^t\left\|\*W^t(\*e_i - \*e_j)\right\|_1
\end{align*}
For any $t\geq 0$, on one hand
\begin{align*}
    \left\|\*W^{t}(\*e_i - \*e_j)\right\|_1 \leq \sqrt{n}\left\|\*W^{t}(\*e_i - \*e_j)\right\|_2 \leq \sqrt{n}\left\|\*W^{t}\*e_i-\frac{\*1}{n}\right\|+\sqrt{n}\left\|\*W^{t}\*e_j-\frac{\*1}{n}\right\| \leq 2\sqrt{n}\rho^t
\end{align*}
where the last step holds due to the diagonalizability of $\*W$. On the other hand,
\begin{align*}
    \left\|\*W^{t}(\*e_i - \*e_j)\right\|_1 \leq \*1^\top \*W^t\*e_i + \*1^\top \*W^t\*e_i = \*1^\top \*e_i + \*1^\top \*e_j = 2
\end{align*}
As a result
\begin{align*}
    \eta^t\left\|\*W^{t}(\*e_i - \*e_j)\right\|_1 \leq \min\{2\sqrt{n}(\eta\rho)^t, 2\}
\end{align*}
Let $T_0=\left\lceil\frac{-\log(\sqrt{n})}{\log(\eta\rho)}\right\rceil$, so that $\sqrt{n}(\eta\rho)^{T_0}\leq 1$, then we have
\begin{align*}
    \sum_{t=0}^{\infty}\eta^t\left\|\*W^{t}(\*e_i - \*e_j)\right\|_1 = & \sum_{t=0}^{T_0-1}\eta^t\left\|\*W^{t}(\*e_i - \*e_j)\right\|_1 + \sum_{t=T_0}^{\infty}\eta^t\left\|\*W^{t}(\*e_i - \*e_j)\right\|_1\\
    \leq & \sum_{t=0}^{T_0-1}2 + \sum_{t=0}^{\infty}2\sqrt{n}(\eta\rho)^{t+T_0}\\
    \leq & 2\left\lceil\frac{-\log(\sqrt{n})}{\log(\eta\rho)}\right\rceil + \sum_{t=0}^{\infty}2\left(\sqrt{n}(\eta\rho)^{T_0}\right)(\eta\rho)^{t}\\
    \leq & \frac{2\log(\sqrt{n})}{1-\eta\rho} + 2 + \frac{2}{1-\eta\rho}\\
    \leq & \frac{\log(16n)}{1-\eta\rho}
\end{align*}
As a result, we have
\begin{align*}
    \left\|\*X_{k+1}(\*e_i - \*e_j)\right\|_\infty < \frac{\alpha_{k+1}G_\infty C_\alpha\eta\log(16n)}{1-\eta\rho}+ \frac{4\delta C_\alpha}{1-2\delta}\cdot\frac{\log(16n)}{1-\eta\rho}\theta_{k}
\end{align*}
with $\delta=\frac{1-\eta\rho}{8C_\alpha^2\eta\log(16n)+2(1-\eta\rho)}$,
\begin{align*}
\left\|\*X_{k+1}(\*e_i - \*e_j)\right\|_\infty < & 
\frac{\alpha_{k+1}G_\infty C_\alpha\eta\log(16n)}{1-\eta\rho} + \frac{4\delta C_\alpha}{1-2\delta}\cdot\frac{\log(16n)}{1-\eta\rho}\cdot\frac{2\alpha_{k}G_\infty C_\alpha\eta\log(16n)}{1-\eta\rho}\\
\leq & \frac{\alpha_{k+1}G_\infty C_\alpha\eta\log(16n)}{1-\eta\rho} + \frac{4\delta C_\alpha}{1-2\delta}\cdot\frac{\log(16n)}{1-\eta\rho}\cdot\frac{2\alpha_{k+1}C_\alpha \eta G_\infty C_\alpha\eta\log(16n)}{1-\eta\rho}\\
\leq & \frac{2\alpha_{k+1}G_\infty C_\alpha\eta\log(16n)}{1-\eta\rho}=\theta_{k+1}
\end{align*}
Combining I and II, we complete the proof.
\end{proof}

\begin{lemma}\label{dss ready lemma}
The running average of the gradient norm has the following bound:
\begin{align*}
	\sum_{k=0}^{K-1}\alpha_k\mathbb{E}\left\|\nabla f(\*{\overline{X}}_k)\right\|^2\leq & 4(\mathbb{E}f(\*0) - \mathbb{E}f^*) + \frac{2\sigma^2L}{n}\sum_{k=0}^{K-1}\alpha_k^2 + \frac{8\sigma^2L^2}{(1-\rho)^2}\sum_{k=0}^{K-1}\alpha_k^3 + \frac{24\varsigma^2L^2}{(1-\rho)^2}\sum_{k=0}^{K-1}\alpha_k^3\\
	& + \frac{8L^2}{n(1-\rho )^2}\sum_{k=0}^{K-1}\alpha_k\mathbb{E}\left\|\*\Omega_k\right\|^2_F
\end{align*}
\end{lemma}
\begin{proof}
Let $\*1$ denote a n-dimensional vector with all the entries be 1. And we have
\begin{displaymath}
    \*{\overline{X}}_{k+1} = (\*X_k\*W-\alpha_k\*{\tilde{G}}_k + \*\Omega_k)\frac{\*1}{n}= \*{\overline{X}}_k-\alpha_k\overline{\*{\tilde{G}}}_k + (\*{\hat{X}}_k-\*X_k)(\*W-\*I)\frac{\*1}{n}=\*{\overline{X}}_k-\alpha_k\overline{\*{\tilde{G}}}_k
\end{displaymath}
And by Taylor Expansion, we have
\begin{align*}
\mathbb{E}f(\*{\overline{X}}_{k+1}) & = \mathbb{E}f\left(\frac{(\*X_k\*W-\alpha_k\*{\tilde{G}}_k + \*\Omega_k)\*1}{n}\right)\\
& = \mathbb{E}f\left(\*{\overline{X}}_k - \alpha_k\overline{\*{\tilde{G}}}_k\right)\\
&\leq \mathbb{E}f(\*{\overline{X}}_{k}) - \alpha_k\mathbb{E}\langle\nabla f(\overline{\*X}_k), \overline{\*{\tilde{G}}}_k\rangle + \frac{\alpha_k^2L}{2}\mathbb{E}\left\|\overline{\*{\tilde{G}}}_k\right\|^2
\end{align*}
And for the last term, we have
\begin{align*}
\mathbb{E}\left\|\overline{\*{\tilde{G}}}_k\right\|^2 & = \mathbb{E}\left\|\frac{\sum_{i=1}^{n}\*{\tilde{g}}_{k,i}}{n}\right\|^2\\
& = \mathbb{E}\left\|\frac{\sum_{i=1}^{n}\*{\tilde{g}}_{k,i} - \sum_{i=1}^{n}\*g_{k,i}}{n} + \frac{\sum_{i=1}^{n}\*g_{k,i}}{n}\right\|^2\\
& = \mathbb{E}\left\|\frac{\sum_{i=1}^{n}\*{\tilde{g}}_{k,i} - \sum_{i=1}^{n}\*g_{k,i}}{n}\right\|^2 + \mathbb{E}\left\|\frac{\sum_{i=1}^{n}\*g_{k,i}}{n}\right\|^2 +  \mathbb{E}\left\langle\frac{\sum_{i=1}^{n}\*{\tilde{g}}_{k,i} - \sum_{i=1}^{n}\*g_{k,i}}{n} + \frac{\sum_{i=1}^{n}\*g_{k,i}}{n}\right\rangle\\
& = \mathbb{E}\left\|\frac{\sum_{i=1}^{n}\*{\tilde{g}}_{k,i} - \sum_{i=1}^{n}\*g_{k,i}}{n}\right\|^2 + \mathbb{E}\left\|\frac{\sum_{i=1}^{n}\*g_{k,i}}{n}\right\|^2\\
& \overset{\text{Assumption 3}}{=} \frac{1}{n^2}\sum_{i=1}^{n}\mathbb{E}\left\|\*{\tilde{g}}_{k,i} - \*g_{k,i}\right\|^2 + \mathbb{E}\left\|\frac{\sum_{i=1}^{n}\*g_{k,i}}{n}\right\|^2\\
& \overset{\text{Assumption 3}}{\leq} \frac{\sigma^2}{n} + \mathbb{E}\left\|\frac{\sum_{i=1}^{n}\*g_{k,i}}{n}\right\|^2
\end{align*}
Putting it back, we obtain
\begin{align*}
\mathbb{E}f(\overline{\*X}_{k+1}) & \leq \mathbb{E}f(\overline{\*X}_{k}) - \alpha_k\mathbb{E}\langle\nabla f(\overline{\*X}_k), \overline{\*{\tilde{G}}}_k\rangle + \frac{\alpha_k^2L}{2n}\sigma^2 + \frac{\alpha_k^2L}{2}\mathbb{E}\left\|\frac{\sum_{i=1}^{n}\*g_{k,i}}{n}\right\|^2\\
& = \mathbb{E}f(\overline{\*X}_{k}) - \frac{\alpha_k - \alpha_k^2L}{2}\mathbb{E}\left\|\overline{\*G}_k\right\|^2 - \frac{\alpha_k}{2}\mathbb{E}\left\|\nabla f(\overline{\*X}_k)\right\|^2 + \frac{\alpha_k^2L}{2n}\sigma^2+\frac{\alpha_k}{2}\mathbb{E}\left\|\nabla f(\overline{\*X}_k) - \overline{\*G}_k\right\|^2
\end{align*}
where the last step comes from $2\langle \*a,\*b\rangle = \|\*a\|^2 + \|\*b\|^2 = \|\*a-\*b\|^2$
And
\begin{align*}
\mathbb{E}\left\|\nabla f(\overline{\*X}_k) - \overline{\*G}_k\right\|^2& \leq \frac{1}{n}\sum_{i=1}^{n}\mathbb{E}\left\|\nabla f_i\left(\frac{\sum_{i^{'}=1}^{n}\*x_{k,i^{'}}}{n}\right) - \nabla f_i(\*x_{k,i})\right\|^2\\
& \overset{\text{Assumption 1}}{\leq} \frac{L^2}{n}\sum_{i=1}^{n}\mathbb{E}\left\|\frac{\sum_{i^{'}=1}^{n}\*x_{k,i^{'}}}{n} - \*x_{k,i}\right\|^2\\
& = \frac{L^2}{n}\sum_{i=1}^{n}\mathbb{E}\left\|\overline{\*X}_k - \*x_{k,i}\right\|^2
\end{align*}
by Lipschitz assumption, we obtain
\begin{displaymath}
\frac{\alpha_k - \alpha_k^2L}{2}\mathbb{E}\left\|\overline{\*G}_k\right\|^2 + \frac{\alpha_k}{2}\mathbb{E}\left\|\nabla f(\overline{\*X}_k)\right\|^2 
\leq \mathbb{E}f(\overline{\*X}_{k}) - \mathbb{E}f(\overline{\*X}_{k+1}) + \frac{\alpha_k^2L}{2n}\sigma^2 + \frac{\alpha_k L^2}{2n}\sum_{i=1}^{n}\mathbb{E}\left\|\overline{\*X}_k - \*x_{k,i}\right\|^2
\end{displaymath}
summing over from $k=0$ to $K-1$ on both sides, we have
\begin{align*}
\sum_{k=0}^{K-1}(\alpha_k - \alpha_k^2L)\mathbb{E}\left\|\overline{\*G}_k\right\|^2 + \sum_{k=0}^{K-1}\alpha_k\mathbb{E}\left\|\nabla f(\overline{\*X}_k)\right\|^2
\leq & 2(\mathbb{E}f(\overline{\*X}_0) - \mathbb{E}f(\overline{\*X}_K)) + \frac{\sigma^2L}{n}\sum_{k=0}^{K-1}\alpha_k^2\\
& + \frac{L^2}{n}\sum_{k=0}^{K-1}\sum_{i=1}^{n}\alpha_k\mathbb{E}\left\|\overline{\*X}_k - \*x_{k,i}\right\|^2
\end{align*}
From Lemma~\ref{dss_lemma1}, we have
\begin{align*}
    & \sum_{k=0}^{K-1}(\alpha_k - \alpha_k^2L)\mathbb{E}\left\|\overline{\*G}_k\right\|^2 + \sum_{k=0}^{K-1}\alpha_k\mathbb{E}\left\|\nabla f(\overline{\*X}_k)\right\|^2\\ 
    \leq & 2(\mathbb{E}f(\overline{\*X}_0) - \mathbb{E}f(\overline{\*X}_K)) + \frac{\sigma^2L}{n}\sum_{k=0}^{K-1}\alpha_k^2 + \frac{L^2}{n}\sum_{k=0}^{K-1}\sum_{i=1}^{n}\alpha_k\mathbb{E}\left\|\overline{\*X}_k - \*x_{k,i}\right\|^2\\
	\leq & 2(\mathbb{E}f(\overline{\*X}_0) - \mathbb{E}f(\overline{\*X}_K)) + \frac{\sigma^2L}{n}\sum_{k=0}^{K-1}\alpha_k^2 + \frac{4\sigma^2L^2}{(1-\rho)^2}\sum_{k=0}^{K-1}\alpha_k^3 + \frac{12\varsigma^2L^2}{(1-\rho)^2}\sum_{k=0}^{K-1}\alpha_k^3 + \frac{12L^2}{(1-\rho)^2}\sum_{k=0}^{K-1}\alpha_k^3\mathbb{E}\left\|\nabla f(\overline{\*X}_k)\right\|^2\\
	& + \frac{4L^2}{n(1-\rho )^2}\sum_{k=0}^{K-1}\alpha_k\mathbb{E}\left\|\*\Omega_k\right\|^2_F
\end{align*}
Rearrange the terms, we have
\begin{align*}
	\sum_{k=0}^{K-1}\alpha_k\mathbb{E}\left\|\nabla f(\overline{\*X}_k)\right\|^2\leq & 4(\mathbb{E}f(\*0) - \mathbb{E}f^*) + \frac{2\sigma^2L}{n}\sum_{k=0}^{K-1}\alpha_k^2 + \frac{8\sigma^2L^2}{(1-\rho)^2}\sum_{k=0}^{K-1}\alpha_k^3 + \frac{24\varsigma^2L^2}{(1-\rho)^2}\sum_{k=0}^{K-1}\alpha_k^3\\
	& + \frac{8L^2}{n(1-\rho )^2}\sum_{k=0}^{K-1}\alpha_k\mathbb{E}\left\|\*\Omega_k\right\|^2_F
\end{align*}
and that completes the proof
\end{proof}

\begin{lemma}\label{dss_lemma1}
\begin{align*}
\frac{L^2}{n}\sum_{k=0}^{K-1}\sum_{i=1}^{n}\alpha_k\mathbb{E}\left\|\overline{\*X}_k - \*x_{k,i}\right\|^2
\leq & \frac{4\sigma^2L^2}{(1-\rho)^2}\sum_{k=0}^{K-1}\alpha_k^3 + \frac{12\varsigma^2L^2}{(1-\rho)^2}\sum_{k=0}^{K-1}\alpha_k^3 + \frac{12L^2}{(1-\rho)^2}\sum_{k=0}^{K-1}\alpha_k^3\mathbb{E}\left\|\nabla f(\overline{\*X}_k)\right\|^2\\
& + \frac{4L^2}{n(1-\rho )^2}\sum_{k=0}^{K-1}\alpha_k\mathbb{E}\left\|\*\Omega_k\right\|^2_F
\end{align*}
\end{lemma}
\begin{proof}
	\begin{align*}
	&\sum_{k=0}^{K-1}\sum_{i=1}^{n}\alpha_k\mathbb{E}\left\|\overline{\*X}_k - \*x_{k,i}\right\|^2\\
	= & \sum_{k=0}^{K-1}\sum_{i=1}^{n}\alpha_k\mathbb{E}\left\|\*X_k\left(\frac{\*1}{n} - \*e_i\right)\right\|^2\\
	= & \sum_{k=1}^{K-1}\sum_{i=1}^{n}\alpha_k\mathbb{E}\left\|\left(\*X_{k-1}\*W - \alpha \*{\tilde{G}}_{k-1} +  \*\Omega_{k-1}\right)\left(\frac{\*1}{n} - \*e_i\right)\right\|^2\\
	\overset{\*x_{0,i}=\*0}{=} &\sum_{k=1}^{K-1}\sum_{i=1}^{n}\alpha_k\mathbb{E}\left\|\sum_{t=0}^{k-1}\left(-\alpha_t\*{\tilde{G}}_t +  \*\Omega_t\right)\left(\frac{\*1}{n} - \*W^{k-t-1}\*e_i\right)\right\|^2\\
	\leq & 2\sum_{k=1}^{K-1}\alpha_k\sum_{i=1}^{n}\mathbb{E}\left\|\sum_{t=0}^{k-1}\alpha_t\*{\tilde{G}}_t\left(\frac{\*1}{n}-\*W^{k-t-1}\*e_i\right)\right\|^2 + 2\sum_{k=1}^{K-1}\alpha_k\sum_{i=1}^{n}\mathbb{E}\left\|\sum_{t=0}^{k-1}\*\Omega_t\left(\frac{\*1}{n} - \*W^{k-t-1}\*e_i\right)\right\|^2\\
	= & 2\sum_{k=1}^{K-1}\alpha_k\mathbb{E}\left\|\sum_{t=0}^{k-1}\alpha_t\*{\tilde{G}}_t\left(\frac{\*1\*1^\top}{n}-\*W^{k-t-1}\right)\right\|^2_F + 2\sum_{k=1}^{K-1}\mathbb{E}\left\|\sum_{t=0}^{k-1}\*\Omega_t\left(\frac{\*1\*1^\top}{n} - \*W^{k-t-1}\right)\right\|^2_F\\
	\overset{Lemma~\ref{lemma5}}{\leq} &2\sum_{k=1}^{K-1}\alpha_k\left(\sum_{t=0}^{k-1}\rho ^{k-t-1}\alpha_t\mathbb{E}\left\|\*{\tilde{G}}_t\right\|_F\right)^2 + 2\sum_{k=1}^{K-1}\alpha_k\left(\sum_{t=0}^{k-1}\rho ^{k-t-1}\mathbb{E}\left\|\*\Omega_t\right\|_F\right)^2\\
    \overset{Lemma~\ref{lemmasequence}}{\leq} & \frac{2}{(1-\rho)^2}\sum_{k=0}^{K-1}\alpha_k^3\mathbb{E}\left\|\*{\tilde{G}}_k\right\|^2_F + \frac{2}{(1-\rho )^2}\sum_{k=0}^{K-1}\alpha_k\mathbb{E}\left\|\*\Omega_k\right\|^2_F\\
    \overset{Lemma~\ref{dss_lemma2}}{\leq} & \frac{2}{(1-\rho)^2}\left(n\sigma^2\sum_{k=0}^{K-1}\alpha_k^3 + 3L^2\sum_{k=0}^{K-1}\sum_{i=1}^{n}\alpha_k^3\mathbb{E}\left\|\overline{\*X}_k - \*x_{k,i}\right\|^2 + 3n\varsigma^2\sum_{k=0}^{K-1}\alpha_k^3 + 3n\sum_{k=0}^{K-1}\alpha_k^3\mathbb{E}\left\|\nabla f(\overline{\*X}_k)\right\|^2\right)\\
    & + \frac{2}{(1-\rho )^2}\sum_{k=0}^{K-1}\alpha_k\mathbb{E}\left\|\*\Omega_k\right\|^2_F
	\end{align*}
Rearrange the terms, we have
\begin{align*}
	\sum_{k=0}^{K-1}\alpha_k\left(1 -\frac{6\alpha_k^2L^2}{(1-\rho)^2}\right)\sum_{i=1}^{n}\mathbb{E}\left\|\overline{\*X}_k - \*x_{k,i}\right\|^2	\leq & \frac{2n\sigma^2}{(1-\rho)^2}\sum_{k=0}^{K-1}\alpha_k^3 + \frac{6n\varsigma^2}{(1-\rho)^2}\sum_{k=0}^{K-1}\alpha_k^3 + \frac{6n}{(1-\rho)^2}\sum_{k=0}^{K-1}\alpha_k^3\mathbb{E}\left\|\nabla f(\overline{\*X}_k)\right\|^2\\
	& + \frac{2}{(1-\rho )^2}\sum_{k=0}^{K-1}\alpha_k\mathbb{E}\left\|\*\Omega_k\right\|^2_F
\end{align*}
Let $1 - \frac{6\alpha_k^2L^2}{(1-\rho)^2}\geq \frac{1}{2}$, we have
\begin{align*}
\frac{L^2}{n}\sum_{k=0}^{K-1}\sum_{i=1}^{n}\alpha_k\mathbb{E}\left\|\overline{\*X}_k - \*x_{k,i}\right\|^2
\leq & \frac{4\sigma^2L^2}{(1-\rho)^2}\sum_{k=0}^{K-1}\alpha_k^3 + \frac{12\varsigma^2L^2}{(1-\rho)^2}\sum_{k=0}^{K-1}\alpha_k^3 + \frac{12L^2}{(1-\rho)^2}\sum_{k=0}^{K-1}\alpha_k^3\mathbb{E}\left\|\nabla f(\overline{\*X}_k)\right\|^2\\
& + \frac{4L^2}{n(1-\rho )^2}\sum_{k=0}^{K-1}\alpha_k\mathbb{E}\left\|\*\Omega_k\right\|^2_F
\end{align*}
That completes the proof.
\end{proof}

\begin{lemma}\label{dss_lemma2}
	\begin{displaymath}
	\sum_{k=0}^{K-1}\alpha_k^3\mathbb{E}\left\|\*{\tilde{G}}_k\right\|^2_F \leq n\sigma^2\sum_{k=0}^{K-1}\alpha_k^3 + 3L^2\sum_{k=0}^{K-1}\sum_{i=1}^{n}\alpha_k^3\mathbb{E}\left\|\overline{\*X}_k - \*x_{k,i}\right\|^2 + 3n\varsigma^2\sum_{k=0}^{K-1}\alpha_k^3 + 3n\sum_{k=0}^{K-1}\alpha_k^3\mathbb{E}\left\|\nabla f(\overline{\*X}_k)\right\|^2
	\end{displaymath}
\end{lemma}
\begin{proof}
	From the property of Frobenius norm, we have
	\begin{displaymath}
	\mathbb{E}\left\|\*{\tilde{G}}_k\right\|^2_F = \sum_{i=1}^{n}\mathbb{E}\left\|\*{\tilde{g}}_{k,i}\right\|^2
	\end{displaymath}
Since
\begin{align*}
	\mathbb{E}\left\|\*{\tilde{g}}_{k,i}\right\|^2= & \mathbb{E}\left\|\*{\tilde{g}}_{k,i}-\*g_{k,i}\right\|^2 + \mathbb{E}\left\|\*g_{k,i}\right\|^2\\
	= & \sigma^2 + 3\mathbb{E}\left\|\nabla f_i(\*x_{k,i}) - \nabla f_i(\overline{\*X}_k)\right\|^2 + 3\mathbb{E}\left\|\nabla f_i(\overline{\*X}_k) - \nabla f(\overline{\*X}_k)\right\|^2 + 3\mathbb{E}\left\|\nabla f(\overline{\*X}_k)\right\|^2\\
	\leq & \sigma^2 + 3L^2\mathbb{E}\left\|\overline{\*X}_k - \*x_{k,i}\right\|^2 + 3\varsigma^2 + 3\mathbb{E}\left\|\nabla f(\overline{\*X}_k)\right\|^2
	\end{align*}
	Summing from $k=0$ to $K-1$, we obtain
	\begin{align*}
	&\sum_{k=0}^{K-1}\alpha_k^3\mathbb{E}\left\|\*{\tilde{G}}_{k}\right\|^2_F\\
	= & \sum_{k=0}^{K-1}\alpha_k^3\sum_{i=1}^{n}\mathbb{E}\left\|\*{\tilde{g}}_{k,i}\right\|^2\\
	\leq & \sum_{k=0}^{K-1}\alpha_k^3\sum_{i=1}^{n}\sigma^2 + 3L^2\sum_{k=0}^{K-1}\alpha_k^3\sum_{i=1}^{n}\mathbb{E}\left\|\overline{\*X}_k - \*x_{k,i}\right\|^2 + 3\sum_{k=0}^{K-1}\alpha_k^3\sum_{i=1}^{n}\varsigma^2 + 3\sum_{k=0}^{K-1}\alpha_k^3\sum_{i=1}^{n}\mathbb{E}\left\|\nabla f(\overline{\*X}_k)\right\|^2\\
	= & n\sigma^2\sum_{k=0}^{K-1}\alpha_k^3 + 3L^2\sum_{k=0}^{K-1}\sum_{i=1}^{n}\alpha_k^3\mathbb{E}\left\|\overline{\*X}_k - \*x_{k,i}\right\|^2 + 3n\varsigma^2\sum_{k=0}^{K-1}\alpha_k^3 + 3n\sum_{k=0}^{K-1}\alpha_k^3\mathbb{E}\left\|\nabla f(\overline{\*X}_k)\right\|^2
\end{align*}
That completes the proof.
\end{proof}

\begin{lemma}\label{lemmasequence}
Given $0\leq\rho<1$ and $T$, a positive integer. Also given non-negative sequences $\{a_t\}_{t=1}^{\infty}$ and $\{b_t\}_{t=1}^{\infty}$ with $\{a_t\}_{t=1}^{\infty}$ being non-increasing, the following inequalities holds:
\begin{align*}
    \sum_{t=1}^{k}a_t\left(\sum_{s=1}^{t}\rho^{-\left\lfloor\frac{t-s}{T}\right\rfloor}b_s\right) \leq & \frac{T}{1-\rho}\sum_{s=1}^{k}a_sb_s\\
    \sum_{t=1}^{k}a_t\left(\sum_{s=1}^{t}\rho^{-\left\lfloor\frac{t-s}{T}\right\rfloor}b_s\right)^2 \leq & \frac{T^2}{(1-\rho)^2}\sum_{s=1}^{k}a_sb_s^2
\end{align*}
\end{lemma}
\begin{proof}
Firstly,
\begin{align*}
    S_k = \sum_{t=1}^{k}a_t\left(\sum_{s=1}^{t}\rho^{-\left\lfloor\frac{t-s}{T}\right\rfloor}b_s\right) = \sum_{s=1}^{k}\sum_{t=s}^{k}\alpha_t\rho^{-\left\lfloor\frac{t-s}{T}\right\rfloor}b_s \leq \sum_{s=1}^{k}a_sb_s\sum_{t=0}^{T-1}\sum_{m=0}^{\infty}\rho^m \leq \frac{T}{1-\rho}\sum_{s=1}^{k}a_sb_s
\end{align*}
further we have
\begin{align*}
    &\sum_{t=1}^{k}a_t\left(\sum_{s=1}^{t}\rho^{-\left\lfloor\frac{t-s}{T}\right\rfloor}b_s\right)^2 = \sum_{t=1}^{k}a_t\sum_{s=1}^{t}\rho^{-\left\lfloor\frac{t-s}{T}\right\rfloor}b_s\sum_{r=1}^{t}\rho^{-\left\lfloor\frac{t-r}{T}\right\rfloor}b_r = \sum_{t=1}^{k}a_t\sum_{s=1}^{t}\sum_{r=1}^{t}\rho^{-\left\lfloor\frac{t-s}{T}\right\rfloor+\left\lfloor\frac{t-r}{T}\right\rfloor}b_sb_r\\
    \leq & \sum_{t=1}^{k}a_t\sum_{s=1}^{t}\sum_{r=1}^{t}\rho^{-\left\lfloor\frac{t-s}{T}\right\rfloor+\left\lfloor\frac{t-r}{T}\right\rfloor}\frac{b_s^2 + b_r^2}{2} = \sum_{t=1}^{k}a_t\sum_{s=1}^{t}\sum_{r=1}^{t}\rho^{-\left\lfloor\frac{t-s}{T}\right\rfloor+\left\lfloor\frac{t-r}{T}\right\rfloor}b_s^2\\
    \leq & \sum_{t=1}^{k}a_t\sum_{s=1}^{t}b_s^2\rho^{-\left\lfloor\frac{t-s}{T}\right\rfloor}\sum_{r=1}^{t}\rho^{-\left\lfloor\frac{t-r}{T}\right\rfloor}\leq \sum_{t=1}^{k}a_t\sum_{s=1}^{t}b_s^2\rho^{-\left\lfloor\frac{t-s}{T}\right\rfloor}\sum_{r=0}^{T-1}\sum_{m=0}^{\infty}\rho^m\\
    \leq & \frac{T}{1-\rho}\sum_{t=1}^{k}a_t\sum_{s=1}^{t}\rho^{-\left\lfloor\frac{t-s}{T}\right\rfloor}b_s^2 \overset{\text{Using $S_k$}}{\leq} \frac{T^2}{(1-\rho)^2}\sum_{s=1}^{k}a_sb_s^2
\end{align*}
That completes the proof.
\end{proof}

\subsection{Proof to Theorem~\ref{arbitrary lemma}.}
\begin{proof}
Let $\overline{\rho}$ denote the spectral gap of matrix $\overline{\*W}$, it is straightforward to know that $\overline{\rho}=\gamma\rho+(1-\gamma)$.
we first use mathematical induction to prove at iteration $\forall k\leq K$, for any worker $i$ and $j$, with probability $(1-\epsilon)^k$
\begin{align*}
    \|\*X_k(\*e_i-\*e_j)\|_\infty < \theta = \frac{2\alpha\log(16n)G_\infty}{\gamma(1-\rho)}
\end{align*}
where $\gamma=\frac{2}{1-\rho + \frac{16\delta^2}{(1-2\delta)^2}\cdot\frac{32\log(4n)}{1-\rho}\log\left(\frac{1}{\epsilon}\right)}$.

I. When $k=0$, $\left\|\*X_0(\*e_i - \*e_j)\right\|_\infty=0<\theta$

II. Suppose $\left\|\*X_t(\*e_i - \*e_j)\right\|_\infty<\theta$ holds for $\forall t\leq k$, then for $k+1$ we have
\begin{align*}
\left\|\*X_{k+1}(\*e_i - \*e_j)\right\|_\infty
= & \left\|\left(\*X_k\overline{\*W} - \alpha \*{\tilde{G}}_k + \gamma\*\Omega_k\right)(\*e_i - \*e_j)\right\|_\infty\\
\overset{\*X_0=0}{=} & \left\|\sum_{t=0}^{k}\left(-\alpha\*{\tilde{G}}_t+\gamma\*\Omega_t\right)\overline{W}^{k-t}(\*e_i - \*e_j)\right\|_\infty\\
\leq & \left\|\sum_{t=0}^{k}\alpha\*{\tilde{G}}_t\overline{\*W}^{k-t}(\*e_i - \*e_j)\right\|_\infty + \left\|\sum_{t=0}^{k}\gamma\*\Omega_t\overline{\*W}^{k-t}(\*e_i - \*e_j)\right\|_\infty
\end{align*}
We bound these two terms seperately. First from Lemma~\ref{dss_bound} we know that
\begin{equation}\label{upperbound}
    \sum_{t=0}^{\infty}\left\|\overline{\*W}^{t}(\*e_i - \*e_j)\right\|_1 < \frac{\log(16n)}{1-\overline{\rho}} = \frac{\log(16n)}{\gamma(1-\rho)}
\end{equation}
then we have for the first term,
\begin{align*}
    \left\|\sum_{t=0}^{k}\alpha\*{\tilde{G}}_t\overline{\*W}^{k-t}(\*e_i - \*e_j)\right\|_\infty
    \leq & \sum_{t=0}^{k}\left\|\alpha\*{\tilde{G}}_t\right\|_{1,\infty}\left\|\overline{\*W}^{k-t}(\*e_i - \*e_j)\right\|_1\\
    \leq & \alpha G_\infty\sum_{t=0}^{\infty}\left\|\overline{\*W}^{t}(\*e_i - \*e_j)\right\|_1\\
    < & \frac{\alpha\log(16n)G_\infty}{\gamma(1-\rho)}
\end{align*}
Next, we bound the second term. Suppose the infinity norm of the term $\sum_{t=0}^{k}\gamma\*\Omega_t\overline{\*W}^{k-t}(\*e_i - \*e_j)$ is taken at coordinate $h$, then we have
\begin{align*}
    \left\|\sum_{t=0}^{k}\gamma\*\Omega_t\overline{\*W}^{k-t}(\*e_i - \*e_j)\right\|_\infty
    = & \gamma\left|\*e_h^\top\left(\sum_{t=0}^{k}\*\Omega_t\overline{\*W}^{k-t}(\*e_i - \*e_j)\right)\right|\\
    = & \gamma\left|\sum_{t=0}^{k}\*e_h^\top\left(\*\Omega_t\overline{\*W}^{k-t}(\*e_i - \*e_j)\right)\right|
\end{align*}
Let
\begin{align*}
    u_t=\sum_{m=0}^{t}\*e_h^\top\left(\*\Omega_{k-m}\overline{\*W}^m(\*e_i - \*e_j)\right)
\end{align*}
from the induction hypothesis we know that $\{u_t\}_{t\leq k}$ is a martingale sequence. Note that,
\begin{align*}
    |u_t-u_{t-1}| = &  \left|\*e_h^\top\left(\*\Omega_{k-t}\overline{\*W}^t(\*e_i - \*e_j)\right)\right|\\
    \leq & \left\|\*\Omega_{k-t}\overline{\*W}^t(\*e_i - \*e_j)\right\|_\infty\\ \overset{Equation~\ref{upperbound}}{\leq} & \left\|\*\Omega_{k-t}\right\|_{1,\infty}\min\{2\sqrt{n}\overline{\rho}^t, 2\}\\
    \leq & 2\delta B_\theta\min\{2\sqrt{n}\overline{\rho}^t, 2\}
\end{align*}
where $ B_\theta=\frac{2}{1-2\delta}\theta$, then by using Azuma's inequality we obtain
\begin{align*}
    \mathbb{P}\left[\left|\sum_{t=0}^{k}\*e_h^\top\left(\*\Omega_t\overline{\*W}^{k-t}(\*e_i - \*e_j)\right)\right| > a\right]\leq & \exp{\left(-\frac{a^2}{8\delta^2 B_\theta^2\sum_{t=0}^{k}\min\{2\sqrt{n}\overline{\rho}^t, 2\}^2}\right)}\\
    \leq & \exp{\left(-\frac{a^2}{32\delta^2 B_\theta^2\sum_{t=0}^{\infty}\min\{n\overline{\rho}^{2t}, 1\}}\right)}
\end{align*}
Here we use the induction hypothesis.
Similar as before, Let $T_0=\left\lceil\frac{-\log(n)}{2\log(\overline{\rho})}\right\rceil$, so that $n\overline{\rho}^{2T_0}\leq 1$, then we have
\begin{align*}
    \sum_{t=0}^{\infty}\min\{n\overline{\rho}^{2t}, 1\} = & \sum_{t=0}^{T_0-1}\min\{n\overline{\rho}^{2t}, 1\} + \sum_{t=T_0}^{\infty}\min\{n\overline{\rho}^{2t}, 1\}\\
    < & \sum_{t=0}^{T_0-1}1 + \sum_{t=0}^{\infty}n\overline{\rho}^{2t+2T_0}\\
    \leq & \left\lceil\frac{-\log(n)}{2\log(\overline{\rho})}\right\rceil + \sum_{t=0}^{\infty}\left(n\overline{\rho}^{2T_0}\right)\overline{\rho}^{2t}\\
    \leq & \frac{\log(n)}{1-\overline{\rho}^2} + 1 + \frac{1}{1-\overline{\rho}^2}\\
    \leq & \frac{\log(4n)}{1-\overline{\rho}^2}\\
    = & \frac{\log(4n)}{\gamma(1-\rho)(2-\gamma(1-\rho))}
\end{align*}
Putting it back, we obtain
\begin{align*}
    \mathbb{P}\left[\left|\sum_{t=0}^{k}\*e_h^\top\left(\*\Omega_t\overline{\*W}^{k-t}(\*e_i - \*e_j)\right)\right| > a\right]
    \leq & \exp{\left(-\frac{a^2\gamma(1-\rho)(2-\gamma(1-\rho))}{32\delta^2 B_\theta^2\log(4n)}\right)}
\end{align*}
In other words, with probability $1-\epsilon$,
\begin{align*}
    \left\|\sum_{t=0}^{k}\gamma\*\Omega_t\overline{\*W}^{k-t}(\*e_i - \*e_j)\right\|_\infty = & \gamma\left|\sum_{t=0}^{k}\*e_h^\top\left(\*\Omega_t\overline{\*W}^{k-t}(\*e_i - \*e_j)\right)\right| \leq \delta B_\theta \sqrt{\frac{32\log(4n)\gamma}{(1-\rho)(2-\gamma(1-\rho))}\log\left(\frac{1}{\epsilon}\right)}
\end{align*}
Combine them together, we obtain
\begin{align*}
\left\|\*X_{k+1}(\*e_i - \*e_j)\right\|_\infty
< & \frac{\alpha\log(16n)G_\infty}{\gamma(1-\rho)} + \delta B_\theta \sqrt{\frac{32\log(4n)\gamma}{(1-\rho)(2-\gamma(1-\rho))}\log\left(\frac{1}{\epsilon}\right)}\\
< & \frac{\alpha\log(16n)G_\infty}{\gamma(1-\rho)} + \frac{2\delta}{1-2\delta}\theta \sqrt{\frac{32\log(4n)\gamma}{(1-\rho)(2-\gamma(1-\rho))}\log\left(\frac{1}{\epsilon}\right)}
\end{align*}
Let $\gamma=\frac{2}{1-\rho + \frac{16\delta^2}{(1-2\delta)^2}\cdot\frac{32\log(4n)}{1-\rho}\log\left(\frac{1}{\epsilon}\right)}$
\begin{align*}
    \left\|\*X_{k+1}(\*e_i - \*e_j)\right\|_\infty < \frac{\alpha\log(16n)G_\infty}{\gamma(1-\rho)} + \frac{1}{2}\theta \leq \theta
\end{align*}
Combining I and II, we complete the proof.

We proceed to obtain the convergence rate.
From Theorem 2 we have with $\alpha_k=\alpha$
\begin{align*}
	\frac{1}{K}\sum_{k=0}^{K-1}\mathbb{E}\left\|\nabla f(\overline{\*X}_k)\right\|^2\leq & \frac{4(\mathbb{E}f(\*0) - \mathbb{E}f^*)}{\alpha K} + \frac{2\alpha\sigma^2L}{n} + \frac{8\alpha^2\sigma^2L^2}{(1-\overline{\rho})^2} + \frac{24\alpha^2\varsigma^2L^2}{(1-\overline{\rho})^2} + \frac{8\alpha L^2}{n(1-\overline{\rho} )^2K}\sum_{k=0}^{K-1}\mathbb{E}\left\|\gamma\*\Omega_k\right\|^2_F
\end{align*}
Note that with probability $(1-\epsilon)^K$
\begin{align*}
\sum_{k=0}^{K-1}\mathbb{E}\left\|\gamma\*\Omega_k\right\|^2_F
= \gamma^2\sum_{k=0}^{K-1}\sum_{i=1}^{n}\mathbb{E}\left\|\sum_{j=1}^{n}\left((\*{\hat{x}}_{k,j}-\*x_{k,j}) - (\*{\hat{x}}_{k,i}-\*x_{k,i})\right)\*W_{ji}\right\|^2
\overset{\text{Lemma}~\ref{modifynoise}}{\leq} \frac{16\delta^2\gamma^2}{(1-2\delta)^2}\theta^2dnK
\end{align*}
Fit in $\theta=\frac{2\alpha\log(16n)G_\infty}{\gamma(1-\rho)}$, we obtain
\begin{align*}
    \sum_{k=0}^{K-1}\mathbb{E}\left\|\gamma\*\Omega_k\right\|^2_F \leq \frac{64\alpha^2\delta^2\log^2(16n)G_\infty^2}{(1-2\delta)^2(1-\rho)^2}dnK
\end{align*}
Let $\mathcal{E}$ denote the event that the bound $\theta$ holds for all $0\leq t\leq T-1$, then,
\begin{align*}
    \frac{1}{K}\sum_{k=0}^{K-1}\mathbb{E}\left\|\nabla f(\overline{\*X}_k)\right\|^2
    = & \left[\frac{1}{K}\sum_{k=0}^{K-1}\mathbb{E}\left\|\nabla f(\overline{\*X}_k)\right\|^2|\mathcal{E}\right]\mathbb{P}(\mathcal{E}) + \left[\frac{1}{K}\sum_{k=0}^{K-1}\mathbb{E}\left\|\nabla f(\overline{\*X}_k)\right\|^2|\neg\mathcal{E}\right]\mathbb{P}(\neg\mathcal{E})\\
    \leq & \frac{4(f(\*0) - f^*)}{\alpha K} + \frac{2\alpha L}{n}\sigma^2 +  \frac{8\alpha^2L^2\left(\sigma^2 + 3\varsigma^2\right)}{(1-\overline{\rho})^2} + \frac{8L^2}{nK(1-\overline{\rho})^2}\sum_{k=1}^{K-1}\mathbb{E}\left\|\gamma\*\Omega_k\right\|^2_F\\
    & + G_\infty^2d\left(1-(1-\epsilon)^K\right)\\
    \leq & \frac{4(f(\*0) - f^*)}{\alpha K} + \frac{2\alpha L}{n}\sigma^2 +  \frac{8\alpha^2L^2\left(\sigma^2 + 3\varsigma^2\right)}{\gamma^2(1-\rho)^2} + \frac{512\alpha^2\delta^2L^2\log^2(16n)G_\infty^2d}{\gamma^2(1-\rho)^4(1-2\delta)^2}\\
    & + G_\infty^2d\left(1-(1-\epsilon)^K\right)
\end{align*}
Assign $\epsilon=\frac{1}{K^2}$ and 
set $\alpha=\frac{1}{\varsigma^{\frac{2}{3}}K^{\frac{1}{3}}+\sigma\sqrt{\frac{K}{n}}+2L}$, we have
\begin{align*}
\frac{1}{K}\sum_{k=0}^{K-1}\mathbb{E}\left\|\nabla f(\overline{\*X}_k)\right\|^2
\lesssim & \frac{\sigma}{\sqrt{nK}} + \frac{1}{K} + \frac{\varsigma^{\frac{2}{3}}\delta^4\log^2(n)\log^2(K)}{K^{\frac{2}{3}}(1-2\delta)^4} + \frac{\sigma^2n\delta^4\log^2(n)\log^2(K)}{(\sigma^2K+n)(1-2\delta)^4}+\frac{n\delta^6\log^4(n)\log^2(K)}{(\sigma^2K+n)(1-2\delta)^6}
\end{align*}
That completes the proof
\end{proof}
\section{Moniqua on $D^2$ (Proof to Theorem~\ref{thmMoniquaD2})}
\subsection{Setting}
We first show the pseudo code in Algorithm~\ref{D2algo}.

\begin{algorithm}[t]
	\caption{Moniqua with Variance Reduction on worker $i$}\label{D2algo}
	\begin{algorithmic}[1]
		\Require initial point $\*x_{0,i} = \*x_0$, step size $\alpha$, the discrepency bound $B_\theta$, communication matrix $\*W$, number of iterations $K$, neighbor list of worker $i$: $\mathcal{N}_i$, quantizer $Q_\delta$
		\For{$k=0,1,2,\cdots,K-1$}
		\State Randomly sample data $\xi_{k,i}$ from local memory
		
		\State Compute a local stochastic gradient based on $\xi_{k,i}$ and current weight $\*x_{k,i}$: $\*{\tilde{g}}_{k,i}$
		
		\If{$k=0$}
		
		\State Update local weight: $\*x_{k+\frac{1}{2},i} \leftarrow \*x_{k,i} - \alpha\*{\tilde{g}}_{k,i}$
		
		\Else
		
		\State Update local weight: $\*x_{k+\frac{1}{2},i} \leftarrow 2\*x_{k,i} - \*x_{k-1,i} - \alpha\*{\tilde{g}}_{k,i} + \alpha\*{\tilde{g}}_{k-1,i}$
		
		\EndIf
		
		\State Send modulo-ed model to neighbors: $\*q_{k+\frac{1}{2},i}\leftarrow \mathcal{Q}_\delta\left(\frac{\*x_{k+\frac{1}{2},i}}{B_\theta} \bmod 1\right)$
		
		\State Compute local biased term $\*{\hat{x}}_{k+\frac{1}{2},i}$ as:
		\begin{align*}
		    \*{\hat{x}}_{k+\frac{1}{2},i}=\*q_{k+\frac{1}{2},i}B_\theta-\*x_{k+\frac{1}{2},i}\bmod B_\theta+ \*x_{k+\frac{1}{2},i}
		\end{align*}
		
		\State Recover model received from worker $j$ as:
		\begin{align*}
		  \*{\hat{x}}_{k+\frac{1}{2},j}=(\*q_{k+\frac{1}{2},j}B_\theta-\*x_{k+\frac{1}{2},j})\bmod B_\theta + \*x_{k+\frac{1}{2},i}
		\end{align*}
		
		\State Average with neighboring workers: $\*x_{k+1,i} \leftarrow \*x_{k+\frac{1}{2},i} + \sum_{j\in\mathcal{N}_i}(\*{\hat{x}}_{k+\frac{1}{2},j}-\*{\hat{x}}_{k+\frac{1}{2},i})\*W_{ji}$
		
	\EndFor
	\Ensure $\overline{\*X}_K=\frac{1}{n}\sum\nolimits_{i=1}^{n}\*x_{K,i}$
	\end{algorithmic}
\end{algorithm}

$D^2$ makes the following assumptions (1-4), and we add the additional assumption (5):
\begin{enumerate}
	\item \textbf{Lipschitzian Gradient}: All the function $f_i$ have L-Lipschitzian gradients.\label{D2_Assumption1}
	\item \textbf{Communication Matrix}: Communication matrix $\*W$ is a symmetric doubly stochastic matrix. Let the eigenvalues of $\*W\in\mathbb{R}^{n\times n}$ be $\lambda_1\geq\cdots\geq\lambda_n$. We assume $\lambda_2<1, \lambda_n>-\frac{1}{3}$.
	\label{D2_Assumption2}
	\item \textbf{Bounded Variance}:
	\begin{displaymath}
	\mathbb{E}_{\xi_i\sim\mathcal{D}_i}\left\|\nabla\tilde{f}_i(\*x;\xi_i) - \nabla f_i(\*x)\right\|^2 \leq \sigma^2, \forall i
	\end{displaymath}
	where $\nabla\tilde{f}_i(\*x;\xi_i)$ denotes gradient sample on worker $i$ computed via data sample $\xi_i$.
	\label{D2_Assumption3}
	\item\textbf{Initialization}: All the models are initialized by the same parameters: $\*x_{0,i}=\*x_0, \forall i$ and with out the loss of generality $\*x_0=0$. \label{D2_Assumption4}
	\item \textbf{Gradient magnitude}: The norm of a sampled gradient is bounded by $\left\|\*{\tilde{g}}_{k,i}\right\|_\infty \leq G_\infty$ for some constant $G_\infty$. \label{D2_Assumption5}
\end{enumerate}

\subsection{Proof to Theorem~\ref{thmMoniquaD2}}
\begin{proof}
From a local view, define $\*x_{-1}=\*{\tilde{g}}_{-1}=0$, the update rule of Moniqua on $D^2$ on worker $i$ in iteration $k$ can be written as
\begin{align*}
\*x_{k+\frac{1}{2},i} & = 2\*x_{k,i} - \*x_{k-1,i} - \alpha\*{\tilde{g}}_{k,i} + \alpha\*{\tilde{g}}_{k-1,i}\\
\*x_{k+1,i} & = \sum_{j=1}^{n}\*x_{k+\frac{1}{2},j}\*W_{ji} + \sum_{j=1}^{n}\left((\*{\hat{x}}_{k+\frac{1}{2},j}-\*x_{k+\frac{1}{2},j}) - (\*{\hat{x}}_{k+\frac{1}{2},i}-\*x_{k+\frac{1}{2},i})\right)\*W_{ji}
\end{align*}
For a more compact expression,
\begin{align*}
\*X_{k+\frac{1}{2}} &= 2\*X_k - \*X_{k-1} - \alpha\*{\tilde{G}}_k + \alpha\*{\tilde{G}}_{k-1}\\
\*X_{k+1} &= \*X_{k+\frac{1}{2}}\*W + (\*{\hat{X}}_{k+\frac{1}{2}}-\*X_{k+\frac{1}{2}})(\*W-\*I)
\end{align*}
Define
\begin{displaymath}
\*\Omega_k = (\*{\hat{X}}_{k+\frac{1}{2}}-\*X_{k+\frac{1}{2}})(\*W-\*I)
\end{displaymath}
Since $\*W$ is symmetric, it can be diagonalized as $\*W=\*P\*\Lambda \*P^\top$, where the i-th column of $\*P$ and $\*\Lambda$ are $\*W$'s i-th eigenvector and eigenvalue, respectively. And we obtain
\begin{displaymath}
\*X_{k+1} = 2\*X_k\*P\*\Lambda \*P^\top - \*X_{k-1}\*P\*\Lambda \*P^\top - \alpha\*{\tilde{G}}_{k}\*P\*\Lambda \*P^\top + \alpha\*{\tilde{G}}_{k-1}\*P\*\Lambda \*P^\top + \*\Omega_k
\end{displaymath}
and
\begin{displaymath}
\*X_{k+1}\*P = 2\*X_k\*P\*\Lambda - \*X_{k-1}\*P\*\Lambda - \alpha\*{\tilde{G}}_{k}\*P\*\Lambda + \alpha\*{\tilde{G}}_{k-1}\*P\*\Lambda + \*\Omega_k\*P
\end{displaymath}
Denote $\*Y_k = \*X_k\*P$, $\*H(\*X_k;\xi_k) = \*{\tilde{G}}_{k}\*P$, and denote $\*y_{k,i}$, $\*h_{k,i}$ and $\*r_{k,i}$ as the $i$-th column of $\*Y_k$, $\*H_k$ and $\*\Omega_k\*P$, respectively. Then we have
\begin{displaymath}
\*y_{k+1, i} = \lambda_i(2\*y_{k,i} - \*y_{k-1,i} - \alpha \*h_{k,i} + \alpha \*h_{k-1,i}) + \*r_{k,i}
\end{displaymath}

From Lemma~\ref{D2_lemma5} (Constants $C_1$, $C_2$, $C_3$ andn $C_4$ are defined in the Lemma~\ref{D2_lemma1}. Constants $D_1$ and $D_2$ are defined in Lemma~\ref{D2_lemma5}) we get
\begin{align*}
&\left(1-\frac{3C_1\alpha^2L^2}{C_4}\right)\mathbb{E}\left\|\nabla f(\*0)\right\| + \left(1-\alpha L-3\frac{C_2}{C_4}\alpha^4L^4\right)\frac{1}{K}\sum_{k=1}^{K-1}\mathbb{E}\left\|\overline{\*G}_k\right\|^2 + \frac{1}{K}\sum_{k=0}^{K-1}\mathbb{E}\left\|\nabla f(\overline{\*X}_k)\right\|^2\\ 
\leq & \frac{2(f(0) - f^*)}{\alpha K} + \frac{\alpha L}{n}\sigma^2 + \frac{3C_1\alpha^2L^2(\sigma^2+\varsigma_0^2)}{C_4K} + 6\frac{C_2}{C_4}\alpha^2\sigma^2L^2 + 3\frac{C_2}{nC_4}\alpha^4\sigma^2L^4 + \frac{C_3L^2}{C_4}\left(\frac{6D_1n+8}{6D_2n+1}\right)^2\alpha^2G_\infty^2d
\end{align*}
Let $\alpha=\frac{1}{\sigma\sqrt{K/n}+2L}$, we have
\begin{align*}
&\frac{1}{K}\sum_{k=0}^{K-1}\mathbb{E}\left\|\nabla f(\overline{\*X}_k)\right\|^2\\ 
\leq & \frac{2(f(\*0) - f^*)}{\alpha K} + \frac{\alpha L}{n}\sigma^2 + \frac{3C_1\alpha^2L^2(\sigma^2+\varsigma_0^2)}{C_4K} + 6\frac{C_2}{C_4}\alpha^2\sigma^2L^2 + 3\frac{C_2}{nC_4}\alpha^4\sigma^2L^4 + \left(\frac{6D_1n+8}{6D_2n+1}\right)^2\frac{C_3L^2}{C_4}G_\infty^2d\alpha^2\\
\leq & \frac{4(f(\*0)-f^*)L}{K} + \frac{2\sigma(f(\*0)-f^*+L/2)}{\sqrt{nK}} + \frac{3C_1L^2(\sigma^2+\varsigma_0^2)n}{C_4(\sigma^2K^2+4nL^2K)} + \frac{6C_2L^2\sigma^2n}{C_4(\sigma^2K+4nL^2)}\\
& + \frac{3C_2n\sigma^2L^2}{C_4(\sigma^4K^2+16n^2L^4)} + \left(\frac{6D_1n+8}{6D_2n+1}\right)^2\frac{C_3G_\infty^2dL^2n}{C_4(\sigma^2K+4nL^2)}\\
\lesssim & \frac{1}{K} + \frac{\sigma}{\sqrt{nK}} + \frac{(\sigma^2+\varsigma_0^2)n}{\sigma^2K^2+nK} + \frac{\sigma^2n}{\sigma^2K+n} + \frac{\sigma^2n}{\sigma^4K^2+n^2} + \frac{G_\infty^2dn}{\sigma^2K+n}\\
\lesssim & \frac{1}{K} + \frac{\sigma}{\sqrt{nK}} +  \frac{\sigma^2n}{\sigma^2K+n} + \frac{G_\infty^2dn}{\sigma^2K+n}
\end{align*}
That completes the proof.
\end{proof}


\subsection{Lemma for $D^2$}
\begin{lemma}\label{D2_lemma1}
Define
\begin{align*}
    D_1 &= \max\left\{|v_n|+\frac{2|\lambda_n|}{1-|v_n|}, \sqrt{\frac{\lambda_2}{1-\lambda_2}}+\frac{2\lambda_2}{1-\lambda_2}\right\}\\
    D_2 &= \max\left\{\frac{2}{1-|v_n|}, \frac{2}{\sqrt{1-\lambda_2}}\right\}\\
    v_n &= \lambda_n - \sqrt{\lambda_n^2 - \lambda_n}
\end{align*}
Let $\delta=\frac{1}{12nD_2+2}$, and we have for $\forall i,j$
\begin{displaymath}
\left\|\*x_{k+\frac{1}{2}}(\*e_i-\*e_j)\right\|_\infty<\theta=(6D_1n + 8)\alpha G_\infty
\end{displaymath}
\end{lemma}
\begin{proof}
We use mathematical induction to prove this:

I. When $k=0$,
\begin{align*}
\left\|\*X_{0+\frac{1}{2}}(\*e_i-\*e_j)\right\|_\infty = \left\|-\alpha\*{\tilde{G}}_0(\*e_i-\*e_j)\right\|_\infty \leq \alpha\left\|\*{\tilde{G}}_0\right\|_{1,\infty}\left\|\*e_i-\*e_j\right\|_1 < 2\alpha G_\infty \leq (6D_1n + 8)\alpha G_\infty
\end{align*}

II. Suppose for $k\geq 0$, $\forall t\leq k$, we have $\left\|\*X_{t+\frac{1}{2}}(\*e_i-\*e_j)\right\|<(6D_1n + 8)\alpha G_\infty$, then
for $\forall i,j$
\begin{align*}
& \left\|\*X_{k+1}(\*e_i - \*e_j)\right\|_\infty\\
\leq & \left\|\*X_{k+1}\left(\frac{\*1}{n} - \*e_i\right)\right\|_\infty + \left\|X_{k+1}\left(\frac{\*1}{n} - \*e_j\right)\right\|_\infty\\
= & \left\|\*X_{k+1}\*P\*P^\top \*e_i - \*X_{k+1}\*P\begin{bmatrix}
1 & 0 & 0 & \dots  & 0 \\
0 & 0 & 0 & \dots  & 0 \\
0 & 0 & 0 & \dots  & 0 \\
\vdots & \vdots & \vdots & \ddots & \vdots \\
0 & 0 & 0 & \dots  & 0
\end{bmatrix}\*P^\top \*e_i\right\|_\infty + \left\|\*X_{k+1}\*P\*P^\top \*e_j - \*X_{k+1}\*P\begin{bmatrix}
1 & 0 & 0 & \dots  & 0 \\
0 & 0 & 0 & \dots  & 0 \\
0 & 0 & 0 & \dots  & 0 \\
\vdots & \vdots & \vdots & \ddots & \vdots \\
0 & 0 & 0 & \dots  & 0
\end{bmatrix}\*P^\top \*e_j\right\|_\infty\\
\leq & \left\|\*X_{k+1}\*P\begin{bmatrix}
0 & 0 & 0 & \dots  & 0 \\
0 & 1 & 0 & \dots  & 0 \\
0 & 0 & 1 & \dots  & 0 \\
\vdots & \vdots & \vdots & \ddots & \vdots \\
0 & 0 & 0 & \dots  & 1
\end{bmatrix}\right\|_{1,\infty}\|\*P^\top \*e_i\|_1 + \left\|\*X_{k+1}\*P\begin{bmatrix}
0 & 0 & 0 & \dots  & 0 \\
0 & 1 & 0 & \dots  & 0 \\
0 & 0 & 1 & \dots  & 0 \\
\vdots & \vdots & \vdots & \ddots & \vdots \\
0 & 0 & 0 & \dots  & 1
\end{bmatrix}\right\|_{1,\infty}\|\*P^\top \*e_j\|_1\\
\leq & 2\sqrt{n}\left\|\*X_{k+1}\*P\begin{bmatrix}
0 & 0 & 0 & \dots  & 0 \\
0 & 1 & 0 & \dots  & 0 \\
0 & 0 & 1 & \dots  & 0 \\
\vdots & \vdots & \vdots & \ddots & \vdots \\
0 & 0 & 0 & \dots  & 1
\end{bmatrix}\right\|_{1,\infty}
\end{align*}
From the update rule, we have
\begin{displaymath}
\*y_{k+1,i} = \lambda_i(2\*y_{k,i} - \*y_{k-1,i} - \alpha \*h_{k,i} + \alpha \*h_{k-1,i}) + \*r_{k,i}= \lambda_i(2\*y_{k,i} - \*y_{k-1,i}) + \lambda_i\*\beta_{k,i} + \*r_{k,i}
\end{displaymath}
where$\*\beta_{k,i} = -\alpha \*h_{k,i} + \alpha \*h_{k-1,i}$, for all $\*y_i$ with $-\frac{1}{3}<\lambda_i<0$, from Lemma~\ref{D2_lemma3} we have
\begin{displaymath}
\*y_{k+1,i} = \*y_{1,i}\left(\frac{u_i^{k+1} - v_i^{k+1}}{u_i - v_i}\right) + \sum_{s=1}^{k}(\lambda_i\*\beta_{s,i}+\*r_{s,i})\frac{u_i^{k-s+1} - v_i^{k-s+1}}{u_i - v_i}
\end{displaymath}
where $u_i = \lambda_i + \sqrt{\lambda_i^2 - \lambda_i}$ and $v_i = \lambda_i - \sqrt{\lambda_i^2 - \lambda_i}$, we obtain
\begin{displaymath}
\left\|\*y_{k+1,i}\right\|_\infty \leq \left\|\*y_{1,i}\right\|_\infty\left|\frac{u_i^{k+1} - v_i^{k+1}}{u_i - v_i}\right| + |\lambda_i|\sum_{s=1}^{k}\left\|\*\beta_{s,i}\right\|_\infty\left|\frac{u_i^{k-s+1} - v_i^{k-s+1}}{u_i - v_i}\right| + \sum_{s=1}^{k}\left\|\*r_{s,i}\right\|_\infty\left|\frac{u_i^{k-s+1} - v_i^{k-s+1}}{u_i - v_i}\right|
\end{displaymath}
Since
\begin{displaymath}
\left|\frac{u_i^{n+1} - v_i^{n+1}}{u_i - v_i}\right| \leq |v_i|^n\left|\frac{u_i\left(\frac{u_i}{v_i}\right)^n - v_i}{u_i - v_i}\right| \leq |v_i|^n
\end{displaymath}
We obtain
\begin{displaymath}
\left\|\*y_{k+1,i}\right\|_\infty \leq\left\|\*y_{1,i}\right\|_\infty|v_i|^k + |\lambda_i|\sum_{s=1}^{k}\left\|\*\beta_{s,i}\right\|_\infty|v_i|^{k-s} + \sum_{s=1}^{k}\left\|\*r_{s,i}\right\|_\infty|v_i|^{k-s}
\end{displaymath}
For $\beta_{s,i}$, we have
\begin{align*}
\left\|\*\beta_{s,i}\right\|_\infty = & \left\|-\alpha \*h_{k,i} + \alpha \*h_{k-1,i}\right\|_\infty \leq 2\alpha(\|\*h_{k,i}\|_\infty + \|\*h_{k-1,i}\|_\infty)\\
\leq & 2\alpha(\|\*G_k\|_{1,\infty}\|\*P\*e_i\|_1 + \|\*G_{k-1}\|_{1,\infty}\|\*P\*e_i\|_1)\\
\leq & 2\alpha\sqrt{n}G_\infty
\end{align*}
For $\*r_{s,i}$, we have
\begin{displaymath}
\left\|\*r_{k,i}\right\|_\infty =  \left\|\*\Omega_k\*P\*e_i\right\|_\infty \leq \left\|\*\Omega_k\right\|_{1,\infty}\|\*P\*e_i\|_1 \leq 2\sqrt{n}\delta B_\theta
\end{displaymath}
when $\lambda_i<0$, we have
\begin{align*}
\left\|\*y_{k+1,i}\right\|_\infty \leq &\left\|\*y_{1,i}\right\|_\infty|v_i|^k + |\lambda_i|\sum_{s=1}^{k}\left\|\*\beta_{s,i}\right\|_\infty|v_i|^{k-s} + \sum_{s=1}^{k}\left\|\*r_{s,i}\right\|_\infty|v_i|^{k-s}\\
\leq & \left\|\*y_{1,i}\right\|_\infty|v_n|^k + |\lambda_n|\sum_{s=1}^{k}\left\|\*\beta_{s,i}\right\|_\infty|v_n|^{k-s} + \sum_{s=1}^{k}\left\|\*r_{s,i}\right\|_\infty|v_n|^{k-s}\\
\leq & \alpha\sqrt{n}G_\infty|v_n|^k + 2\alpha\sqrt{n}G_\infty|\lambda_n|\sum_{s=1}^{\infty}|v_n|^{k-s} + 2\sqrt{n}\delta B_\theta\sum_{s=1}^{\infty}|v_n|^{k-s}\\
\leq & \alpha\sqrt{n}G_\infty|v_n| + \frac{2\alpha\sqrt{n}G_\infty|\lambda_n|}{1-|v_n|} + \frac{2\sqrt{n}\delta B_\theta}{1-|v_n|}\\
\end{align*}
where $v_n = \lambda_n - \sqrt{\lambda_n^2 - \lambda_n}$. 

On the other hand, when $0 \leq \lambda_i < 1$, from Lemma~\ref{D2_lemma3} we have
\begin{displaymath}
\*y_{k+1,i}\sin\phi_i = \*y_{1,i}\lambda_i^{\frac{k}{2}}\sin[(t+1)\phi_i] + \lambda_i\sum_{s=1}^{k}\*\beta_{s,i}\lambda^{\frac{k-s}{2}}_i\sin[(k+1-s)\phi_i] + \sum_{s=1}^{k}\*r_{s,i}\lambda^{\frac{k-s}{2}}_i\sin[(k+1-s)\phi_i]
\end{displaymath}
By taking norm, we get
\begin{align*}
\left\|\*y_{k+1,i}\right\|_\infty|\sin\phi_i| & = \left\|\*y_{1,i}\right\|_\infty\lambda_i^{\frac{k}{2}}|\sin[(t+1)\phi_i]| + \lambda_i\sum_{s=1}^{k}\left\|\*\beta_{s,i}\right\|_\infty|\lambda^{\frac{k-s}{2}}_i||\sin[(k+1-s)\phi_i]|\\
& + \sum_{s=1}^{k}\left\|\*r_{s,i}\right\|_\infty|\lambda^{\frac{k-s}{2}}_i||\sin[(k+1-s)\phi_i]|\\
& < \left\|\*y_{1,i}\right\|_\infty\lambda_2^{\frac{k}{2}} + 2\alpha\sqrt{n}G_\infty\lambda_2\sum_{s=1}^{\infty}\lambda^{\frac{s}{2}}_2 + 2\sqrt{n}\delta B_\theta\sum_{s=1}^{\infty}\lambda^{\frac{s}{2}}_2\\
& \leq \alpha\sqrt{n}G_\infty\sqrt{\lambda_2}+\frac{2\alpha\sqrt{n}G_\infty\lambda_2 + 2\sqrt{n}\delta B_\theta}{\sqrt{1-\lambda_2}}
\end{align*}
Since $|\sin\phi_i| \geq \sqrt{1-\lambda_2}$, putting it back, we get
\begin{displaymath}
\left\|\*y_{k+1,i}\right\|<\alpha\sqrt{n}G_\infty\sqrt{\frac{\lambda_2}{1-\lambda_2}}+\frac{2\alpha\sqrt{n}G_\infty\lambda_2 + 2\sqrt{n}\delta B_\theta}{1-\lambda_2}
\end{displaymath}
So there exists $D_1, D_2$
\begin{align*}
    D_1 &= \max\left\{|v_n|+\frac{2|\lambda_n|}{1-|v_n|}, \sqrt{\frac{\lambda_2}{1-\lambda_2}}+\frac{2\lambda_2}{1-\lambda_2}\right\}\\
    D_2 &= \max\left\{\frac{2}{1-|v_n|}, \frac{2}{\sqrt{1-\lambda_2}}\right\}
\end{align*}
such that
\begin{displaymath}
\left\|\*y_{k+1,i}\right\|_\infty< D_1\alpha\sqrt{n}G_\infty+D_2\sqrt{n}\delta B_\theta
\end{displaymath}
Putting it back we have $\forall i,j$
\begin{displaymath}
\left\|\*X_{k+1}(\*e_i-\*e_j)\right\|_\infty < D_1\alpha nG_\infty+D_2 n\delta B_\theta
\end{displaymath}
As a result
\begin{align*}
& \left\|\*X_{k+1+\frac{1}{2}}(\*e_i-\*e_j)\right\|_\infty\\
= & \left\|(2\*X_{k+1}-\*X_{k}-\alpha\*{\tilde{G}}_{k+1}+\alpha\*{\tilde{G}}_{k})(\*e_i-\*e_j)\right\|_\infty\\
\leq & 2\left\|\*X_{k+1}(\*e_i-\*e_j)\right\|_\infty + \left\|\*X_{k}(\*e_i-\*e_j)\right\|_\infty + \alpha\left\|\*{\tilde{G}}_{k+1}\right\|_{1,\infty}\left\|\*e_i-\*e_j\right\|_1 + \alpha\left\|\*{\tilde{G}}_{k}\right\|_{1,\infty}\left\|\*e_i-\*e_j\right\|_1\\
< & 3(D_1\alpha nG_\infty+D_2 n\delta B_\theta) + 4\alpha G_\infty\\
\leq & (6D_1n + 8)\alpha G_\infty
\end{align*}
The last step is because $\delta=\frac{1}{12nD_2+2}$

Combining I and II we complete the proof.
\end{proof}

\begin{lemma}\label{D2_lemma2}
By defining
\begin{align*}
    C_1 &= \max\left\{\frac{3}{1-|v_n|^2}, \frac{3}{(1-\lambda_2)^2}\right\}\\
    C_2 &= \max\left\{\frac{3\lambda_n^2}{(1-|v_n|)^2}, \frac{3\lambda_2^2}{(1-\sqrt{\lambda_2})^2(1-\lambda_2)}\right\}\\
    C_3 &= \max\left\{\frac{3}{(1-|v_n|)^2}, \frac{3}{(1-\sqrt{\lambda_2})^2(1-\lambda_2)}\right\}
\end{align*}
we have
	\begin{align*}
	& (1 - 12C_2\alpha^2L^2)\sum_{i=1}^{n}\sum_{k=1}^{K}\mathbb{E}\left\|\overline{\*X}_k - \*x_{k,i}\right\|^2\\
	\leq & 3C_1\alpha^2n\sigma^2 + 3C_1\alpha^2n\varsigma_0^2 + 3C_1\alpha^2n\mathbb{E}\left\|\nabla f(\*0)\right\| + 6C_2\alpha^2n\sigma^2K + 3C_2\alpha^4\sigma^2L^2K\\
	& + 3C_2\alpha^4nL^2\sum_{k=1}^{K-1}\mathbb{E}\left\|\overline{\*G}_k\right\|^2 + C_3\sum_{k=1}^{K-1}\mathbb{E}\left\|\*\Omega_k\right\|^2_F
	\end{align*}
\end{lemma}
\begin{proof}
\begin{displaymath}
    \begin{split}
        \sum_{i=1}^{n}\left\|\overline{\*X}_k - \*x_{k,i}\right\|^2 & = \sum_{i=1}^{n}\left\|\*X_k\left(\*e_i - \frac{\*1}{n}\right)\right\|^2\\
        & = \left\|\*X_k\left(\*I - \frac{\*1\*1^\top}{n}\right)\right\|^2_F\\
        & = \left\|\*X_k\*P\*P^\top - \*X_k\*v_1\*v_1^\top\right\|^2_F\\
        & \overset{\text{Lemma~\ref{D2_lemma4}}}{=} \left\|\*X_k\*P\begin{bmatrix}
        0 & 0 & 0 & \dots  & 0 \\
        0 & 1 & 0 & \dots  & 0 \\
        0 & 0 & 1 & \dots  & 0 \\
        \vdots & \vdots & \vdots & \ddots & \vdots \\
        0 & 0 & 0 & \dots  & 1
        \end{bmatrix}\right\|^2_F\\
        & = \sum_{i=2}^{n}\left\|\*y_{k,i}\right\|^2
    \end{split}
    \end{displaymath}
    From the update rule, we obtain,
	\begin{displaymath}
	\*y_{k+1,i} = \lambda_i(2\*y_{k,i} - \*y_{k-1,i} - \alpha \*h_{k,i} + \alpha \*h_{k-1,i}) + \*r_{k,i} = \lambda_i(2\*y_{k,i} - \*y_{k-1,i}) + \lambda_i\*\beta_{k,i} + \*r_{k,i}
	\end{displaymath}
	where$\*\beta_{k,i} = -\alpha \*h_{k,i} + \alpha \*h_{k-1,i}$, for all $\*y_i$ with $-\frac{1}{3}<\lambda_i<0$, from Lemma~\ref{D2_lemma3} we have
	\begin{displaymath}
	\*y_{k+1,i} = \*y_{1,i}\left(\frac{u_i^{k+1} - v_i^{k+1}}{u_i - v_i}\right) + \sum_{s=1}^{k}(\lambda_i\*\beta_{s,i}+\*r_{k,i})\frac{u_i^{k-s+1} - v_i^{k-s+1}}{u_i - v_i}
	\end{displaymath}
	where $u_i = \lambda_i + \sqrt{\lambda_i^2 - \lambda_i}$ and $v_i = \lambda_i - \sqrt{\lambda_i^2 - \lambda_i}$, we obtain
	\begin{align*}
	    \left\|\*y_{k+1,i}\right\|^2 &\leq 3\left\|\*y_{1,i}\right\|^2\left(\frac{u_i^{k+1} - v_i^{k+1}}{u_i - v_i}\right)^2 + 3\lambda_i^2\left(\sum_{s=1}^{k}\left\|\*\beta_{s,i}\right\|\left|\frac{u_i^{k-s+1} - v_i^{k-s+1}}{u_i - v_i}\right|\right)^2\\
	    &+ 3\left(\sum_{s=1}^{k}\left\|\*r_{s,i}\right\|\left|\frac{u_i^{k-s+1} - v_i^{k-s+1}}{u_i - v_i}\right|\right)^2
	\end{align*}
	Since
	\begin{displaymath}
	\left|\frac{u_i^{n+1} - v_i^{n+1}}{u_i - v_i}\right| \leq |v_i|^n\left|\frac{u_i\left(\frac{u_i}{v_i}\right)^n - v_i}{u_i - v_i}\right| \leq |v_i|^n
	\end{displaymath}
	We obtain
	\begin{displaymath}
	\left\|\*y_{k+1,i}\right\|^2 \leq 3\left\|\*y_{1,i}\right\|^2|v_i|^{2t} + 3\lambda_i^2\left(\sum_{s=1}^{k}\left\|\*\beta_{s,i}\right\||v_i|^{k-s}\right)^2 + 3\left(\sum_{s=1}^{k}\left\|\*r_{s,i}\right\||v_i|^{k-s}\right)^2
	\end{displaymath}
	Summing over from $k=0$ to $t=K-1$, we obtain
	\begin{align*}
	&\sum_{k=0}^{K-1}\left\|\*y_{k+1,i}\right\|^2 = \sum_{k=1}^{K}\left\|\*y_{k,i}\right\|^2\\
	\leq & 3\left\|\*y_{1,i}\right\|^2\sum_{k=0}^{K-1}|v_i|^{2k} + 3\lambda_i^2\sum_{k=1}^{K-1}\left(\sum_{s=1}^{k}\left\|\*\beta_{s,i}\right\||v_i|^{k-s}\right)^2 + 3\sum_{k=1}^{K-1}\left(\sum_{s=1}^{k}\left\|\*r_{s,i}\right\||v_i|^{k-s}\right)^2\\
	\leq & \frac{3\left\|\*y_{1,i}\right\|^2}{1-|v_i|^2} + \frac{3\lambda_i^2}{(1 - |v_i|)^2}\sum_{k=1}^{K-1}\left\|\*\beta_{k,i}\right\|^2 + \frac{3}{(1 - |v_i|)^2}\sum_{k=1}^{K-1}\left\|\*r_{k,i}\right\|^2\\
	\leq & \frac{3\left\|\*y_{1,i}\right\|^2}{1-|v_n|^2} + \frac{3\lambda_n^2}{(1 - |v_n|)^2}\sum_{k=1}^{K-1}\left\|\*\beta_{k,i}\right\|^2 + \frac{3}{(1 - |v_n|)^2}\sum_{k=1}^{K-1}\left\|\*r_{k,i}\right\|^2
	\end{align*}
	where $v_n = \lambda_n - \sqrt{\lambda_n^2 - \lambda_n}$.
	
	On the other hand, when $0\leq\lambda_i<1$, from Lemma~\ref{D2_lemma3} we have
	\begin{displaymath}
	\*y_{k+1,i}\sin\phi_i = \*y_{1,i}\lambda_i^{\frac{k}{2}}\sin[(t+1)\phi_i] + \lambda_i\sum_{s=1}^{k}\*\beta_{s,i}\lambda^{\frac{k-s}{2}}_i\sin[(k+1-s)\phi_i] + \sum_{s=1}^{k}\*r_{s,i}\lambda^{\frac{k-s}{2}}_i\sin[(k+1-s)\phi_i]
	\end{displaymath}
	And we have
	\begin{align*}
	\left\|\*y_{k+1,i}\right\|^2\sin^2\phi_i & \leq 3\left\|\*y_{1,i}\right\|^2\lambda_i^k\sin^2[(t+1)\phi_i] + 3\lambda_i^2\left(\sum_{s=1}^{k}\|\*\beta_{s,i}\|\lambda^{\frac{k-s}{2}}_i\sin[(k+1-s)\phi_i]\right)^2\\
	& + 3\left(\sum_{s=1}^{k}\|\*r_{s,i}\|\lambda^{\frac{k-s}{2}}_i\sin[(k+1-s)\phi_i]\right)^2\\
	& \leq 3\left\|\*y_{1,i}\right\|^2\lambda_i^k + 3\lambda_i^2\left(\sum_{s=1}^{k}\|\*\beta_{s,i}\|\lambda^{\frac{k-s}{2}}_i\right)^2 + 3\left(\sum_{s=1}^{k}\|\*r_{s,i}\|\lambda^{\frac{k-s}{2}}_i\right)^2
	\end{align*}
Summing from $k=0$ to $K-1$, we have
	\begin{align*}
	& \sum_{k=0}^{K-1}\left\|\*y_{k+1,i}\right\|^2\sin^2\phi_i = \sum_{k=1}^{K}\left\|\*y_{k,i}\right\|^2\sin^2\phi_i\\ 
	\leq & 3\left\|\*y_{1,i}\right\|^2\sum_{k=0}^{K-1}\lambda_i^t + 3\lambda_i^2\sum_{k=1}^{K-1}\left(\sum_{s=1}^{k}\left\|\*\beta_{s,i}\right\|\lambda_i^{\frac{t-s}{2}}\right)^2 + 3\sum_{k=1}^{K-1}\left(\sum_{s=1}^{k}\|\*r_{s,i}\|\lambda^{\frac{k-s}{2}}_i\right)^2\\
	\leq & \frac{3\left\|\*y_{1,i}\right\|^2}{1-\lambda_i} + \frac{3\lambda_i^2}{(1 - \sqrt{\lambda_i})^2}\sum_{k=1}^{K-1}\left\|\*\beta_{k,i}\right\|^2 + \frac{3}{(1-\sqrt{\lambda_i})^2}\sum_{k=1}^{K-1}\left\|\*r_{k,i}\right\|^2
	\end{align*}
Since $\sin^2\phi_i=1-\lambda_i$, we have
	\begin{align*}
	\sum_{k=1}^{K}\left\|\*y_{k,i}\right\|^2 & \leq \frac{3\left\|\*y_{1,i}\right\|^2}{(1-\lambda_i)^2} + \frac{3\lambda_i^2}{(1 - \sqrt{\lambda_i})^2(1-\lambda_i)}\sum_{k=1}^{K-1}\left\|\*\beta_{k,i}\right\|^2 + \frac{3}{(1-\sqrt{\lambda_i})^2(1-\lambda_i)}\sum_{k=1}^{K-1}\left\|\*r_{k,i}\right\|^2\\
	& \leq \frac{3\left\|\*y_{1,i}\right\|^2}{(1-\lambda_2)^2} + \frac{3\lambda_2^2}{(1 - \sqrt{\lambda_2})^2(1-\lambda_2)}\sum_{k=1}^{K-1}\left\|\*\beta_{k,i}\right\|^2 + \frac{3}{(1-\sqrt{\lambda_2})^2(1-\lambda_2)}\sum_{k=1}^{K-1}\left\|\*r_{k,i}\right\|^2\\
	\end{align*}
So there exists $C_1, C_2, C_3$
	\begin{align*}
	    C_1 &= \max\left\{\frac{3}{1-|v_n|^2}, \frac{3}{(1-\lambda_2)^2}\right\}\\
	    C_2 &= \max\left\{\frac{3\lambda_n^2}{(1-|v_n|)^2}, \frac{3\lambda_2^2}{(1-\sqrt{\lambda_2})^2(1-\lambda_2)}\right\}\\
	    C_3 &= \max\left\{\frac{3}{(1-|v_n|)^2}, \frac{3}{(1-\sqrt{\lambda_2})^2(1-\lambda_2)}\right\}
	\end{align*}
	\begin{displaymath}
	\sum_{k=1}^{K}\left\|\*y_{k,i}\right\|^2 \leq C_1\left\|\*y_{1,i}\right\|^2 + C_2\sum_{k=1}^{K-1}\left\|\*\beta_{k,i}\right\|^2 + C_3\sum_{k=1}^{K-1}\left\|\*r_{k,i}\right\|^2
	\end{displaymath}
By taking expectation we have
	\begin{displaymath}
	\sum_{k=1}^{K}\mathbb{E}\left\|\*y_{k,i}\right\|^2 \leq C_1\mathbb{E}\left\|\*y_{1,i}\right\|^2 + C_2\sum_{k=1}^{K-1}\mathbb{E}\left\|\*\beta_{k,i}\right\|^2 + C_3\sum_{k=1}^{K-1}\mathbb{E}\left\|\*r_{k,i}\right\|^2
	\end{displaymath}
We next analyze $\beta_{k,i}$:
	\begin{align*}
	& \sum_{i=2}^{n}\mathbb{E}\left\|\*\beta_{k,i}\right\|^2\\
	= & \alpha^2\sum_{i=2}^{n}\mathbb{E}\left\|\*h_{k,i} - \*h_{k-1,i}\right\|^2\\
	= & \alpha^2\sum_{i=2}^{n}\mathbb{E}\left\|\*{\tilde{G}}_{k}\*P\*e_i - \*{\tilde{G}}_{k-1}\*P\*e_i\right\|^2\\
	\leq & \alpha^2\sum_{i=1}^{n}\mathbb{E}\left\|\*{\tilde{G}}_{k}\*P\*e_i - \*{\tilde{G}}_{k-1}\*P\*e_i\right\|^2\\
	\leq & \alpha^2\mathbb{E}\left\|\*{\tilde{G}}_{k}\*P - \*{\tilde{G}}_{k-1}\*P\right\|^2_F\\
	\overset{\text{Lemma~\ref{D2_lemma4}}}{\leq} & \alpha^2\mathbb{E}\left\|\*{\tilde{G}}_{k} - \*{\tilde{G}}_{k-1}\right\|^2_F\\
	= & \alpha^2\sum_{i=1}^{n}\mathbb{E}\left\|\*{\tilde{G}}_{k}\*e_i - \*{\tilde{G}}_{k-1}\*e_i\right\|^2\\
	\leq & 3\alpha^2\sum_{i=1}^{n}\mathbb{E}\left\|\*{\tilde{G}}_{k}\*e_i - \*G_{k}\*e_i\right\|^2 + 3\alpha^2\sum_{i=1}^{n}\mathbb{E}\left\|\*{\tilde{G}}_{k-1}\*e_i - \*G_{k-1}\*e_i\right\|^2\\
	& + 3\alpha^2\sum_{i=1}^{n}\mathbb{E}\left\|\*G_{k}\*e_i - \*G_{k-1}\*e_i\right\|^2\\
	\leq & 6\alpha^2n\sigma^2 + 3\alpha^2\sum_{i=1}^{n}\mathbb{E}\left\|\*G_{k}\*e_i - \*G_{k-1}\*e_i\right\|^2\\
	\leq & 6\alpha^2n\sigma^2 + 3\alpha^2L^2\sum_{i=1}^{n}\mathbb{E}\left\|\*x_{k,i} - \*x_{k-1,i}\right\|^2\\
	\leq & 6\alpha^2n\sigma^2 + 3\alpha^2L^2\sum_{i=1}^{n}\mathbb{E}\left\|\*Y_k\*P^\top \*e_i - \*Y_{k-1}\*P^\top \*e_i\right\|^2\\
	\leq & 6\alpha^2n\sigma^2 + 3\alpha^2L^2\mathbb{E}\left\|\*Y_k\*P^\top - \*Y_{k-1}\*P^\top\right\|^2_F\\
	\overset{\text{Lemma~\ref{D2_lemma4}}}{\leq} & 6\alpha^2n\sigma^2 + 3\alpha^2L^2\mathbb{E}\left\|\*Y_k - \*Y_{k-1}\right\|^2_F\\
	\leq & 6\alpha^2n\sigma^2 + 3\alpha^2L^2\sum_{i=1}^{n}\mathbb{E}\left\|\*y_{k,i} - \*y_{k-1,i}\right\|^2\\
	\end{align*}
Putting it back, we have
	\begin{align*}
	& \sum_{i=2}^{n}\sum_{k=1}^{K}\mathbb{E}\left\|\*y_{k,i}\right\|^2\\
	\leq & C_1\mathbb{E}\left\|\*Y_{1}\right\|^2_F + C_2\sum_{i=2}^{n}\sum_{k=1}^{K-1}\mathbb{E}\left\|\*\beta_{k,i}\right\|^2 + C_3\sum_{k=1}^{K-1}\sum_{i=2}^{n}\mathbb{E}\left\|\*r_{k,i}\right\|^2\\
	\leq & C_1\mathbb{E}\left\|\*Y_{1}\right\|^2_F + C_2\sum_{k=1}^{K-1}\left(6\alpha^2n\sigma^2 + 3\alpha^2L^2\sum_{i=1}^{n}\mathbb{E}\left\|\*y_{k,i} - \*y_{k-1,i}\right\|^2\right) + C_3\sum_{k=1}^{K-1}\sum_{i=2}^{n}\mathbb{E}\left\|\*r_{k,i}\right\|^2\\
	\overset{\text{Lemma~\ref{D2_lemma4}}}{\leq} & C_1\mathbb{E}\left\|\*Y_{1}\right\|^2_F + 6C_2\alpha^2n\sigma^2K + 3C_2\alpha^2L^2\sum_{k=1}^{K-1}\sum_{i=1}^{n}\mathbb{E}\left\|\*y_{k,i} - \*y_{k-1,i}\right\|^2 + C_3\sum_{k=1}^{K-1}\mathbb{E}\left\|\*\Omega_k\right\|^2_F
	\end{align*}
Since
	\begin{align*}
	& \mathbb{E}\left\|\*y_{k,1} - \*y_{k-1,1}\right\|^2
	= \mathbb{E}\left\|\*X_k\*P\*e_1 - \*X_{k-1}\*P\*e_1\right\|^2
	= \mathbb{E}\left\|\*X_k\*v_1 - \*X_{k-1}\*v_1\right\|^2\\
	= & \mathbb{E}\left\|\*X_k\frac{1}{\sqrt{n}}\*1 - X_{k-1}\frac{1}{\sqrt{n}}\*1\right\|^2
	= n\mathbb{E}\left\|\overline{\*X}_k - \overline{\*X}_{k-1}\right\|^2
	= n\alpha^2\mathbb{E}\left\|\overline{\*{\tilde{G}}_k}\right\|^2\\
	\leq & n\alpha^2\mathbb{E}\left\|\overline{\*{\tilde{G}}_k} - \overline{\*G}_k\right\|^2 + n\alpha^2\mathbb{E}\left\|\overline{\*G}_k\right\|^2
	\leq n\alpha^2\frac{\sigma^2}{n} + n\alpha^2\mathbb{E}\left\|\overline{\*G}_k\right\|^2\\
	= & \alpha^2\sigma^2 + n\alpha^2\mathbb{E}\left\|\overline{\*G}_k\right\|^2
	\end{align*}
Putting it back, and we obtain
	\begin{align*}
	&\sum_{i=2}^{n}\sum_{k=1}^{K}\mathbb{E}\left\|\*y_{k,i}\right\|^2\\
	\leq & C_1\mathbb{E}\left\|\*Y_{1}\right\|^2_F + 6C_2\alpha^2n\sigma^2K + 3C_2\alpha^4\sigma^2L^2K + 3C_2\alpha^4nL^2\sum_{k=1}^{K-1}\mathbb{E}\left\|\overline{\*G}_k\right\|^2\\
	& + 3C_2\alpha^2L^2\sum_{k=1}^{K-1}\sum_{i=2}^{n}\mathbb{E}\left\|\*y_{k,i} - \*y_{k-1,i}\right\|^2 + C_3\sum_{k=1}^{K-1}\mathbb{E}\left\|\*\Omega_k\right\|^2_F\\
	\leq & C_1\mathbb{E}\left\|\*Y_{1}\right\|^2_F + 6C_2\alpha^2n\sigma^2K + 3C_2\alpha^4\sigma^2L^2K + 3C_2\alpha^4nL^2\sum_{k=1}^{K-1}\mathbb{E}\left\|\overline{\*G}_k\right\|^2\\
	& + 6C_2\alpha^2L^2\sum_{k=1}^{K-1}\sum_{i=2}^{n}\mathbb{E}\left(\left\|\*y_{k,i}\right\|^2 + \left\|\*y_{k-1,i}\right\|^2\right) + C_3\sum_{k=1}^{K-1}\mathbb{E}\left\|\*\Omega_k\right\|^2_F\\
	\leq & C_1\mathbb{E}\left\|\*Y_{1}\right\|^2_F + 6C_2\alpha^2n\sigma^2K + 3C_2\alpha^4\sigma^2L^2K + 3C_2\alpha^4nL^2\sum_{k=1}^{K-1}\mathbb{E}\left\|\overline{\*G}_k\right\|^2\\
	& + 12C_2\alpha^2L^2\sum_{k=1}^{K-1}\sum_{i=2}^{n}\mathbb{E}\left\|\*y_{k,i}\right\|^2 + C_3\sum_{k=1}^{K-1}\mathbb{E}\left\|\*\Omega_k\right\|^2_F\\
	\end{align*}
	Rearrange the terms, we get
	\begin{align*}
	& (1 - 12C_2\alpha^2L^2)\sum_{i=2}^{n}\sum_{k=1}^{K}\mathbb{E}\left\|\*y_{k,i}\right\|^2\\
	\leq & C_1\mathbb{E}\left\|\*Y_{1}\right\|^2_F + 6C_2\alpha^2n\sigma^2K + 3C_2\alpha^4\sigma^2L^2K + 3C_2\alpha^4nL^2\sum_{k=1}^{K-1}\mathbb{E}\left\|\overline{\*G}_k\right\|^2 + C_3\sum_{k=1}^{K-1}\mathbb{E}\left\|\*\Omega_k\right\|^2_F\\
	\leq & C_1\mathbb{E}\left\|\*X_{1}\right\|^2_F + 6C_2\alpha^2n\sigma^2K + 3C_2\alpha^4\sigma^2L^2K + 3C_2\alpha^4nL^2\sum_{k=1}^{K-1}\mathbb{E}\left\|\overline{\*G}_k\right\|^2 + C_3\sum_{k=1}^{K-1}\mathbb{E}\left\|\*\Omega_k\right\|^2_F
	\end{align*}
Considering
	\begin{align*}
	\mathbb{E}\left\|\*X_1\right\|^2_F & = \alpha^2\mathbb{E}\left\|\*{\tilde{G}}_0\right\|^2_F\\
	& = \alpha^2\sum_{i=1}^{n}\mathbb{E}\left\|\*{\tilde{G}}_{0,i} - \*G_{0,i} + \*G_{0,i} - \nabla f(\*0) + \nabla f(\*0)\right\|^2\\
	& \leq 3\alpha^2\sum_{i=1}^{n}\mathbb{E}\left\|\*{\tilde{G}}_{0,i} - \*G_{0,i}\right\|^2 + 3\alpha^2\sum_{i=1}^{n}\mathbb{E}\left\|\*G_{0,i} - \nabla f(\*0)\right\|^2 + 3\alpha^2\sum_{i=1}^{n}\mathbb{E}\left\|\nabla f(\*0)\right\|^2\\
	& \leq 3\alpha^2n\sigma^2 + 3\alpha^2n\varsigma_0^2 + 3\alpha^2n\mathbb{E}\left\|\nabla f(\*0)\right\|
	\end{align*}
We finally get
	\begin{align*}
	& (1 - 12C_2\alpha^2L^2)\sum_{i=2}^{n}\sum_{k=1}^{K}\mathbb{E}\left\|\*y_{k,i}\right\|^2\\
	= & (1 - 12C_2\alpha^2L^2)\sum_{i=1}^{n}\sum_{k=1}^{K}\mathbb{E}\left\|\overline{\*X}_k - \*x_{k,i}\right\|^2\\
	\leq & 3C_1\alpha^2n\sigma^2 + 3C_1\alpha^2n\varsigma_0^2 + 3C_1\alpha^2n\mathbb{E}\left\|\nabla f(\*0)\right\| + 6C_2\alpha^2n\sigma^2K + 3C_2\alpha^4\sigma^2L^2K\\
	& + 3C_2\alpha^4nL^2\sum_{k=1}^{K-1}\mathbb{E}\left\|\overline{\*G}_k\right\|^2 + C_3\sum_{k=1}^{K-1}\mathbb{E}\left\|\*\Omega_k\right\|^2_F
	\end{align*}
That completes the proof.
\end{proof}

\begin{lemma}\label{D2_lemma3}
	Given $\rho\in\left(-\frac{1}{3}, 0\right)\cup\left(0, 1\right)$, for any two sequence $\{a_t\}_{t=1}^{\infty}$, $\{b_t\}_{t=1}^{\infty}$ and $\{c_t\}_{t=1}^{\infty}$ that satisfying
	\begin{displaymath}
	\begin{split}
	a_0 & = b_0 = 0,\\
	a_{t+1} & = \rho\left(2a_t - a_{t-1}\right) + b_t - b_{t-1} + c_t, \forall t \geq 1
	\end{split}
	\end{displaymath}
	we have
	\begin{displaymath}
	a_{t+1} = a_1\left(\frac{u^{t+1} - v^{t+1}}{u - v}\right) + \sum_{s=1}^{t}(b_s - b_{s-1} + c_s)\left(\frac{u^{t-s+1} - v^{t-s+1}}{u - v}\right), \forall t\geq 0
	\end{displaymath}
	where
	\begin{displaymath}
	u = \rho + \sqrt{\rho^2 - \rho}, v = \rho - \sqrt{\rho^2 - \rho}
	\end{displaymath}
	Moreover, if $0<\rho<1$, we have
	\begin{displaymath}
    a_{t+1} = a_1\rho^{\frac{t}{2}}\frac{\sin[(t+1)\phi]}{\sin\phi} + \sum_{s=1}^{t}(b_s - b_{s-1} + c_s)\rho^{\frac{t-s}{2}}\frac{\sin[(t-s+1)\phi]}{\sin\phi}
    \end{displaymath}
	where
	\begin{displaymath}
	\phi = \arccos\left(\sqrt{\rho}\right)
	\end{displaymath}
\end{lemma}
\begin{proof}
when $t \geq 1$, we have
\begin{displaymath}
a_{t+1} = 2\rho a_t - \rho a_{t-1} + b_t - b_{t-1} + c_t
\end{displaymath}
since,
\begin{displaymath}
u = \rho + \sqrt{\rho^2 - \rho}, v = \rho - \sqrt{\rho^2 - \rho}
\end{displaymath}
we obtain
\begin{displaymath}
a_{t+1} - ua_t = (a_t - ua_{t-1})v + b_t - b_{t-1} + c_t
\end{displaymath}
Recursively we have
\begin{displaymath}
\begin{split}
    a_{t+1} - ua_t & = (a_t - ua_{t-1})v + b_t - b_{t-1} + c_t\\
    & = (a_{t-1} - ua_{t-2})v^2 + (b_{t-1} - b_{t-2} + c_{t-1})v + b_t - b_{t-1} + c_t\\
    & = (a_1 - ua_0)v^t + \sum_{s=1}^{t}(b_s - b_{s-1} + c_s)v^{t-s}\\
    & = a_1v^t + \sum_{s=1}^{t}(b_s - b_{s-1} + c_s)v^{t-s}
\end{split}
\end{displaymath}
Dividing both sides by $u^{t+1}$, we have
\begin{displaymath}
\begin{split}
    \frac{a_{t+1}}{u^{t+1}} & = \frac{a_t}{u^t} + u^{-(t+1)}\left(a_1v^t + \sum_{s=1}^{t}(b_s - b_{s-1} + c_s)v^{t-s}\right)\\
    & = \frac{a_{t-1}}{u^{t-1}} + u^{-t}\left(a_1v^{t-1} + \sum_{s=1}^{t-1}(b_s - b_{s-1} + c_s)v^{t-1-s}\right)\\
    & + u^{-(t+1)}\left(a_1v^t + \sum_{s=1}^{t}(b_s - b_{s-1} + c_s)v^{t-s}\right)\\
    & = \frac{a_1}{u} + \sum_{k=1}^{t}u^{-k-1}\left(a_1v^k + \sum_{s=1}^{k}(b_s - b_{s-1} + c_s)v^{k-s}\right)
\end{split}
\end{displaymath}
Multiplying both sides by $u^{t+1}$
\begin{displaymath}
\begin{split}
    a_{t+1} & = a_1u^t + \sum_{k=1}^{t}u^{t-k}\left(a_1v^k + \sum_{s=1}^{k}(b_s - b_{s-1} + c_s)v^{t-s}\right)\\
    & = a_1u^t\left(1 + \sum_{k=1}^{t}\left(\frac{v}{u}\right)^k\right) + u^t\sum_{k=1}^{t}\sum_{s=1}^{k}(b_s - b_{s-1} + c_s)v^{-s}\left(\frac{v}{u}\right)^k\\
    & = a_1u^t\sum_{k=0}^{t}\left(\frac{v}{u}\right)^k + u^t\sum_{s=1}^{t}\sum_{k=s}^{t}(b_s - b_{s-1} + c_s)v^{-s}\left(\frac{v}{u}\right)^k\\
    & = a_1u^t\left(\frac{1-\left(\frac{v}{u}\right)^{t+1}}{1-\frac{v}{u}}\right) + u^t\sum_{s=1}^{t}(b_s - b_{s-1} + c_s)v^{-s}\left(\frac{v}{u}\right)^s\frac{1-\left(\frac{v}{u}\right)^{t-s-1}}{1-\frac{v}{u}}\\
    & = a_1\left(\frac{u^{t+1} - v^{t+1}}{u-v}\right) + \sum_{s=1}^{t}(b_s - b_{s-1} + c_s)\frac{u^{t-s+1} - v^{t-s+1}}{u-v}
\end{split}
\end{displaymath}
Note that when $0 < \rho < 1$, both $u$ and $v$ are complex numbers, we have
\begin{displaymath}
u = \sqrt{\rho}e^{i\phi}, v = \sqrt{\rho}e^{-i\phi}
\end{displaymath}
where $\phi = \arccos{\sqrt{\rho}}$. And under this context, we have
\begin{displaymath}
a_{t+1} = a_1\rho^{\frac{t}{2}}\frac{\sin[(t+1)\phi]}{\sin\phi} + \sum_{s=1}^{t}(b_s - b_{s-1} + c_s)\rho^{\frac{t-s}{2}}\frac{\sin[(t-s+1)\phi]}{\sin\phi}
\end{displaymath}
That completes the proof.
\end{proof}

\begin{lemma}\label{D2_lemma4}
	For any matrix $\*X\in\mathbb{R}^{N\times n}$, we have
	\begin{displaymath}
	\begin{split}
	\sum_{i=2}^{n}\left\|\*X\*v_i\right\|^2 & \leq \sum_{i=1}^{n}\left\|\*X\*v_i\right\|^2 = \left\|\*X\right\|^2_F\\
	\sum_{i=1}^{n}\left\|\*X\*P^\top \*e_i\right\|^2 & = \left\|\*X\*P^\top\right\|^2_F = \left\|\*X\right\|^2_F
	\end{split}
	\end{displaymath}
\end{lemma}
\begin{proof}
\begin{displaymath}
\sum_{i=2}^{n}\left\|\*X_t\*v_i\right\|^2 \leq \sum_{i=1}^{n}\left\|\*X_t\*v_i\right\|^2 = \left\|\*X_t\*P\right\|^2_F = Tr(\*X_t\*P\*P^\top \*X_t^\top) = Tr(\*X_t\*X_t^\top) = \left\|\*X_t\right\|^2_F
\end{displaymath}
And similarly,
\begin{displaymath}
\sum_{i=1}^{n}\left\|\*X\*P^\top \*e_i\right\|^2 = \left\|\*X\*P^\top\right\|^2_F = Tr(\*X_t\*P^\top \*P\*X_t^\top) = Tr(\*X_t\*X_t^\top) = \left\|\*X_t\right\|^2_F
\end{displaymath}
That completes the proof.
\end{proof}

\begin{lemma}\label{D2_lemma5}
If we run Algorithm~\ref{D2algo} for $K$ iterations the following inequality holds:
\begin{align*}
&\left(1-\frac{3C_1\alpha^2L^2}{C_4}\right)\mathbb{E}\left\|\nabla f(\*0)\right\| + \left(1-\alpha L-3\frac{C_2}{C_4}\alpha^4L^4\right)\frac{1}{K}\sum_{k=1}^{K-1}\mathbb{E}\left\|\overline{\*G}_k\right\|^2 + \frac{1}{K}\sum_{k=0}^{K-1}\mathbb{E}\left\|\nabla f(\overline{\*X}_k)\right\|^2\\ 
\leq & \frac{2(f(0) - f^*)}{\alpha K} + \frac{\alpha L}{n}\sigma^2 + \frac{3C_1\alpha^2L^2(\sigma^2+\varsigma_0^2)}{C_4K} + 6\frac{C_2}{C_4}\alpha^2\sigma^2L^2 + 3\frac{C_2}{nC_4}\alpha^4\sigma^2L^4\\
+ & \frac{C_3L^2}{C_4}\left(\frac{6D_1n+8}{6D_2n+1}\right)^2\alpha^2G_\infty^2d
\end{align*}
where
\begin{align*}
    C_1 &= \max\left\{\frac{3}{1-|v_n|^2}, \frac{3}{(1-\lambda_2)^2}\right\}\\
    C_2 &= \max\left\{\frac{3\lambda_n^2}{(1-|v_n|)^2}, \frac{3\lambda_2^2}{(1-\sqrt{\lambda_2})^2(1-\lambda_2)}\right\}\\
    C_3 &= \max\left\{\frac{3}{(1-|v_n|)^2}, \frac{3}{(1-\sqrt{\lambda_2})^2(1-\lambda_2)}\right\}\\
    C_4 &= 1 - 12C_2\alpha^2L^2
\end{align*}
\end{lemma}
\begin{proof}
Since
\begin{align*}
\overline{\*X}_{k+1} & = (2\*X_k - \*X_{k-1} - \alpha\*{\tilde{G}}_k + \alpha\*{\tilde{G}}_{k-1})\*W\frac{\*1}{n} + (\*{\hat{X}}_{k+\frac{1}{2}}-\*X_{k+\frac{1}{2}})(\*W-\*I)\frac{\*1}{n}\\
& = 2\overline{\*X}_k - \overline{\*X}_{k-1} - \alpha\overline{\*{\tilde{G}}}_k + \alpha\overline{\*{\tilde{G}}}_{k-1}
\end{align*}
and we have
\begin{align*}
\overline{\*X}_{k+1} - \overline{\*X}_k & = \overline{\*X}_k - \overline{\*X}_{k-1} - \alpha\overline{\*{\tilde{G}}}_k + \alpha\overline{\*{\tilde{G}}}_{k-1}\\
& = \overline{\*X}_1 - \overline{\*X}_{0} - \alpha\sum_{t=1}^{k}(\overline{\*{\tilde{G}}}_t - \overline{\*{\tilde{G}}}_{t-1})\\
& = -\alpha\overline{\*{\tilde{G}}}_k
\end{align*}
Note that the update of the averaged model is exactly the same as D-PSGD, thus we can reuse the result from D-PSGD for $D^2$ as follows:
\begin{displaymath}
\frac{1 - \alpha L}{K}\sum_{k=0}^{K-1}\mathbb{E}\left\|\overline{\*G}_k\right\|^2 + \frac{1}{K}\sum_{k=0}^{K-1}\mathbb{E}\left\|\nabla f(\overline{\*X}_k)\right\|^2 
\leq \frac{2(f(0) - f^*)}{\alpha K} + \frac{\alpha L}{n}\sigma^2 + \frac{ L^2}{nK}\sum_{k=0}^{K-1}\sum_{i=1}^{n}\mathbb{E}\left\|\overline{\*X}_k - \*x_{k,i}\right\|^2
\end{displaymath}
From Lemma~\ref{D2_lemma2} we obatin
\begin{align*}
&\frac{1 - \alpha L}{K}\sum_{k=0}^{K-1}\mathbb{E}\left\|\overline{\*G}_k\right\|^2 + \frac{1}{K}\sum_{k=0}^{K-1}\mathbb{E}\left\|\nabla f(\overline{\*X}_k)\right\|^2\\ 
\leq & \frac{2(f(\*0) - f^*)}{\alpha K} + \frac{\alpha L}{n}\sigma^2 + \frac{3C_1\alpha^2L^2(\sigma^2+\varsigma_0^2+\mathbb{E}\left\|\nabla f(\*0)\right\|)}{C_4K} + 6\frac{C_2}{C_4}\alpha^2\sigma^2L^2 + 3\frac{C_2}{nC_4}\alpha^4\sigma^2L^4\\
+ & 3\frac{C_2}{C_4}\alpha^4L^4\frac{1}{K}\sum_{k=1}^{K-1}\mathbb{E}\left\|\overline{\*G}_k\right\|^2 + \frac{C_3L^2}{C_4nK}\sum_{k=1}^{K-1}\mathbb{E}\left\|\*\Omega_k\right\|^2_F
\end{align*}
Rearrange the terms, we get
\begin{align*}
&\left(1-\frac{3C_1\alpha^2L^2}{C_4}\right)\mathbb{E}\left\|\nabla f(\*0)\right\| + \left(1-\alpha L-3\frac{C_2}{C_4}\alpha^4L^4\right)\frac{1}{K}\sum_{k=1}^{K-1}\mathbb{E}\left\|\overline{\*G}_k\right\|^2 + \frac{1}{K}\sum_{k=0}^{K-1}\mathbb{E}\left\|\nabla f(\overline{\*X}_k)\right\|^2\\ 
\leq & \frac{2(f(0) - f^*)}{\alpha K} + \frac{\alpha L}{n}\sigma^2 + \frac{3C_1\alpha^2L^2(\sigma^2+\varsigma_0^2)}{C_4K} + 6\frac{C_2}{C_4}\alpha^2\sigma^2L^2 + 3\frac{C_2}{nC_4}\alpha^4\sigma^2L^4 + \frac{C_3L^2}{C_4nK}\sum_{k=1}^{K-1}\mathbb{E}\left\|\*\Omega_k\right\|^2_F
\end{align*}
Similar to the case in D-PSGD, we have
\begin{align*}
\sum_{k=0}^{K-1}\mathbb{E}\left\|\*\Omega_k\right\|^2_F
= & \sum_{k=0}^{K-1}\sum_{i=1}^{n}\mathbb{E}\left\|\sum_{j=1}^{n}\left((\*{\hat{x}}_{k+{\frac{1}{2}},j}-\*x_{k+{\frac{1}{2}},j}) - (\*{\hat{x}}_{k+{\frac{1}{2}},i}-\*x_{k+{\frac{1}{2}},i})\right)\*W_{ji}\right\|^2\\
\overset{\text{Lemma}~\ref{modifynoise}}{\leq} & 4\sum_{k=0}^{K-1}\sum_{i=1}^{n}\delta^2 B_\theta^2d
\leq \left(\frac{6D_1n+8}{6D_2n+1}\right)^2\alpha^2G_\infty^2dnK
\end{align*}
Putting it back, we obtain
\begin{align*}
&\left(1-\frac{3C_1\alpha^2L^2}{C_4}\right)\mathbb{E}\left\|\nabla f(\*0)\right\| + \left(1-\alpha L-3\frac{C_2}{C_4}\alpha^4L^4\right)\frac{1}{K}\sum_{k=1}^{K-1}\mathbb{E}\left\|\overline{\*G}_k\right\|^2 + \frac{1}{K}\sum_{k=0}^{K-1}\mathbb{E}\left\|\nabla f(\overline{\*X}_k)\right\|^2\\
\leq & \frac{2(f(\*0) - f^*)}{\alpha K} + \frac{\alpha L}{n}\sigma^2 + \frac{3C_1\alpha^2L^2(\sigma^2+\varsigma_0^2)}{C_4K} + 6\frac{C_2}{C_4}\alpha^2\sigma^2L^2 + 3\frac{C_2}{nC_4}\alpha^4\sigma^2L^4 + \frac{C_3L^2}{C_4}\left(\frac{6D_1n+8}{6D_2n+1}\right)^2\alpha^2G_\infty^2d
\end{align*}
That completes the proof.
\end{proof}
\section{Moniqua on AD-PSGD (Proof to Theorem~\ref{thmMoniquaAD-PSGD})}\label{Async Moniqua}
\subsection{Definition and Notation}
In the original analysis of AD-PSGD, to better capture the nature of workers computing at different speed, the objective function is expressed as 
\begin{align*}
f(\*x) = \sum_{i=1}^{n}p_if_i(\*x)
\end{align*}
where $p_i$ is a parameter denoting the speed of $i$-th worker gradient updates. In the rest of the proof, we denote $p=\max_i\{p_i\}$

For simplicity, we also define the following terms
\begin{align*}
\nabla F(\*X_k) & = n\left[p_1\*g_{k,1}, \cdots, p_n\*g_{k,n}\right]\in\mathbb{R}^{d\times n}\\
\nabla\widetilde{F}(\*X_k) & = n\left[p_1\*{\tilde{g}}_{k,1}, \cdots, p_n\*{\tilde{g}}_{k,n}\right]\in\mathbb{R}^{d\times n}\\
\*{\tilde{G}}_k & = \left[\cdots, \*{\tilde{g}}_{k,i_k}, \cdots\right]\\
\*G_k & = \left[\cdots, \*g_{k,i_k}, \cdots\right]\\
\*\Lambda_{a}^{b} & = \frac{\*1\*1^\top}{n} - \prod_{q=a}^{b}\*W_q
\end{align*}
\subsection{Setting}
The pseudo code can be found in Algorithm~\ref{async_algo}.
\begin{algorithm}[t]
	\caption{Moniqua with Asynchronous Communication}\label{async_algo}
	\begin{algorithmic}[1]
		\Require initial point $\*x_{0,i} = \*x_0$, step size $\alpha$, the discrepency bound $B_\theta$, number of iterations $K$, quantization function $\mathcal{Q}_\delta$, initial random seed
		
		\For{$k=0,1,2,\cdots,K-1$}
		\State worker $i_k$ is updating the gradient while during this iteration the global communication behaviour is written in the form of $\*W_k$.
		
		\State Compute a local stochastic gradient with model delayed by $\tau_k$: $\*{\tilde{g}}_{k-\tau_k,i_k}$
		
		\State Send modulo-ed model to one randomly selected neighbor $j_k$: $\*q_{k,i_k}\leftarrow \mathcal{Q}_\delta\left(\frac{\*x_{k,i_k}}{B_\theta} \bmod 1\right)$
		
		\State Compute local biased term $\*{\hat{x}}_{k,i_k}$ as:
		\begin{align*}
		    \*{\hat{x}}_{k,i_k}=\*q_{k,i_k}B_\theta-\*x_{k,i_k}\bmod B_\theta + \*x_{k,i_k}
		\end{align*}
		
		\State Randomly select one neighbor $j_k$ and recover its model as:
		\begin{align*}
		    \*{\hat{x}}_{k,j_k}=(\*q_{k,j_k}B_\theta-\*x_{k,i_k})\bmod B_\theta + \*x_{k,i}
		\end{align*}
		
		\State Average with neighboring workers: $\*x_{k,i_k} \leftarrow \*x_{k,i_k} + \sum_{j\in\mathcal{N}_i}(\*{\hat{x}}_{k,j_k}-\*{\hat{x}}_{k,i_k})\*W_{ji}$
		
		\State Update the local weight with local gradient: $\*x_{k+1,i_k} \leftarrow \*x_{k,i_k} - \alpha \*{\tilde{g}}_{k-\tau_k,i_k}$		
		\EndFor
		\Ensure $\overline{\*X}_K=\frac{1}{n}\sum\nolimits_{i=1}^{n}\*x_{K,i}$
	\end{algorithmic}
\end{algorithm}
We makes the following assumptions:
\begin{enumerate}
\item \textbf{Lipschitzian Gradient}: All the function $f_i$ have L-Lipschitzian gradients.\label{A_Assumption1}
\item \textbf{Communication Matrix \footnote{Please refer to Section~\ref{matrix} for more details}}: The communication matrix $\*W_k$ is doubly stochastic for any $k\geq 0$ and for any $b\geq a\geq 0$, there exists $\tmix{}$ such that
\begin{align*}
\left\|\prod_{q=a}^{b}\*W_q\left(\*I - \frac{\*1\*1^\top}{n}\right)\right\|_1 \leq 2\cdot 2^{-\left\lfloor\frac{b-a+1}{\tmix{}}\right\rfloor}
\end{align*}\label{A_Assumption2}
\item \textbf{Bounded Variance}:
\begin{align*}
& \mathbb{E}_{\xi_i\sim\mathcal{D}_i}\left\|\nabla\tilde{f}_i(\*x;\xi_i) - \nabla f_i(\*x)\right\|^2 \leq \sigma^2, \forall i\\
& \mathbb{E}_{i\sim\{1, \cdots, n\}}\left\|\nabla f_i(\*x) - \nabla f(\*x)\right\|^2 \leq \varsigma^2, \forall i
\end{align*}
where $\nabla\tilde{f}_i(\*x;\xi_i)$ denotes gradient sample on worker $i$ computed via data sample $\xi_i$.\label{A_Assumption3}
\item \textbf{Bounded Staleness}: There exists $T$ such that $\tau_k\leq T, \forall k$ \label{A_Assumption4}
\item \textbf{Gradient magnitude}: The norm of a sampled gradient is bounded by $\left\|\*{\tilde{g}}_{k,i}\right\|_\infty \leq G_\infty$ for some constant $G_\infty$. \label{A_Assumption5}
\end{enumerate}

\subsection{Proof to Theorem~\ref{thmMoniquaAD-PSGD}.}
\begin{proof}
We start from
\begin{align*}
&\frac{1}{K}\sum_{k=0}^{K-1}\mathbb{E}\left\|\nabla f(\overline{\*X}_{k})\right\|^2 + \left(1-\frac{2\alpha L}{n}\right)\frac{1}{K}\sum_{k=0}^{K-1}\mathbb{E}\left\|\nabla\overline{F}(\*X_{k-\tau_k})\right\|^2\\
\overset{\text{Lemma~\ref{Async_lemma4}}}{\leq} & \frac{2n(f(\*0) - f^*)}{\alpha K} + \frac{(\sigma^2+6\varsigma^2)\alpha L}{n} + \left(2L^2 + \frac{12\alpha L^3}{n}\right)\frac{1}{K}\sum_{k=0}^{K-1}\sum_{i=1}^{n}p_i\mathbb{E}\left\|\*X_{k-\tau_k}\left(\frac{\*1}{n} - \*e_i\right)\right\|^2\\
+ & \frac{2L^2}{K}\sum_{k=0}^{K-1}\mathbb{E}\left\|\frac{(\*X_k - \*X_{k-\tau_k})\*1}{n}\right\|^2\\
\overset{\text{Lemma~\ref{Async_lemma5}}}{\leq} & \frac{2n(f(\*0) - f^*)}{\alpha K} + \frac{(\sigma^2+6\varsigma^2)\alpha L}{n} + \frac{2\alpha^2T^2(\sigma^2+6\varsigma^2)L^2}{n^2} + \frac{4\alpha^2T^2L^2}{n^2K}\sum_{k=0}^{K-1}\mathbb{E}\left\| \sum_{i=1}^{n}p_i\*g_{k-\tau_k,i}\right\|^2\\
+ & \left(2L^2 + \frac{12\alpha L^3}{n} + \frac{24L^4\alpha^2T^2}{n^2}\right)\frac{1}{K}\sum_{k=0}^{K-1}\sum_{i=1}^{n}p_i\mathbb{E}\left\|\*X_{k-\tau_k}\left(\frac{\*1}{n} - e_i\right)\right\|^2\\
\overset{\text{Lemma~\ref{Async_lemma3}}}{\leq} & \frac{2n(f(\*0) - f^*)}{\alpha K} + \frac{(\sigma^2+6\varsigma^2)\alpha L}{n} + \frac{2\alpha^2T^2(\sigma^2+6\varsigma^2)L^2}{n^2} + \frac{4\alpha^2T^2L^2}{n^2K}\sum_{k=0}^{K-1}\mathbb{E}\left\| \sum_{i=1}^{n}p_i\*g_{k-\tau_k,i}\right\|^2\\
+ & \frac{128\alpha^2\tmix^2{}L^2}{A_1}\left((\sigma^2+6\varsigma^2)p + \frac{2p}{K}\sum_{k=0}^{K-1}\mathbb{E}\left\| \sum_{i=1}^{n}p_i\*g_{k-\tau_k,i}\right\|^2 + G_\infty^2d\right)\\
\end{align*}
where $A_1 = 1-192p\alpha^2\tmix^2{}L^2$ as defined in Lemma~\ref{Async_lemma3}.

Rearrange the terms, we get
\begin{align*}
\frac{1}{K}\sum_{k=0}^{K-1}\mathbb{E}\left\|\nabla f(\overline{\*X}_{k})\right\|^2 & \leq \frac{2n(f(\*0) - f^*)}{\alpha K} + \frac{(\sigma^2+6\varsigma^2)\alpha L}{n} + \frac{2\alpha^2T^2(\sigma^2+6\varsigma^2)L^2}{n^2}\\
& + \frac{128p\alpha^2\tmix^2{}L^2}{A_1}(\sigma^2+6\varsigma^2) + \frac{128\alpha^2\tmix^2{}L^2}{A_1}G_\infty^2d
\end{align*}
By setting $\alpha=\frac{n}{2L + \sqrt{K(\sigma^2+6\varsigma^2)}}$
\begin{align*}
\frac{1}{K}\sum_{k=0}^{K-1}\mathbb{E}\left\|\nabla f(\overline{\*X}_{k})\right\|^2 \lesssim & \frac{1}{K} + \frac{\sqrt{\sigma^2+6\varsigma^2}}{\sqrt{K}} + \frac{p\tmix^2{}(\sigma^2+6\varsigma^2)n^2}{(\sigma^2+6\varsigma^2)K+4L^2} + \frac{n^2\tmix^2{}G_\infty^2d}{(\sigma^2+6\varsigma^2)K+4L^2}\\
\lesssim & \frac{1}{K} + \frac{\sqrt{\sigma^2+6\varsigma^2}}{\sqrt{K}} + \frac{(\sigma^2+6\varsigma^2)\tmix^2{}n^2}{(\sigma^2+6\varsigma^2)K+1} + \frac{n^2\tmix^2{}G_\infty^2d}{(\sigma^2+6\varsigma^2)K+1}
\end{align*}
\end{proof}


\subsection{Lemma for Moniqua on AD-PSGD}
\begin{lemma}\label{Async_lemma1}
\begin{displaymath}
\mathbb{E}\left\|\*{\tilde{G}}_{k-\tau_k}\frac{\*1}{n}\right\|^2 \leq \frac{\sigma^2}{n^2} + \frac{1}{n^2}\sum_{i=1}^{n}p_i\mathbb{E}\left\|\*g_{k-\tau_k,i}\right\|^2, \forall k\geq 0.
\end{displaymath}
\end{lemma}
\begin{proof}
\begin{align*}
\mathbb{E}\left\|\*{\tilde{G}}_{k-\tau_k}\frac{\*1}{n}\right\|^2 & \leq \sum_{i=1}^{n}p_i\mathbb{E}\left\|\frac{\*{\tilde{g}}_{k-\tau_k,i}}{n}\right\|^2\\
& = \sum_{i=1}^{n}p_i\mathbb{E}\left\|\frac{\*{\tilde{g}}_{k-\tau_k,i} - \*g_{k-\tau_k,i}}{n}\right\|^2 + \sum_{i=1}^{n}p_i\mathbb{E}\left\|\frac{\*g_{k-\tau_k,i}}{n}\right\|^2\\
& \leq \frac{\sigma^2}{n^2} + \frac{1}{n^2}\sum_{i=1}^{n}p_i\mathbb{E}\left\|\*g_{k-\tau_k,i}\right\|^2
\end{align*}
\end{proof}

\begin{lemma}\label{Async_lemma2}	\begin{displaymath}
	\sum_{i=1}^{n}p_i\mathbb{E}\left\|\*g_{k-\tau_k,i}\right\|^2 \leq 12L^2\sum_{i=1}^{n}p_i\mathbb{E}\left\|\*X_{k-\tau_k}\left(\frac{\*1}{n} - \*e_i\right)\right\|^2 + 6\varsigma^2 + 2\mathbb{E}\left\| \sum_{i=1}^{n}p_i\*g_{k-\tau_k,i}\right\|^2, \forall k \geq 0.
\end{displaymath}
\end{lemma}
\begin{proof}
\begin{align*}
	\sum_{i=1}^{n}p_i\mathbb{E}\left\|\*g_{k-\tau_k,i}\right\|^2 & = \sum_{i=1}^{n}p_i\mathbb{E}\left\|\*g_{k-\tau_k,i} -  \sum_{i=1}^{n}p_i\*g_{k-\tau_k,i} +  \sum_{i=1}^{n}p_i\*g_{k-\tau_k,i}\right\|^2\\
	& \leq 2\sum_{i=1}^{n}p_i\mathbb{E}\left\|\*g_{k-\tau_k,i} -  \sum_{i=1}^{n}p_i\*g_{k-\tau_k,i}\right\|^2 + 2\sum_{i=1}^{n}p_i\mathbb{E}\left\| \sum_{i=1}^{n}p_i\*g_{k-\tau_k,i}\right\|^2\\
	& = 2\sum_{i=1}^{n}p_i\mathbb{E}\left\|\*g_{k-\tau_k,i} -  \sum_{i=1}^{n}p_i\*g_{k-\tau_k,i}\right\|^2 + 2\mathbb{E}\left\| \sum_{i=1}^{n}p_i\*g_{k-\tau_k,i}\right\|^2
\end{align*}
And
\begin{align*}
&\sum_{i=1}^{n}p_i\mathbb{E}\left\|\*g_{k-\tau_k,i} -  \sum_{i=1}^{n}p_i\*g_{k-\tau_k,i}\right\|^2\\
\leq & 3\sum_{i=1}^{n}p_i\mathbb{E}\left\|\*g_{k-\tau_k,i} - \nabla f_i(\overline{\*X}_{k-\tau_k})\right\|^2 + 3\sum_{i=1}^{n}p_i\mathbb{E}\left\|\nabla f_i(\overline{\*X}_{k-\tau_k}) - \sum_{j=1}^{n}p_j\nabla f_j(\overline{\*X}_{k-\tau_k})\right\|^2\\
& + 3\sum_{i=1}^{n}p_i\mathbb{E}\left\| \sum_{i=1}^{n}p_i\*g_{k-\tau_k,i} - \sum_{j=1}^{n}p_j\nabla f_j(\overline{\*X}_{k-\tau_k})\right\|^2\\
\leq & 3L^2\sum_{i=1}^{n}p_i\mathbb{E}\left\|\*x_{k-\tau_k,i} - \overline{\*X}_{k-\tau_k}\right\|^2 + 3\sum_{i=1}^{n}p_i\mathbb{E}\left\|\nabla f_i(\overline{\*X}_{k-\tau_k}) - \sum_{j=1}^{n}p_j\nabla f_j(\overline{\*X}_{k-\tau_k})\right\|^2\\
& + 3\mathbb{E}\left\| \sum_{i=1}^{n}p_i\*g_{k-\tau_k,i} - \sum_{j=1}^{n}p_j\nabla f_j(\overline{\*X}_{k-\tau_k})\right\|^2\\
\leq & 3L^2\sum_{i=1}^{n}p_i\mathbb{E}\left\|\*X_{k-\tau_k}\left(\frac{\*1}{n} - e_i\right)\right\|^2 + 3\sum_{i=1}^{n}p_i\mathbb{E}\left\|\nabla f_i(\overline{\*X}_{k-\tau_k}) - \nabla f(\overline{\*X}_{k-\tau_k})\right\|^2\\
& + 3\sum_{j=1}^{n}p_j\mathbb{E}\left\|\*g_{k-\tau_k,j} - \nabla f_j(\overline{\*X}_{k-\tau_k})\right\|^2\\
\leq & 6L^2\sum_{i=1}^{n}p_i\mathbb{E}\left\|\*X_{k-\tau_k}\left(\frac{\*1}{n} - \*e_i\right)\right\|^2 + 3\varsigma^2
\end{align*}
That completes the proof.
\end{proof}

\begin{lemma}\label{Async_lemma3}
Let $A_1 = 1-192p\alpha^2\tmix^2{}L^2$,
\begin{align*}
\sum_{k=0}^{K-1}\sum_{i=1}^{n}p_i\mathbb{E}\left\|\*X_{{k-\tau_k}}\left(\frac{\*1}{n} - e_i\right)\right\|^2
\leq & \frac{32\alpha^2\tmix^2{}}{A_1}\left((\sigma^2+6\varsigma^2)pK + 2p\sum_{k=0}^{K-1}\mathbb{E}\left\| \sum_{i=1}^{n}p_i\*g_{k-\tau_k,i}\right\|^2 + G_\infty^2dK\right)
\end{align*}
\end{lemma}
\begin{proof}
\begin{align*}
& \sum_{i=1}^{n}p_i\mathbb{E}\left\|\*X_{k}\left(\frac{\*1}{n} - e_i\right)\right\|^2\\
= & \sum_{i=1}^{n}p_i\mathbb{E}\left\|\left(\*X_{k-1}\*W_{k-1} - \alpha\*{\tilde{G}}_{k-1-\tau_{k-1}} + \*\Omega_{k-1}\right)\left(\frac{\*1}{n} - \*e_i\right)\right\|^2\\
\overset{\text{$X_0=0$}}{=} & \sum_{i=1}^{n}p_i\mathbb{E}\left\|\sum_{t=0}^{k-1}\left(- \alpha\*{\tilde{G}}_{t-\tau_t} + \*\Omega_t\right)\*\Lambda_{t+1}^{k-1}\*e_i\right\|^2\\
\leq & 2\sum_{i=1}^{n}p_i\mathbb{E}\left\|\sum_{t=0}^{k-1}\alpha\*{\tilde{G}}_{t-\tau_t}\*\Lambda_{t+1}^{k-1}\*e_i\right\|^2 + 2\sum_{i=1}^{n}p_i\mathbb{E}\left\|\sum_{t=0}^{k-1}\*\Omega_t\*\Lambda_{t+1}^{k-1}\*e_i\right\|^2\\
\end{align*}
Now for the first term, we have
\begin{align*}
2\sum_{i=1}^{n}p_i\mathbb{E}\left\|\sum_{t=0}^{k-1}\alpha\*{\tilde{G}}_{t-\tau_t}\*\Lambda_{t+1}^{k-1}e_i\right\|^2
\leq & 2p\alpha^2\mathbb{E}\left\|\sum_{t=0}^{k-1}\*{\tilde{G}}_{t-\tau_t}\*\Lambda_{t+1}^{k-1}\right\|^2_F\\
\leq & 2p\alpha^2\mathbb{E}\left(\sum_{t=0}^{k-1}\left\|\*{\tilde{G}}_{t-\tau_t}\right\|_F\left\|\*\Lambda_{t+1}^{k-1}\right\|\right)^2\\
\leq & 2p\alpha^2\mathbb{E}\left(\sum_{t=0}^{k-1}\left\|\*{\tilde{G}}_{t-\tau_t}\right\|_F\left\|\*\Lambda_{t+1}^{k-1}\right\|_1\right)^2\\
\leq & 8p\alpha^2\mathbb{E}\left(\sum_{t=0}^{k-1}\left\|\*{\tilde{G}}_{t-\tau_t}\right\|_F2^{-\left\lfloor\frac{k-t-1}{\tmix{}}\right\rfloor}\right)^2
\end{align*}
Now we replace $k$ with $k-\tau_k$, that is
\begin{align*}
\sum_{i=1}^{n}p_i\mathbb{E}\left\|\*X_{k-\tau_k}\left(\frac{\*1}{n} - \*e_i\right)\right\|^2
\leq 8p\alpha^2\mathbb{E}\left(\sum_{t=0}^{k-\tau_k-1}\left\|\*{\tilde{G}}_{t-\tau_t}\right\|_F2^{-\left\lfloor\frac{k-\tau_k-t-1}{\tmix{}}\right\rfloor}\right)^2 + 2\sum_{i=1}^{n}p_i\mathbb{E}\left\|\sum_{t=0}^{k-\tau_k-1}\*\Omega_t\*\Lambda_{t+1}^{k-\tau_k-1}\*e_i\right\|^2
\end{align*}
Summing from $k=0$ to $K-1$ on both sides, we obtain
\begin{align*}
& \sum_{k=0}^{K-1}\sum_{i=1}^{n}p_i\mathbb{E}\left\|\*X_{k-\tau_k}\left(\frac{\*1}{n} - \*e_i\right)\right\|^2\\
\leq & 8p\alpha^2\sum_{k=0}^{K-1}\mathbb{E}\left(\sum_{t=0}^{k-\tau_k-1}\left\|\*{\tilde{G}}_{t-\tau_t}\right\|_F2^{-\left\lfloor\frac{k-\tau_k-t-1}{\tmix{}}\right\rfloor}\right)^2\\
& + 2\sum_{i=1}^{n}p_i\sum_{k=0}^{K-1}\mathbb{E}\left\|\sum_{t=0}^{k-\tau_k-1}\*\Omega_t\*\Lambda_{t+1}^{k-\tau_k-1}e_i\right\|^2\\
\leq & 8p\alpha^2\sum_{k=0}^{K-1}\mathbb{E}\left(\sum_{t=0}^{k-\tau_k-1}\left\|\*{\tilde{G}}_{t-\tau_t}\right\|_F2^{-\left\lfloor\frac{k-\tau_k-t-1}{\tmix{}}\right\rfloor}\right)^2\\
& + 2\sum_{i=1}^{n}p_i\sum_{k=0}^{K-1}\mathbb{E}\left(\sum_{t=0}^{k-\tau_k-1}\left\|\*\Omega_t\right\|_{1,2}\left\|\*\Lambda_{t+1}^{k-\tau_k-1}\right\|_1\left\|\*e_i\right\|_1\right)^2\\
\leq & 8p\alpha^2\sum_{k=0}^{K-1}\mathbb{E}\left(\sum_{t=0}^{k-\tau_k-1}\left\|\*{\tilde{G}}_{t-\tau_t}\right\|_F2^{-\left\lfloor\frac{k-\tau_k-t-1}{\tmix{}}\right\rfloor}\right)^2\\
& + 8\sum_{i=1}^{n}p_i\sum_{k=0}^{K-1}\mathbb{E}\left(\sum_{t=0}^{k-\tau_k-1}\left\|\*\Omega_t\right\|_{1,2}2^{-\left\lfloor\frac{k-\tau_k-t-1}{\tmix{}}\right\rfloor}\right)^2\\
\overset{\text{Lemma~\ref{Async_lemma6}}}{\leq} & 8p\alpha^2\sum_{k=0}^{K-1}\mathbb{E}\left(\sum_{t=0}^{k-\tau_k-1}\left\|\*{\tilde{G}}_{t-\tau_t}\right\|_F2^{-\left\lfloor\frac{k-\tau_k-t-1}{\tmix{}}\right\rfloor}\right)^2 + 32\tmix^2{}\sum_{i=1}^{n}p_i\sum_{k=0}^{K-1}\mathbb{E}\left\|\*\Omega_k\right\|_{1,2}^2\\
\leq & 8p\alpha^2\sum_{k=0}^{K-1}\mathbb{E}\left(\sum_{t=0}^{k-\tau_k-1}\left\|\*{\tilde{G}}_{t-\tau_t}\right\|_F2^{-\left\lfloor\frac{k-\tau_k-t-1}{\tmix{}}\right\rfloor}\right)^2 + 128\delta^2 B_\theta^2d\tmix^2{}K\\
\overset{\text{Lemma~\ref{Async_lemma6}}}{\leq} & 32p\alpha^2\tmix^2{}\sum_{k=0}^{K-1}\mathbb{E}\left\|\*{\tilde{G}}_{k-\tau_k}\right\|_F^2 + 128\delta^2 B_\theta^2d\tmix^2{}K
\end{align*}
Note that for the first term, we have
\begin{align*}
& \sum_{k=0}^{K-1}\mathbb{E}\left\|\*{\tilde{G}}_{k-\tau_k}\right\|^2_F\\ 
= & \sum_{k=0}^{K-1}\mathbb{E}\left\|\*{\tilde{g}}_{k-\tau_k,i_k}\right\|^2\\
= & \sum_{k=0}^{K-1}\mathbb{E}\left\|\*{\tilde{g}}_{k-\tau_k,i_k} - \*g_{k-\tau_k,i_k}\right\|^2 + \sum_{k=0}^{K-1}\mathbb{E}\left\|\*g_{k-\tau_k,i_k}\right\|^2\\
\leq & \sigma^2K + \sum_{k=0}^{K-1}\sum_{i=1}^{n}p_i\mathbb{E}\left\|\*g_{t-\tau_t,i}\right\|^2\\
\leq & (\sigma^2+6\varsigma^2)K + 12L^2\sum_{k=0}^{K-1}\sum_{i=1}^{n}p_i\mathbb{E}\left\|\*X_{k-\tau_k}\left(\frac{\*1}{n} - \*e_i\right)\right\|^2 + 2\sum_{k=0}^{K-1}\mathbb{E}\left\| \sum_{i=1}^{n}p_i\*g_{k-\tau_k,i}\right\|^2
\end{align*}
Putting these two terms back, we obtain
\begin{align*}
& \sum_{k=0}^{K-1}\sum_{i=1}^{n}p_i\mathbb{E}\left\|\*X_{{k-\tau_k}}\left(\frac{\*1}{n} - \*e_i\right)\right\|^2\\
\leq & 32p\alpha^2\tmix^2{}\left((\sigma^2+6\varsigma^2)K + 12L^2\sum_{k=0}^{K-1}\sum_{i=1}^{n}p_i\mathbb{E}\left\|\*X_{k-\tau_k}\left(\frac{\*1}{n} - \*e_i\right)\right\|^2 + 2\sum_{k=0}^{K-1}\mathbb{E}\left\| \sum_{i=1}^{n}p_i\*g_{k-\tau_k,i}\right\|^2\right)\\
+ & 128\delta^2 B_\theta^2d\tmix^2{}K
\end{align*}
Rearrange the terms, we obtain
\begin{align*}
& \left(1-192p\alpha^2\tmix^2{}L^2\right)\sum_{k=0}^{K-1}\sum_{i=1}^{n}p_i\mathbb{E}\left\|\*X_{{k-\tau_k}}\left(\frac{\*1}{n} - \*e_i\right)\right\|^2\\ 
\leq & 32p\alpha^2\tmix^2{}\left((\sigma^2+6\varsigma^2)K + 2\sum_{k=0}^{K-1}\mathbb{E}\left\| \sum_{i=1}^{n}p_i\*g_{k-\tau_k,i}\right\|^2\right) + 128\delta^2 B_\theta^2\tmix^2{}K\\
\overset{\text{Lemma~\ref{Async_lemma7}}}{\leq} & 32\alpha^2\tmix^2{}\left((\sigma^2+6\varsigma^2)pK + 2p\sum_{k=0}^{K-1}\mathbb{E}\left\| \sum_{i=1}^{n}p_i\*g_{k-\tau_k,i}\right\|^2 + G_\infty^2dK\right)
\end{align*}
Let $A_1 = 1-192p\alpha^2\tmix^2{}L^2$, we obtain
\begin{align*}
\sum_{k=0}^{K-1}\sum_{i=1}^{n}p_i\mathbb{E}\left\|\*X_{{k-\tau_k}}\left(\frac{\*1}{n} - \*e_i\right)\right\|^2
\leq & \frac{32\alpha^2\tmix^2{}}{A_1}\left((\sigma^2+6\varsigma^2)pK + 2p\sum_{k=0}^{K-1}\mathbb{E}\left\| \sum_{i=1}^{n}p_i\*g_{k-\tau_k,i}\right\|^2 + G_\infty^2dK\right)
\end{align*}
\end{proof}

\begin{lemma}\label{Async_lemma4}
\begin{align*}
&\frac{1}{K}\sum_{k=0}^{K-1}\mathbb{E}\left\|\nabla f(\overline{\*X}_{k})\right\|^2 + \left(1-\frac{2\alpha L}{n}\right)\frac{1}{K}\sum_{k=0}^{K-1}\mathbb{E}\left\|\nabla\overline{F}(\*X_{k-\tau_k})\right\|^2\\
\leq & \frac{2n(f(\*0) - f^*)}{\alpha K} + \frac{2L^2}{K}\sum_{k=0}^{K-1}\mathbb{E}\left\|\frac{(\*X_k - \*X_{k-\tau_k})\*1}{n}\right\|^2\\
+ & \left(2L^2 + \frac{12\alpha L^3}{n}\right)\frac{1}{K}\sum_{k=0}^{K-1}\sum_{i=1}^{n}p_i\mathbb{E}\left\|\*X_{k-\tau_k}\left(\frac{\*1}{n} - e_i\right)\right\|^2 + \frac{(\sigma^2+6\varsigma^2)\alpha L}{n}
\end{align*}
\end{lemma}
\begin{proof}
We start from $f(\overline{\*X}_{k+1})$
Since
\begin{align*}
\overline{\*X}_{k+1}=\*X_k\*W_k\frac{\*1}{n} + (\*{\hat{X}}_k-\*X_k)(\*W_k-\*I)\frac{\*1}{n} - \alpha\overline{\*{\tilde{G}}}_{k-\tau_k}=\overline{\*X}_{k}- \alpha\overline{\*{\tilde{G}}}_{k-\tau_k}
\end{align*}
Then from Taylor Expansion, we have
\begin{align*}
  & \mathbb{E}f(\overline{\*X}_{k+1})\\
= & \mathbb{E}f\left(\overline{\*X}_{k} - \alpha\overline{\*{\tilde{G}}}_{k-\tau_k}\right)\\
\leq & \mathbb{E}f(\overline{\*X}_{k}) - \alpha\mathbb{E}\langle\nabla f(\overline{\*X}_{k}), \overline{\*{\tilde{G}}}_{k-\tau_k}\rangle + \frac{\alpha^2L}{2}\mathbb{E}\left\|\overline{\*{\tilde{G}}}_{k-\tau_k}\right\|^2\\
= & \mathbb{E}f(\overline{\*X}_{k}) - \alpha\mathbb{E}\langle\nabla f(\overline{\*X}_{k}), \overline{\*G}_{k-\tau_k}\rangle - \alpha\mathbb{E}\langle\nabla f(\overline{\*X}_{k}), \overline{\*{\tilde{G}}}_{k-\tau_k} - \overline{\*G}_{k-\tau_k}\rangle +  \frac{\alpha^2L}{2}\mathbb{E}\left\|\overline{\*{\tilde{G}}}_{k-\tau_k}\right\|^2 \\
= & \mathbb{E}f(\overline{\*X}_{k}) - \frac{\alpha}{n}\mathbb{E}\langle\nabla f(\overline{\*X}_{k}), \nabla\overline{F}(\*X_{k-\tau_k})\rangle + \frac{\alpha^2L}{2}\mathbb{E}\left\|\frac{\*{\tilde{g}}_{k-\tau_k,i_k}}{n}\right\|^2\\
\leq & \mathbb{E}f(\overline{\*X}_{k}) - \frac{\alpha}{n}\mathbb{E}\langle\nabla f(\overline{\*X}_{k}), \nabla\overline{F}(\*X_{k-\tau_k})\rangle\\
& + \frac{\alpha^2L}{2}\sum_{i=1}^{n}p_i\mathbb{E}\left\|\frac{\*{\tilde{g}}_{k-\tau_k,i_k} - \*g_{k-\tau_k,i_k}}{n}\right\|^2 + \frac{\alpha^2L}{2}\sum_{i=1}^{n}p_i\mathbb{E}\left\|\frac{\*g_{k-\tau_k,i}}{n}\right\|^2\\
\leq &\mathbb{E}f(\overline{\*X}_{k}) - \frac{\alpha}{n}\mathbb{E}\langle\nabla f(\overline{\*X}_{k}), \nabla\overline{F}(\*X_{k-\tau_k})\rangle + \frac{\alpha^2L\sigma^2}{2n^2} + \frac{\alpha^2L}{2n^2}\sum_{i=1}^{n}p_i\mathbb{E}\|\*g_{k-\tau_k,i}\|^2\\
= & \mathbb{E}f(\overline{\*X}_{k}) + \frac{\alpha }{2n}\mathbb{E}\left\|\nabla f(\overline{\*X}_{k}) - \nabla\overline{F}(\*X_{k-\tau_k})\right\|^2 - \frac{\alpha }{2n}\mathbb{E}\left\|\nabla f(\overline{\*X}_{k})\right\|^2 - \frac{\alpha }{2n}\mathbb{E}\left\|\nabla\overline{F}(\*X_{k-\tau_k})\right\|^2\\
+ & \frac{\alpha^2L\sigma^2}{2n^2} + \frac{\alpha^2L}{2n^2}\sum_{i=1}^{n}p_i\mathbb{E}\|\*g_{k-\tau_k,i}\|^2
 \end{align*}
Rearrange these terms, we can get
\begin{align*}
&\frac{\alpha }{2n}\mathbb{E}\left\|\nabla f(\overline{\*X}_{k})\right\|^2 + \frac{\alpha}{2n}\mathbb{E}\left\|\nabla\overline{F}(\*X_{k-\tau_k})\right\|^2\\
&\leq \mathbb{E}f(\overline{\*X}_{k}) - \mathbb{E}f(\overline{\*X}_{k+1}) + \frac{\alpha}{2n}\mathbb{E}\left\|\nabla f(\overline{\*X}_{k}) - \nabla\overline{F}(\*X_{k-\tau_k})\right\|^2\\
& + \frac{\alpha^2L\sigma^2}{2n^2} + \frac{\alpha^2L}{2n^2}\sum_{i=1}^{n}p_i\mathbb{E}\|\*g_{k-\tau_k,i}\|^2
\end{align*}
Summing over $k=0$ to $K-1$ on both sides, we can get
\begin{align*}
&\frac{1}{K}\sum_{k=0}^{K-1}\mathbb{E}\left\|\nabla f(\overline{\*X}_{k})\right\|^2 + \frac{1}{K}\sum_{k=0}^{K-1}\mathbb{E}\left\|\nabla\overline{F}(\*X_{k-\tau_k})\right\|^2\\
&\leq \frac{2n(f(\*0) - f^*)}{\alpha K} + \frac{1}{K}\sum_{k=0}^{K-1}\mathbb{E}\left\|\nabla f(\overline{\*X}_{k}) - \nabla\overline{F}(\*X_{k-\tau_k})\right\|^2 + \frac{\alpha L\sigma^2}{n} + \frac{\alpha L}{nK}\sum_{k=0}^{K-1}\sum_{i=1}^{n}p_i\mathbb{E}\|\*g_{k-\tau_k,i}\|^2
\end{align*}
For $\sum_{k=0}^{K-1}\mathbb{E}\left\|\nabla f(\overline{\*X}_{k}) -\nabla\overline{F}(\*X_{k-\tau_k})\right\|^2$, we have
\begin{align*}
&\sum_{k=0}^{K-1}\mathbb{E}\left\|\nabla f(\overline{\*X}_{k}) -\nabla\overline{F}(\*X_{k-\tau_k})\right\|^2\\
\leq & 2\sum_{k=0}^{K-1}\mathbb{E}\left\|\nabla f(\overline{\*X}_{k}) - \nabla f(\overline{\*X}_{k-\tau_k})\right\|^2 + 2\sum_{k=0}^{K-1}\mathbb{E}\left\|\nabla f(\overline{\*X}_{k-\tau_k}) - \nabla\overline{F}(\*X_{k-\tau_k})\right\|^2\\
= & 2\sum_{k=0}^{K-1}\mathbb{E}\left\|\nabla f(\overline{\*X}_{\*k}) - \nabla f(\overline{\*X}_{k-\tau_k})\right\|^2 + 2\sum_{k=0}^{K-1}\mathbb{E}\left\|\sum_{i=1}^{n}p_i\left(\nabla f_i(\overline{\*X}_{k-\tau_k}) - \*g_{k-\tau_k,i}\right)\right\|^2\\
\leq &2\sum_{k=0}^{K-1}\mathbb{E}\left\|\nabla f(\overline{\*X}_{k}) - \nabla f(\overline{\*X}_{k-\tau_k})\right\|^2 + 2\sum_{k=0}^{K-1}\mathbb{E}\sum_{i=1}^{n}p_i\left\|\nabla f_i(\overline{\*X}_{k-\tau_k}) - \*g_{k-\tau_k,i}\right\|^2\\
\leq & 2L^2\sum_{k=0}^{K-1}\mathbb{E}\left\|\frac{(\*X_k - \*X_{k-\tau_k})\*1}{n}\right\|^2 + 2L^2\sum_{k=0}^{K-1}\sum_{i=1}^{n}p_i\mathbb{E}\left\|\*X_{k-\tau_k}\left(\frac{\*1}{n} - e_i\right)\right\|^2\\
\end{align*}
Putting it back, we have
\begin{align*}
&\frac{1}{K}\sum_{k=0}^{K-1}\mathbb{E}\left\|\nabla f(\overline{\*X}_{k})\right\|^2 + \frac{1}{K}\sum_{k=0}^{K-1}\mathbb{E}\left\|\nabla\overline{F}(\*X_{k-\tau_k})\right\|^2\\
\leq & \frac{2n(f(\*0) - f^*)}{\alpha K} + \frac{2L^2}{K}\sum_{k=0}^{K-1}\mathbb{E}\left\|\frac{(\*X_k - \*X_{k-\tau_k})\*1}{n}\right\|^2\\
& + \frac{2L^2}{K}\sum_{k=0}^{K-1}\sum_{i=1}^{n}p_i\mathbb{E}\left\|\*X_{k-\tau_k}\left(\frac{\*1}{n} - e_i\right)\right\|^2 + \frac{\alpha L\sigma^2}{n} + \frac{\alpha L}{nK}\sum_{k=0}^{K-1}\sum_{i=1}^{n}p_i\mathbb{E}\|\*g_{k-\tau_k,i}\|^2\\
\overset{Lemma~\ref{Async_lemma2}}{\leq} & \frac{2n(f(\*0) - f^*)}{\alpha K} + \frac{2L^2}{K}\sum_{k=0}^{K-1}\mathbb{E}\left\|\frac{(\*X_k - \*X_{k-\tau_k})\*1}{n}\right\|^2\\
& + \frac{2L^2}{K}\sum_{k=0}^{K-1}\sum_{i=1}^{n}p_i\mathbb{E}\left\|\*X_{k-\tau_k}\left(\frac{\*1}{n} - e_i\right)\right\|^2 + \frac{\alpha L\sigma^2}{n}\\
& + \frac{\alpha L}{nK}\sum_{k=0}^{K-1}\left(12L^2\sum_{i=1}^{n}p_i\mathbb{E}\left\|\*X_{k-\tau_k}\left(\frac{\*1}{n} - \*e_i\right)\right\|^2
+ 6\varsigma^2 + 2\mathbb{E}\left\| \sum_{i=1}^{n}p_i\*g_{k-\tau_k,i}\right\|^2\right)\\
= & \frac{2n(f(\*0) - f^*)}{\alpha K} + \frac{2L^2}{K}\sum_{k=0}^{K-1}\mathbb{E}\left\|\frac{(\*X_k - \*X_{k-\tau_k})\*1}{n}\right\|^2\\
& + \left(2L^2 + \frac{12\alpha L^3}{n}\right)\frac{1}{K}\sum_{k=0}^{K-1}\sum_{i=1}^{n}p_i\mathbb{E}\left\|\*X_{k-\tau_k}\left(\frac{\*1}{n} - \*e_i\right)\right\|^2\\
& + \frac{(\sigma^2+6\varsigma^2)\alpha L}{n}+ \frac{2\alpha L}{nK}\sum_{k=0}^{K-1}\mathbb{E}\left\| \sum_{i=1}^{n}p_i\*g_{k-\tau_k,i}\right\|^2
\end{align*}
Note that
\begin{align*}
\mathbb{E}\left\| \sum_{i=1}^{n}p_i\*g_{k-\tau_k,i}\right\|^2 = \mathbb{E}\left\|\nabla\overline{F}(\*X_{k-\tau_k})\right\|^2
\end{align*}
Moving it to the left side, we finally get
\begin{align*}
&\frac{1}{K}\sum_{k=0}^{K-1}\mathbb{E}\left\|\nabla f(\overline{\*X}_{k})\right\|^2 + \left(1-\frac{2\alpha L}{n}\right)\frac{1}{K}\sum_{k=0}^{K-1}\mathbb{E}\left\|\nabla\overline{F}(\*X_{k-\tau_k})\right\|^2\\
\leq & \frac{2n(f(\*0) - f^*)}{\alpha K} + \frac{2L^2}{K}\sum_{k=0}^{K-1}\mathbb{E}\left\|\frac{(\*X_k - \*X_{k-\tau_k})\*1}{n}\right\|^2\\
+ & \left(2L^2 + \frac{12\alpha L^3}{n}\right)\frac{1}{K}\sum_{k=0}^{K-1}\sum_{i=1}^{n}p_i\mathbb{E}\left\|\*X_{k-\tau_k}\left(\frac{\*1}{n} - \*e_i\right)\right\|^2
+ \frac{(\sigma^2+6\varsigma^2)\alpha L}{n}
\end{align*}
That completes the proof.
\end{proof}

\begin{lemma}\label{Async_lemma5}
For all $k \geq 0$, we have
\begin{align*}
& \frac{2L^2}{K}\sum_{k=0}^{K-1}\mathbb{E}\left\|(\*X_k - \*X_{k-\tau_k})\frac{\*1}{n}\right\|^2\\
\leq & \frac{2\alpha^2T^2(\sigma^2+6\varsigma^2)L^2}{n^2} + \frac{24L^4\alpha^2T^2}{n^2K}\sum_{k=0}^{K-1}\sum_{i=1}^{n}p_i\mathbb{E}\left\|\*X_{k-\tau_k}\left(\frac{\*1}{n} - e_i\right)\right\|^2\\
& + \frac{4\alpha^2T^2L^2}{n^2K}\sum_{k=0}^{K-1}\mathbb{E}\left\| \sum_{i=1}^{n}p_i\*g_{k-\tau_k,i}\right\|^2
\end{align*}
\end{lemma}
\begin{proof}
From Lemma~\ref{Async_lemma4}, we know the fact
\begin{align*}
\overline{\*X}_{k+1}=\*X_k\*W_k\frac{\*1}{n} + (\*{\hat{X}}_k-\*X_k)(\*W_k-\*I)\frac{\*1}{n} - \alpha\overline{\*{\tilde{G}}}_{k-\tau_k}=\overline{\*X}_{k}- \alpha\overline{\*{\tilde{G}}}_{k-\tau_k}
\end{align*}
As a result
\begin{align*}
& \sum_{k=0}^{K-1}\mathbb{E}\left\|(\*X_k - \*X_{k-\tau_k})\frac{\*1}{n}\right\|^2\\
= & \sum_{k=0}^{K-1}\mathbb{E}\left\|\sum_{t=1}^{\tau_k}\alpha\*{\tilde{G}}_{k-t}\frac{\*1}{n}\right\|^2\\
\leq & \alpha^2\sum_{k=0}^{K-1}\tau_k\sum_{t=1}^{\tau_k}\mathbb{E}\left\|\*{\tilde{G}}_{k-t}\frac{\*1}{n}\right\|^2\\
\leq & \alpha^2\sum_{k=0}^{K-1}\tau_k\sum_{t=1}^{\tau_k}\left(\frac{\sigma^2}{n^2} + \frac{1}{n^2}\sum_{i=1}^{n}p_i\mathbb{E}\left\|\*g_{k-t,i}\right\|^2\right)\\
\leq & \frac{\alpha^2T^2\sigma^2K}{n^2} + \frac{\alpha^2T}{n^2}\sum_{k=0}^{K-1}\sum_{t=1}^{\tau_k}\sum_{i=1}^{n}p_i\mathbb{E}\left\|\*g_{k-t,i}\right\|^2\\
\leq & \frac{\alpha^2T^2\sigma^2K}{n^2} + \frac{\alpha^2T}{n^2}\sum_{k=0}^{K-1}\sum_{t=1}^{\tau_k}\left(12L^2\sum_{i=1}^{n}p_i\mathbb{E}\left\|\*X_{k-t}\left(\frac{\*1}{n} - \*e_i\right)\right\|^2 + 6\varsigma^2 + 2\mathbb{E}\left\| \sum_{i=1}^{n}p_i\*g_{k-t,i}\right\|^2\right)\\
\leq & \frac{\alpha^2T^2\sigma^2K}{n^2} + \frac{\alpha^2T^2}{n^2}\sum_{k=0}^{K-1}\left(12L^2\sum_{i=1}^{n}p_i\mathbb{E}\left\|\*X_{k-\tau_k}\left(\frac{\*1}{n} - \*e_i\right)\right\|^2 + 6\varsigma^2 + 2\mathbb{E}\left\| \sum_{i=1}^{n}p_i\*g_{k-\tau_k,i}\right\|^2\right)\\
= & \frac{\alpha^2T^2(\sigma^2+6\varsigma^2)K}{n^2} + \frac{12L^2\alpha^2T^2}{n^2}\sum_{k=0}^{K-1}\sum_{i=1}^{n}p_i\mathbb{E}\left\|\*X_{k-\tau_k}\left(\frac{\*1}{n} - \*e_i\right)\right\|^2\\
& + \frac{2\alpha^2T^2}{n^2}\sum_{k=0}^{K-1}\mathbb{E}\left\| \sum_{i=1}^{n}p_i\*g_{k-\tau_k,i}\right\|^2
\end{align*}
And we get
\begin{align*}
& \frac{2L^2}{K}\sum_{k=0}^{K-1}\mathbb{E}\left\|(\*X_k - \*X_{k-\tau_k})\frac{\*1}{n}\right\|^2\\
\leq & \frac{2\alpha^2T^2(\sigma^2+6\varsigma^2)L^2}{n^2} + \frac{24L^4\alpha^2T^2}{n^2K}\sum_{k=0}^{K-1}\sum_{i=1}^{n}p_i\mathbb{E}\left\|\*X_{k-\tau_k}\left(\frac{\*1}{n} - \*e_i\right)\right\|^2\\
& + \frac{4\alpha^2T^2L^2}{n^2K}\sum_{k=0}^{K-1}\mathbb{E}\left\| \sum_{i=1}^{n}p_i\*g_{k-\tau_k,i}\right\|^2
\end{align*}
That completes the proof.
\end{proof}

\begin{lemma}\label{Async_lemma6}
Given non-negative sequences $\{a_t\}_{t=1}^{\infty}$, $\{b_t\}_{t=1}^{\infty}$ and $\{\tau_t\}_{t=1}^{\infty}$ and a positive number $T$ that satisfying
\begin{displaymath}
a_t = \sum_{s=1}^{t-\tau_t}\rho^{\left\lfloor\frac{t-\tau_t-s}{T}\right\rfloor}b_s
\end{displaymath}
with $0\leq\rho<1$,we have
\begin{align*}
S_k & = \sum_{t=1}^{k}a_t \leq \frac{(2-\rho)T}{1-\rho}\sum_{s=1}^{k}b_s\\
D_k & = \sum_{t=1}^{k}a_t^2 \leq \frac{(2-\rho)T^2}{(1-\rho)^2}\sum_{s=1}^{k}b_s^2\\
\end{align*}
\end{lemma}
\begin{proof}
\begin{align*}
S_k & = \sum_{t=1}^{k}a_t = \sum_{t=1}^{k}\sum_{s=1}^{t-\tau_t}\rho^{\left\lfloor\frac{t-\tau_t-s}{T}\right\rfloor}b_s \leq \sum_{t=1}^{k}\sum_{s=1}^{t}\rho^{\max\left(\left\lfloor\frac{t-\tau_t-s}{T}\right\rfloor, 0\right)}b_s = \sum_{s=1}^{k}\sum_{t=s}^{k}\rho^{\max\left(\left\lfloor\frac{t-\tau_t-s}{T}\right\rfloor, 0\right)}b_s\\
& = \sum_{s=1}^{k}\sum_{t=0}^{k-\tau_k-s}\rho^{\left\lfloor\frac{t}{T}\right\rfloor}b_s + \sum_{s=1}^{k}\sum_{t=1}^{\tau_k}\rho^{0}b_s \leq \sum_{s=1}^{k}\left(\sum_{t=0}^{T-1}\sum_{m=0}^{\infty}\rho^{m}\right)b_s + \tau_k\sum_{s=1}^{k}b_s\leq \left(T + \frac{T}{1-\rho}\right)\sum_{s=1}^{k}b_s\\
D_k & = \sum_{t=1}^{k}a_t^2 = \sum_{t=1}^{k}\sum_{s=1}^{t-\tau_t}\rho^{\left\lfloor\frac{t-\tau_t-s}{T}\right\rfloor}b_s\sum_{r=1}^{t-\tau_t}\rho^{\left\lfloor\frac{t-\tau_t-r}{T}\right\rfloor}b_r = \sum_{t=1}^{k}\sum_{s=1}^{t-\tau_t}\sum_{r=1}^{t-\tau_t}\rho^{\left\lfloor\frac{t-\tau_t-s}{T}\right\rfloor+\left\lfloor\frac{t-\tau_t-r}{T}\right\rfloor}b_sb_r\\
& \leq \sum_{t=1}^{k}\sum_{s=1}^{t-\tau_t}\sum_{r=1}^{t-\tau_t}\rho^{\left\lfloor\frac{t-\tau_t-s}{T}\right\rfloor+\left\lfloor\frac{t-\tau_t-r}{T}\right\rfloor}\frac{b_s^2 + b_r^2}{2} = \sum_{t=1}^{k}\sum_{s=1}^{t-\tau_t}\sum_{r=1}^{t-\tau_t}\rho^{\left\lfloor\frac{t-\tau_t-s}{T}\right\rfloor+\left\lfloor\frac{t-\tau_t-r}{T}\right\rfloor}b_s^2\\
& \leq \sum_{t=1}^{k}\sum_{s=1}^{t-\tau_t}b_s^2\rho^{\left\lfloor\frac{t-\tau_t-s}{T}\right\rfloor}\sum_{r=1}^{t-\tau_t}\rho^{\left\lfloor\frac{t-\tau_t-r}{T}\right\rfloor}\leq \sum_{t=1}^{k}\sum_{s=1}^{t-\tau_t}b_s^2\rho^{\left\lfloor\frac{t-\tau_t-s}{T}\right\rfloor}\sum_{r=0}^{T-1}\sum_{m=0}^{\infty}\rho^m\\
cs6& \leq \frac{T}{1-\rho}\sum_{t=1}^{k}\sum_{s=1}^{t-\tau_t}\rho^{\left\lfloor\frac{t-\tau_t-s}{T}\right\rfloor}b_s^2 \overset{\text{Using} S_k}{\leq} \frac{(2-\rho)T^2}{(1-\rho)^2}\sum_{s=1}^{k}b_s^2
\end{align*}
\end{proof}

\begin{lemma}\label{Async_lemma7}
for $\forall i,j$ and $\forall k\geq 0$, we have
\begin{align*}
\left\|\*X_k(\*e_i - \*e_j)\right\|_\infty <\theta = 16\tmix{}\alpha G_\infty
\end{align*}
\end{lemma}
\begin{proof}
We use mathmatical induction to prove this.

I. First, for $k=0$, we have
\begin{align*}
\left\|\*X_k(\*e_i - \*e_j)\right\|_\infty=0<\theta=16\tmix{}\alpha G_\infty
\end{align*}
II. Suppose for $k\geq 0$, we have $\left\|\*X_t(\*e_i - \*e_j)\right\|_\infty<\theta$, $\forall t\leq k$, then we have
\begin{align*}
    & \left\|\*X_{k+1}(\*e_i - \*e_j)\right\|_\infty\\
    \leq & \left\|\*X_{k+1}\left(\frac{\*1}{n} - \*e_i\right)\right\|_\infty + \left\|\*X_{k+1}\left(\frac{\*1}{n} - \*e_j\right)\right\|_\infty\\
    \leq & \left\|\*X_{k+1}\left(\*I-\frac{\*1\*1^\top}{n}\right)\right\|_{1,\infty}\left\|\*e_i\right\|_1 + \left\|\*X_{k+1}\left(\*I-\frac{\*1\*1^\top}{n}\right)\right\|_{1,\infty}\left\|\*e_j\right\|_1\\
    = & 2\left\|\*X_{k+1}\left(\*I-\frac{\*1\*1^\top}{n}\right)\right\|_{1,\infty}\\
    \leq & 2\left\|\left(\*X_{k}\*W_{k} - \alpha\*{\tilde{G}}_{k-\tau_{k}} + \*\Omega_{k}\right)\left(\frac{\*1}{n} - \*e_i\right)\right\|_{1,\infty}\\
    = & 2\left\|\sum_{t=0}^{k}\left(-\alpha\*{\tilde{G}}_{t-\tau_t}+\*\Omega_t\right)\left(\prod_{q=t+1}^{k}\*W_q-\frac{\*1\*1^\top}{n}\right)\right\|_{1,\infty}\\
    \leq & 2\sum_{t=0}^{k}\left\|\left(-\alpha\*{\tilde{G}}_{t-\tau_t}+\*\Omega_t\right)\left(\prod_{q=t+1}^{k}\*W_q-\frac{\*1\*1^\top}{n}\right)\right\|_{1,\infty}\\
    \leq & 2\sum_{t=0}^{k}\left\|-\alpha\*{\tilde{G}}_{t-\tau_t}+\*\Omega_t\right\|_{1,\infty}\left\|\prod_{q=t+1}^{k}\*W_q-\frac{\*1\*1^\top}{n}\right\|_1\\
    \leq & 4(\alpha G_\infty+2\delta B_\theta)\sum_{t=0}^{k}2^{-\lfloor (k-t)/\tmix{}\rfloor}\\
    < & 4(\alpha G_\infty+2\delta B_\theta)\sum_{t=0}^{\tmix{}-1}\sum_{r=0}^{\infty}2^{-r}\\
    \leq & 8(\alpha G_\infty+2\delta B_\theta)\tmix{}
\end{align*}
Put in $\delta=\frac{1}{64\tmix{}+2}$, we obtain
\begin{align*}
\left\|\*X_{k+1}(\*e_i - \*e_j)\right\|_2< 8(\alpha G_\infty+2\delta B_\theta)\tmix{}=8\tmix{}\alpha G_\infty+8\tmix{}\alpha G_\infty= 16\tmix{}\alpha G_\infty
\end{align*}
Combining I and II and we complete the proof.
\end{proof}

\end{document}